\documentclass{article} %
\usepackage{collas2022_conference,times}

\usepackage{amsmath,amsfonts,bm}

\def\eqref#1{equation~\ref{#1}}

\def\1{\bm{1}}

\DeclareMathAlphabet{\mathsfit}{\encodingdefault}{\sfdefault}{m}{sl}
\SetMathAlphabet{\mathsfit}{bold}{\encodingdefault}{\sfdefault}{bx}{n}

\def\sA{{\mathbb{A}}}

\def\sR{{\mathbb{R}}}
\def\sS{{\mathbb{S}}}

\usepackage{hyperref}
\hypersetup{
    colorlinks=true,
    linkcolor=red,
    filecolor=magenta,      
    urlcolor=blue,
    citecolor=purple,
    pdftitle={Overleaf Example},
    pdfpagemode=FullScreen,
    }

\usepackage[utf8]{inputenc} %
\usepackage[T1]{fontenc}    %
\usepackage{url}            %
\usepackage{booktabs}       %
\usepackage{amsfonts}       %
\usepackage{nicefrac}       %
\usepackage{microtype}      %
\usepackage{xcolor}         %

\usepackage{nameref}
\usepackage{import}
\usepackage{float}
\usepackage{amsthm}

\usepackage{color}
\usepackage{amsmath}
\usepackage{amssymb}
\usepackage{amsthm}
\usepackage{mathtools}
\usepackage[mathscr]{eucal}%
\usepackage{bm}%
\usepackage{bbm}
\usepackage{relsize}
\usepackage{graphics}
\usepackage{wrapfig}
\usepackage{multirow}
\usepackage{enumitem} 
\usepackage{stackengine}
\usepackage{multicol}
\usepackage{cancel}

\usepackage{array}
\usepackage{multicol}

\usepackage{pgf}
\usepackage{xinttools}%
\usepackage{tikz}
\usetikzlibrary{arrows.meta}
\usepackage{textcomp}
\usetikzlibrary{calc,shapes,arrows,positioning,automata,trees}
\usepackage{placeins} %
\usepackage{tabularx}
\usepackage{graphicx}
\usepackage{subcaption}
\usepackage{xcolor}
\usepackage{listings}
\usepackage{xargs,lipsum,changepage,ifthen} %

\usepackage[toc,page]{appendix}
\usepackage{selectp}

\newtheorem{theorem}{Theorem}
\newtheorem{definition}{Definition}
\newtheorem{proposition}{Proposition}%
\newtheorem{lemma}{Lemma}%
\newtheorem*{theorem*}{Theorem}
\newtheorem*{definition*}{Definition}

\newcommand{\dE}{\mathbb{E}}

 \newcommand{\dP}{\mathbb{P}} 
 
 \newcommand{\dT}{\mathbb{T}}

\newcommand{\cA}{\mathcal{A}} 
 \newcommand{\cD}{\mathcal{D}}
 
 \newcommand{\cH}{\mathcal{H}}

\newcommand{\cO}{\mathcal{O}} 
 \newcommand{\cR}{\mathcal{R}}
\newcommand{\cS}{\mathcal{S}} \newcommand{\cT}{\mathcal{T}}

\renewcommand{\sA}{\mathscr{A}}

 \renewcommand{\sR}{\mathscr{R}}
\renewcommand{\sS}{\mathscr{S}}

\newcommand{\p}[4]{{#3}\!\left#1{#4}\right#2}

\renewcommand{\leq}{\leqslant}

\usepackage{pdflscape} %

\newcolumntype{R}[1]{>{\raggedright\arraybackslash}p{#1}}
\newcolumntype{C}[1]{>{\centering\arraybackslash}p{#1}}
\newcolumntype{L}[1]{>{\raggedleft\arraybackslash}p{#1}}

\usepackage{bigdelim}
\usepackage{colortbl} %
\definecolor{mColor1}{rgb}{0.95,0.95,0.95}

\newcommand{\stoptocwriting}{%
  \addtocontents{toc}{\protect\setcounter{tocdepth}{-5}}}
\newcommand{\resumetocwriting}{%
  \addtocontents{toc}{\protect\setcounter{tocdepth}{\arabic{tocdepth}}}}

 \hypersetup{
   colorlinks=true,
   linkcolor=blue,
   citecolor=green,
   urlcolor=magenta,
   pdfborder=0 0 0,
   pdftitle=A Dataset Perspective on Offline Reinforcement Learning,
   pdfsubject= {offline reinforcement learning, RUDDER, reinforcement learning, reward redistribution, return decomposition, delayed reward, sparse reward, episodic reward},
   pdfkeywords={offline reinforcement learning, RUDDER, reinforcement learning, reward redistribution, return decomposition, delayed reward, sparse reward, episodic reward},
   pdfauthor= {Kajetan Schweighofer, Andreas Radler, Marius-Constantin Dinu, Vihang Patil, Markus Hofmarcher, Angela Bitto-Nemling, Hamid Eghbal-zadeh, Sepp Hochreiter},
   pdfstartview=FitH
}

\def\s{{s}}

\def\a{{a}}

\makeatletter
\newcommand{\setword}[2]{%
  \phantomsection
  #1\def\@currentlabel{\unexpanded{#1}}\label{#2}%
}
\makeatother

\title{A Dataset Perspective on Offline~Reinforcement~Learning}

\author{\vspace{0.1cm}\\
    \textbf{Kajetan Schweighofer~\thanks{Authors contributed equally. The code is available at \href{https://github.com/ml-jku/OfflineRL}{github.com/ml-jku/OfflineRL} 
    }~$~^{,}$\footnotemark[4] \quad
    Andreas Radler~\footnotemark[1]~$~^{,}$\footnotemark[4] \quad 
    Marius-Constantin Dinu\footnotemark[1]~$~^{,}$\footnotemark[4]~$~^{,}$\footnotemark[3]} \\ 
    \textbf{Markus Hofmarcher~\footnotemark[4] \quad Vihang Patil~\footnotemark[4] \quad
    Angela Bitto-Nemling~\footnotemark[4]~$~^{,}$\footnotemark[2]} \\
    \textbf{Hamid Eghbal-zadeh~\footnotemark[4] \quad
    Sepp Hochreiter~\footnotemark[4]~$~^{,}$\footnotemark[2]}
     \\ \\
  \footnotemark[4]~~ELLIS Unit Linz and LIT AI Lab, \\ 
                  ~~Institute for Machine Learning, \\ 
                  ~~Johannes Kepler University Linz, Austria \\ 
  \footnotemark[2]~~Institute of Advanced Research in 
                    Artificial Intelligence (IARAI), Vienna, Austria \\
  \footnotemark[3]~~Dynatrace Research, Linz, Austria 
}

\collasfinalcopy %

\begin{document}

\maketitle

\begin{abstract}
The application of Reinforcement Learning (RL) in real world environments 
can be expensive or risky due to sub-optimal policies during training.
In Offline RL, this problem is avoided since interactions with an environment are prohibited.
Policies are learned from a given dataset, 
which solely determines their performance.
Despite this fact, how dataset characteristics influence Offline RL algorithms is still hardly investigated.
The dataset characteristics are determined by the behavioral policy that samples this dataset. 
Therefore, we define characteristics of behavioral policies as exploratory for yielding high expected information in their interaction with the Markov Decision Process (MDP) and as exploitative for having high expected return.
We implement two corresponding empirical measures for the datasets sampled by the behavioral policy in deterministic MDPs.
The first empirical measure SACo is 
defined by the normalized unique state-action pairs and captures exploration.
The second empirical measure TQ is 
defined by the normalized average trajectory return and captures exploitation.
Empirical evaluations show the effectiveness of TQ and SACo.
In large-scale experiments using our proposed measures, we show that the unconstrained off-policy Deep Q-Network family requires datasets with high SACo to find a good policy.
Furthermore, experiments show that policy constraint algorithms
perform well on datasets with high TQ and SACo.
Finally, the experiments show, that purely dataset-constrained Behavioral Cloning performs
competitively to the best Offline RL algorithms for datasets with high~TQ.  
\end{abstract}

\begin{figure}[h]
    \centering
    \includegraphics[width=1\textwidth]{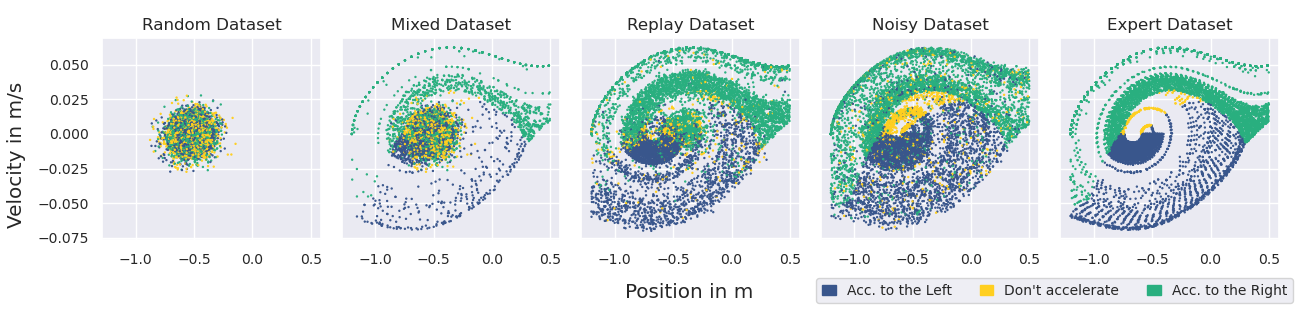}
    \caption{We illustrate the effect of the behavioral policy on the distribution of the sampled dataset. State-action pairs of different datasets were sampled from the \texttt{MountainCar} environment using different behavioral policies. 
    One can visually perceive differences between these datasets, which we aim to quantify by our introduced measures (see Sec.~\ref{derived-measures}).
    }
    \label{fig:res:projections_single}
\end{figure}

\stoptocwriting
\section{Introduction}
\label{sec:introduction}

Central problems in RL 
are credit assignment \citep{Sutton:84, Arjona:19, Holzleitner:20, Patil:20, Widrich:21, Dinu:22}
and efficient exploration of the environment \citep{Wiering:98, Mcfarlane:03,  Schmidhuber:10}.
Exploration can be costly due to 
high computational complexity, violation of physical constraints, 
risk of physical damage, interaction with human experts, etc. \citep{Dulac:2019}. 
Furthermore, exploration may endanger humans through accidents inflicted by self-driving cars, crashes of production machines when optimizing
production processes, or high financial losses when applied in trading or pricing. 
In limiting cases, simulations may alleviate these factors. 
However, designing robust and high quality simulators is a challenging, time consuming and resource intensive task, and introduces problems related to distributional shift and domain gap between the real world environment and the simulation \citep{Rao:20}. 

Confronted with these tasks, one can utilize the framework of Offline RL \citep{Levine:20}, also referred to as Batch RL \citep{Lange:12}, which offers to learn 
policies from pre-collected or logged datasets, 
without interacting with an environment \citep{Agarwal:20, Fujimoto:19a, Fujimoto:19b, Kumar:20}. 
Many such Offline RL datasets already exist for various real world problems \citep{Cabi:2019, dasari:2020, Yu:2020}. 
Offline RL shares
numerous traits with supervised deep learning,
including, but not limited to leveraging large datasets. 
A core obstacle is generalization to unseen data,
as stored samples may not cover the entire state-action space.
In Offline RL, the generalization problem takes the form of
domain shift \citep{Adler:20} during inference.
Apart from the non-stationarity of an environment or agent to environment interactions, the domain shift may be caused by the data collection process itself \citep{Khetarpal:20}. 
We illustrate this by collecting datasets using different behavioral policies in Fig.~\ref{fig:res:projections_single}.

Multiple Offline RL algorithms \citep{Agarwal:20, Fujimoto:19a, Fujimoto:19b, Gulcehre:21, Kumar:20, Wang:20} have been proposed to 
address these problems and have shown good results. 
Well known off-policy algorithms such as Deep Q-Networks (DQN) \citep{Mnih:13} can readily be used in Offline RL, by initializing the replay-buffer with a pre-collected dataset.
In practice, however, those algorithms often fail or lag far behind the performance they attain when trained in an Online RL setting.
The reduced performance is attributed to the extrapolation errors for unseen state-action pairs and the resulting domain shift between the fixed given dataset and the states visited by the learned policy \citep{Fujimoto:19a, Gulcehre:21}.
Several algorithmic improvements tackle these problems, including policy constraints \citep{Fujimoto:19a,Fujimoto:19b,Wang:20}, regularization of learned action-values \citep{Kumar:20}, and off-policy algorithms with more robust action-value estimates \citep{Agarwal:20}.
While unified datasets have been released \citep{Gulcehre:20, Fu:21} for comparisons of Offline RL algorithms, grounded work in understanding how the dataset characteristics influence the performance of algorithms is still lacking \citep{Riedmiller:21, Monier:20}.

We therefore study core dataset characteristics, and derive from first-principles theoretical measures related to exploration and exploitation which are well established policy properties.
We derive a measure of exploration based on the \emph{expected information} of the interaction of the behavioral policy in the MDP, the transition-entropy.  
Furthermore, we show that for deterministic MDPs, the transition-entropy equals the occupancy-entropy.
Our measure of exploitation, the expected trajectory return, is a generalization of the expected return of a policy.
We show that these measures have theoretical guarantees under MDP homomorphisms \citep{vanderPol:20}, thus exhibit certain 
stability traits under such transformations.

To characterize datasets and compare them across environments and generating policies,
we implement two empirical measures that correspond to the theoretical measures:
(1) SACo, corresponding to the occupancy-entropy, defined by the normalized unique state-action pairs, capturing exploration.
(2) TQ, corresponding to the expected trajectory return, defined by the normalized average return of trajectories in the dataset, capturing exploitation.

We conducted experiments on six different environments from three different environment suites \citep{Brockman:16, Chevalier-Boisvert:18, Young:19}, 
to create datasets with different characteristics 
(see Sec.~\ref{sec:dataset_generation}).
On these datasets, 6750 RL learning trials were conducted,
which cover a selection of popular algorithms in the Offline RL setting \citep{Agarwal:20, Dabney:17, Fujimoto:19b, Gulcehre:21, Kumar:20, Mnih:13, Pomerleau:91, Wang:20}.
We evaluated their performance
on datasets with different TQ and SACo. 
Variants of the off-policy DQN family \citep{Mnih:13, Agarwal:20, Dabney:17} 
were found to require datasets with high SACo to perform well. 
Algorithms that constrain the learned policy towards
the distribution of the behavioral policy perform well for
datasets with high TQ or SACo or intermediate variations thereof. 
For datasets with high TQ, Behavioral Cloning (BC) \citep{Pomerleau:91} 
outperforms variants of the DQN family and is competitive to the best performing Offline RL algorithms. 

In summary, our contributions are:
\begin{itemize}
    \item[\textbf{(a)}] we derive theoretical measures that capture exploration and exploitation,
    \item[\textbf{(b)}] we prove theoretical guarantees for the stability of these measures under MDP homomorphisms,
    \item[\textbf{(c)}] we provide an effective method to characterize datasets through the empirical measures TQ and SACo,
    \item[\textbf{(d)}] we conduct an extensive empirical evaluation of how dataset characteristics affect algorithms in Offline RL.
\end{itemize}

\section{Characterizing RL Datasets}
\label{sec:eval_metrics}

The selection of suitable measures for evaluating RL datasets and enabling their comparison with respect to algorithmic performance is a challenging and open research question \citep{Riedmiller:21, Monier:20}.
As the distribution of the dataset is governed by the behavioral policy used to sample it, we aim to find measures for the characteristics of the behavioral policy.
We define the characteristics of the behavioral policy as how exploitative and explorative it acts, and analyze the corresponding measures. 
These are the expected trajectory return for exploitation, and the transition-entropy for exploration in stochastic MDPs, which simplifies to the occupancy-entropy for exploration in deterministic MDPs.
Furthermore, we study these theoretical measures under MDP homomorphisms \citep{Ravindran:01, Givan:03, vanderPol:20, Abel:2022} 
to analyze their dependence on such transformations.
Finally, we implement empirical measures that correspond to the expected trajectory return and the occupancy-entropy, TQ and SACo, which can be calculated for a given set of datasets. 

We define our problem setting as a finite MDP to be a $5$-tuple of $(\sS, \sA, \sR, p, \gamma)$ of finite set $\sS$ with states $s$ (random variable $S_t$ at time $t$), $\sA$ with actions $a$ (random variable $A_t$), $\sR$ with rewards $r$ (random variable $R_{t+1}$), dynamics $p(S_{t+1} = s', R_{t+1} = r \mid S_t = s, A_t = a)$, and $\gamma \in [0, 1)$ as a discount factor.
The agent selects actions $a \sim \pi(S_t = s)$ based on the policy $\pi$, which depends on the current state $s$. 
Our objective is to find the policy $\pi$ that maximizes the expected return $G_t =  \sum_{k=0}^{T} \gamma^k R_{t+k+1}$.
In Offline RL, we assume that a dataset $\cD = \{\tau_i \} |_{i=1}^B$, consisting of $B$ trajectories, is provided. 
A single trajectory $\tau$ consists of a sequence of $(s, a, r, s')$ tuples.

\subsection{Transition Entropy as Measure for Exploration}
\label{sec:exploration}

We define \emph{explorativeness} of a policy $\pi$ as the expected information of the interaction between the policy and a certain MDP.
While the policy actively selects its next action given its current state, it is transitioned into a next state and receives a reward signal according to the dynamics of the MDP $p(r,s' \mid s,a)$, which cannot be influenced by the policy.
Therefore, the policy interacting in the MDP can be seen as a single stochastic process that generates transitions $(s,a,r,s')$.
A policy is explorative, if it is able to generate many different transitions with high probability in an MDP.
As transitions can only be observed through the interaction between policy and MDP, explorativeness of a policy can only be defined in conjunction with a specific MDP.
Policies that act very explorative in one MDP, could act much less explorative in another MDP and vice versa.
For instance, a policy that does not open doors might explore multiple rooms very thoroughly if all doors are already open, but will get stuck in a single room if they are closed initially.

We measure explorativeness of a policy by the Shannon entropy \citep{Shannon:48} of the transition probabilities $p_\pi(s, a, r, s')$ under the policy interacting with the MDP.
We can rewrite the transition probabilities as $~{p(s,a,r,s') = p(s',r \mid s,a) \; p(s,a)}$.
The dynamics $p(r,s' \mid s,a)$ are solely MDP-dependent, while the state-action probability $p_{\pi}(s,a)$ is policy- and MDP-dependent.
The state-action probability $p_\pi(s,a)$ is often referred to as occupancy measure\footnote{To avoid additional assumptions on the MDP, we focus this analysis on a single stage, see \cite{Neu:20}.} $\rho_\pi(s, a)$ \citep{Neu:20, Ho:16}, as it describes how the policy occupies the state-action space.

We start with the Shannon entropy of the transition probabilities  
\begin{equation}
    H(p_\pi(s,a,r,s')) \coloneqq - \sum\limits_{\substack{s,a,r,s' \\ p_\pi(s,a,r,s') > 0}} p_\pi(s,a,r,s') \log(p_\pi(s,a,r,s')), 
\end{equation}
which we will further refer to as transition-entropy.
The transition-entropy can be factored into the occupancy weighted sum of entropies of the dynamics and the occupancy (for a detailed derivation see Eq.~\ref{eq:unified} in Sec.~\ref{appendic_subsec_explorativeness} in the Appendix):
\begin{align}
  H(p_\pi(s,a,r,s')) &=\sum\limits_{\substack{s,a}} \rho_\pi(s,a) \; H(p(r,s' \mid s, a)) + H(\rho_\pi(s, a)).
    \label{eq:unified_main}
\end{align}
An explorative policy should therefore aim to find a good balance between visiting all possible state-actions similarly likely, while visiting state-actions pairs with more stochastic dynamics $p(r, s' \mid s,a)$ more often.
This link between good exploration and visiting more stochastic dynamics is also found in optimal experiment design \citep{Storck:95, Cohn:93}.
Note that there may be other ways to define exploration. 
A reasonable option is to define a \emph{utility} measure for transitions to reach a goal or learn a certain property of an MDP.
Higher exploration would then mean that more util transitions are more likely.
Similarly, another option is to define a \emph{distance} measure for transitions and define exploration as increasing the likelihood to sample transitions which have a higher distance to each other.
A priori, it is hard to define the notions of utility or distance, hence we turned to the Shannon entropy. 
 \begin{figure}
     \centering
     \includegraphics[width=\textwidth]{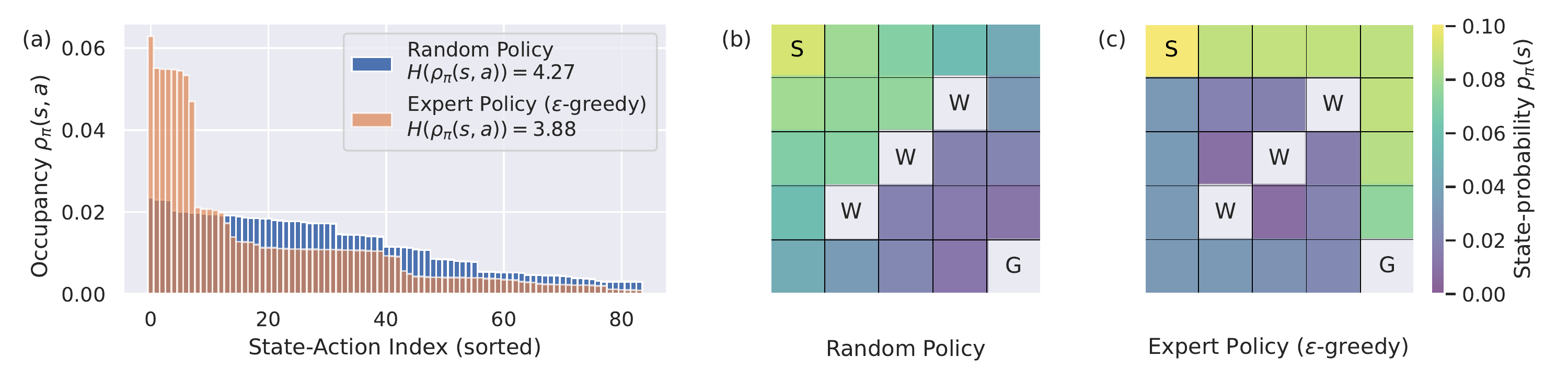}
     \caption{(a) Distribution of occupancies of two policies, a random policy and an expert policy that acts $\epsilon$-greedy with $\epsilon=0.5$. We show the occupancy-entropies, where the random policy attains a higher value than the $\epsilon$-greedy expert policy.
     Both policies are evaluated in a five by five deterministic gridworld with four eligible actions (up, down, left, right).
     Episodes start at the starting position \textbf{S} and end upon reaching the goal state \textbf{G}, which yields a positive return.
     Walls \textbf{W} cannot be passed through.
     (b) \& (c) Illustrate the state probabilities $p_\pi(s)$ under the random (b) and the $\epsilon$-greedy expert (c) policy, which is the sum over actions of the underlying occupancies $\rho_\pi(s, a)$.
     }
     \label{fig:occupancy}
 \end{figure}

\paragraph{Deterministic MDPs} In this class of MDPs, we can simplify the exploration measure from Eq.~\ref{eq:unified_main}.
Since $p(r,s' \mid s,a)$ is deterministic, the dynamics-entropy $H(p(r,s' \mid s,a))$ is zero and the left term in  Eq.~\ref{eq:unified_main} vanishes as shown in Eq.~\ref{eq:exp_info_det} in Sec.~\ref{appendic_subsec_explorativeness} in the Appendix.
Therefore, for deterministic MDPs the transition-entropy simplifies to the occupancy-entropy:
\begin{equation}
    H(\rho_\pi(s, a)) \coloneqq - \sum\limits_{\substack{s,a \\ \rho_\pi(s,a) > 0}} \rho_\pi(s,a) \log(\rho_\pi(s,a)).
    \label{eq:unified_det}
\end{equation}
In this special case, a policy explores maximally if all possible state-action pairs are visited equally likely and thus the occupancy-entropy is maximal.
Fig.~\ref{fig:occupancy} gives an intuition about the occupancy distribution and the resulting occupancy-entropy of different policies.

\subsection{Expected Trajectory Return as Measure for Exploitation}
\label{sec:trajectory_quality}

We define \emph{exploitativeness} of a policy as how performant in terms of expected return it interacts with the MDP.
We measure how exploitative a policy acts by its expected return $~{g_\pi = \dE_\pi \left[ G_t \right]}$.
Furthermore, we generalize the expected return to arbitrary policies (e.g. non-representable or non-Markovian policies as discussed in \cite{Fu:21}), as the Offline RL setting makes no assumptions on the behavioral policy used for dataset collection.
Therefore, we define exploitativeness on a distribution of trajectories $\cT$ observed from arbitrary policies as 
expected trajectory return $g_\cT$, which is given by:
\begin{equation}
    g_\cT \coloneqq \dE_{\tau \sim \cT} \left[ \sum_{t = 0}^\infty \gamma^t r_t \mid r_t \in \tau \right].
    \label{eq:trajectory_quality}
\end{equation}
In this turn, we can thus also evaluate policies for which we can not represent the dataset generating behavior but can generate data indefinitely.
This is of relevance when considering human-generated datasets.

\subsection{Common Structures of MDPs}
\label{common_structures_of_MDPs}

We are interested in the 
\emph{stability} of the proposed measures. 
We regard stability of our measures as the degree to which the measures change under slight modifications of the input.
Modifications to the input are small changes in the state space, action space or dynamics of the MDP.
To formalize changes of the MDP, we introduce the concept of a common abstract MDP (AMDP) (see Fig.~\ref{fig:amdp_mapping}). 
We base the definition of the AMDP on prior work from \cite{Sutton:99smdp, Jong:05, Li06mdp, Abel:2019, vanderPol:20, Abel:2022} by considering state and action abstractions.

\begin{definition}
\label{def:amdp_definition}

Given two MDPs $M = (\sS, \sA, \sR, p, \gamma)$ and $\tilde M = (\tilde \sS, \tilde \sA, \tilde \sR, \tilde p, \gamma)$ with finite and discrete state-action spaces. We assume there exists a common \textit{abstract MDP} (AMDP) $\hat{M} = (\hat{\sS}, \hat{\sA}, \hat{\sR}, \hat{p}, \gamma)$, whereas $M$ and $\tilde M$ are homomorphic images of $\hat{M}$.
We define an MDP homomorphism by the surjective abstraction functions as $\phi : \sS \rightarrow \hat{\sS}$ and $\tilde \phi : \tilde \sS \rightarrow \hat{\sS}$, with $\phi(s), \tilde \phi(\tilde s) \in \hat{\sS}$ for the state abstractions and $\{ \psi_s : \sA \rightarrow \hat{\sA} \mid s \in \sS \}$ and $\{ \tilde \psi_{\tilde s} : \tilde \sA \rightarrow \hat{\sA} \mid \tilde s \in \tilde \sS \}$, with $\psi_s(a), \tilde \psi_{\tilde s}(\tilde a) \in \hat{\sA}$ for the action abstractions (see Appendix Sec.~\ref{sec:amdp_definition} for more details regarding the assumptions on these abstraction functions and their implications on the MDP dynamics).
Let $\pi(a \mid s)$ and $\tilde \pi(\tilde a \mid \tilde s)$ be corresponding policies of $M$, $\tilde M$ such that they map via $\phi, \psi_s$ and $\tilde \phi, \tilde \psi_{\tilde s}$ to the same abstract policy $\hat{\pi}(\hat{a} \mid \hat{s})$ of the common AMDP $\hat{M}$.
Let $\cT$ and $\tilde \cT$ be corresponding trajectory distributions of $M$, $\tilde M$, such that they map via $\phi, \psi_s$ and $\tilde \phi, \tilde \psi_{\tilde s}$ to the same abstract trajectory distribution $\hat \cT$ of the common AMDP.
\end{definition}

\begin{table}
\centering
\begin{tabular}{cc}
\includegraphics[width=0.60\textwidth]{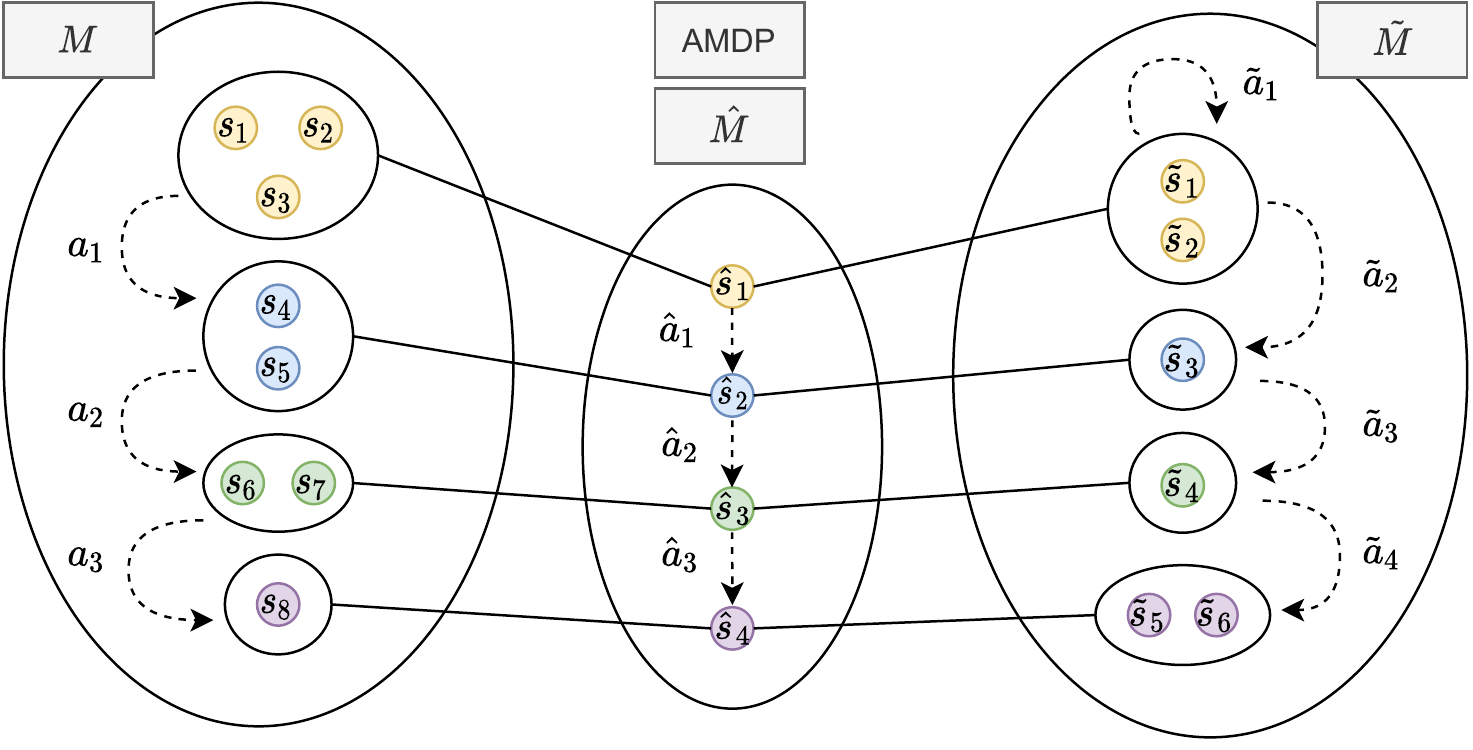} & \includegraphics[width=0.25\textwidth]{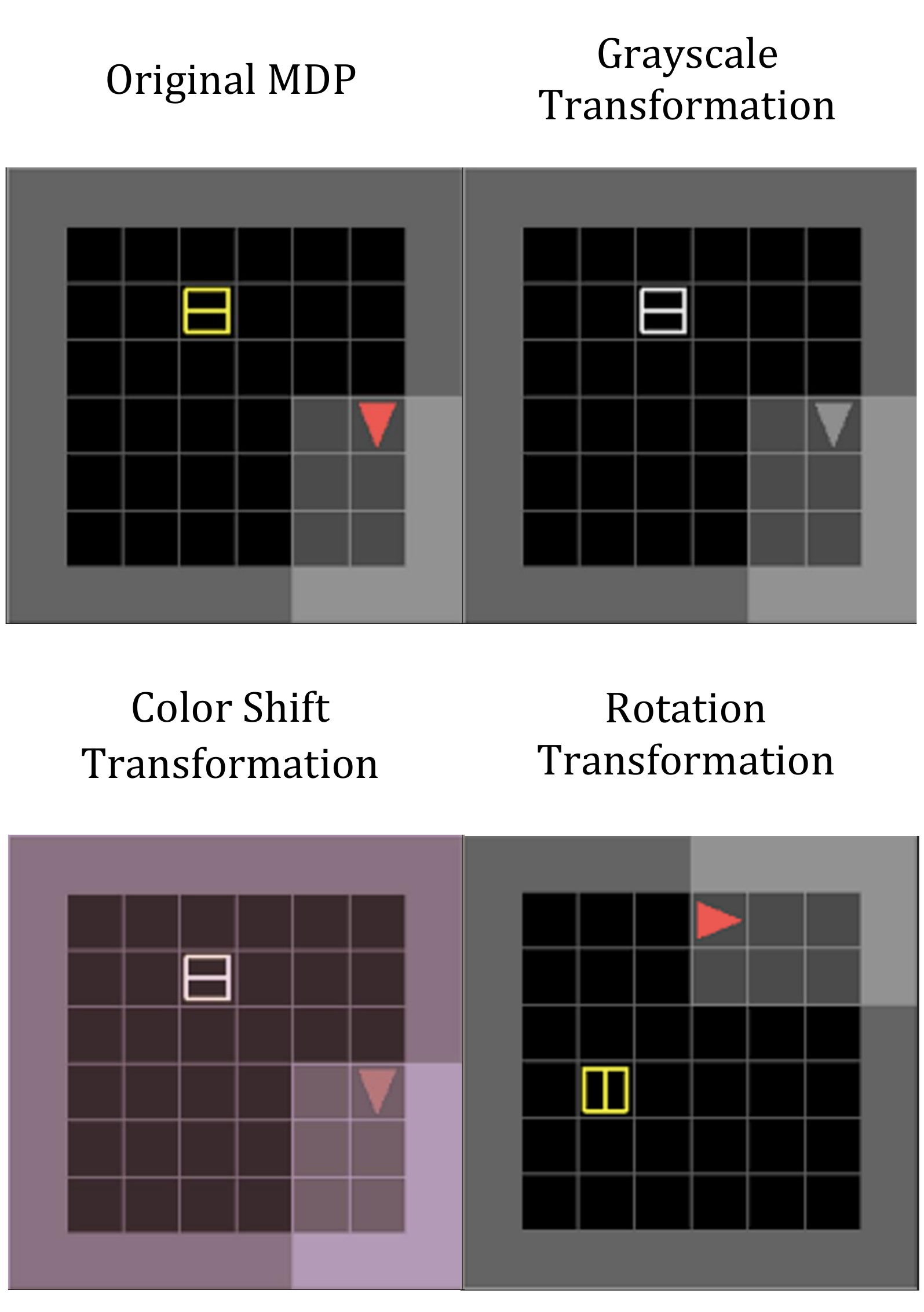} \\
\cr
\begin{tabular}[c]{@{}l@{}}(a)\end{tabular}
 & \begin{tabular}[c]{@{}l@{}}(b)\end{tabular}
\end{tabular}
\captionof{figure}{(a) Illustration of two homomorphic images of an abstract MDP (AMDP), (b) example of homomorphic transformations on \texttt{MiniGrid} \citep{Chevalier-Boisvert:18}. For further details see Sec.~\ref{sec:amdp_overview} in the Appendix.}
\label{fig:amdp_mapping}
\end{table}

We use Def.~\ref{def:amdp_definition} to derive an upper bound of the difference in transition-entropy for two homomorphic images of the same common AMDP. For brevity we write $H(p)$, $H(\tilde p)$ and $H(\hat p)$ for the transition-entropies induced by $\pi$, $\tilde \pi$ and $\hat \pi$ respectively.

\begin{theorem}
\label{th-max_abs_diff_ub} 
Given two homomorphic images $M$ and $\tilde M$ of a common AMDP $\hat{M}$ and their respective transition probabilities $p(s,a,r,s')$, $\tilde p(\tilde s,\tilde a, r,\tilde s')$ and $\hat{p}(\hat{s},\hat{a},r,\hat{s}')$ induced by corresponding policies $\pi$ and $\tilde \pi$ of $M$ and $\tilde M$ and the abstract policy $\hat \pi$ of $\hat M$ they map to.
The maximum absolute difference in transition-entropy is upper bounded by:
\begin{align}
|H(p) - H(\tilde p)| 
&\leq \max \Big [ H(p), H(\tilde p) \Big ] - H(\hat{p}) \label{form-max_abs_diff_ub}
\end{align}
\end{theorem}
\begin{proof}
We use the fact that homomorphic images have a transition-entropy greater or equal to the transition-entropy of the common AMDP. 
For a proof of this statement, see Sec.~\ref{app:derivation of measures} in the Appendix. 
Therefore, we use $H(p) \geq H(\hat{p})$ and $H(\tilde p) \geq H(\hat{p})$. 
Combining these inequalities leads to Eq.~\ref{form-max_abs_diff_ub}. 
\end{proof}

\begin{proposition}
\label{prop:1}
Any corresponding policies $\pi(a \mid s)$ and $\tilde \pi(\tilde a \mid \tilde s)$ of homomorphic images $M$ and $\tilde M$ of a common AMDP $\hat M$, that map to the same abstract policy $\hat \pi(\hat a \mid \hat s)$, exhibit the same expected return $g_\pi = g_{\tilde \pi}$. 
Furthermore, any corresponding trajectory distributions $\cT$ and $\tilde \cT$ of homomorphic images $M$ and $\tilde M$ of a common AMDP $\hat M$, that map to the same abstract trajectory distribution $\hat \cT$, exhibit the same expected trajectory return $g_\cT = g_{\tilde \cT}$. 
\end{proposition}

Prop.~\ref{prop:1} follows from the conditions in Eq.~\ref{eq:amdp_return_1} and Eq.~\ref{eq:amdp_return_2} for $M$ and $\tilde M$ respectively. 
For completeness we mention that for \emph{isomorphic} transformations of MDPs (i.e. under an MDP homomorphism with bijective abstraction functions), it follows directly that the measures discussed in Sec.~\ref{sec:exploration} and Sec.~\ref{sec:trajectory_quality} are preserved.

We conclude: 
(1) Policies, defined on homomorphic images of a common AMDP which map to the same abstract policy, yield the same expected return.
Trajectory distributions on homomorphic images that have a corresponding abstract trajectory distribution, yield the same expected trajectory return.
(2) From Eq.~\ref{form-max_abs_diff_ub} we can formally confirm the intuition about the stability of our exploration measures between homomorphic images of a common AMDP:
If the transition-entropy of the \emph{greatest} common AMDP of two corresponding MDPs increases, the difference in their transition-entropies decreases.

\subsection{Empirical Measures}
\label{derived-measures}

We aim to implement empirical measures that are applicable to a set of given datasets and correspond to the theoretical measures of exploration and exploitation of a policy introduced in Sec.~\ref{sec:exploration} and Sec.~\ref{sec:trajectory_quality}.
Furthermore, these measures are normalized to references, to allow for a comparison of datasets sampled from different MDPs.

\paragraph{SACo.} 

First, we implement an empirical measure that corresponds to the theoretical measure of exploration of a policy in a deterministic MDP, given by Eq.~\ref{eq:unified_det}.
One way to implement such a measure is the na\"{\i}ve entropy estimator of a dataset 
$\hat{H}(\cD)$
or its corresponding perplexity estimator
$\hat{PP}(\cD) = e^{\hat{H}(\cD)}$.
The entropy estimator is upper bounded by the logarithm of \emph{unique} state-action pairs in the dataset $\log(u_{s,a}(\cD))$, thus $\hat{H}(\cD) \leq \log(u_{s,a}(\cD))$.
It follows, that the perplexity estimator is upper-bounded by the number of unique state-action pairs in the dataset $\hat{PP}(\cD) \leq u_{s,a}(\cD)$.
We therefore base our empirical measure SACo on the unique state-action pairs $u_{s,a}(\cD)$, which has the benefit of being easy to interpret while corresponding to the exploration of the policy.
Additionally, $u_{s,a}(\cD)$ captures the notion of \textit{coverage of the state-action space} given a dataset, that is invoked in the Offline RL literature \citep{Fujimoto:19a, Agarwal:20, Kumar:20, Monier:20}.
A further theoretical analysis on the relation between the na\"{\i}ve entropy estimator and unique state-action count is given in Sec.\ref{sec:estimators} in the Appendix.
The unique state-action pairs are normalized with the unique state-action pairs of a \emph{reference dataset} $\cD_{\text{ref}}$ of the same MDP and the same dataset size.
Hence our empirical measure SACo is defined as
\begin{align}
    \mathit{SACo}(\cD) \coloneqq \frac{u_{s,a}(\cD)}{u_{s,a}(\cD_{\text{ref}})}.
    \label{eq:saco}
\end{align}
We use the replay dataset as the reference dataset $\cD_{\text{ref}}$ in our experiments, since it was collected throughout training of the online policy and is assumed to have the most diverse set of state-action pairs.
Counting unique state-action pairs of large datasets is often infeasible due to time and memory restrictions, or if datasets are distributed across different machines. 
We use HyperLogLog \citep{Flajolet:07} as a probabilistic counting method to facilitate the applicability of SACo to large-scale benchmarks (see Sec.~\ref{sec:count_unique_saco} in the Appendix for details).

\paragraph{TQ.} 

Second, we implement an empirical measure that estimates the expected trajectory return $g_\cT$
 given by Eq.~\ref{eq:trajectory_quality}
, which is a measure of how exploitative the behavioral policy is.
The expected trajectory return $g_\cT$ is estimated by the average return of the trajectories in the dataset $~{\bar g(\cD) = \frac{1}{B} \sum_{b=0}^B \sum_{t=0}^{T_b} \gamma^t r_{b, t}}$ for $B$ trajectories with their respective lengths $T_b$.
To allow comparisons across MDPs, we normalize the average trajectory return with the best and the worst behavior observed in the same MDP, which therefore corresponds to the quality of the trajectories relative to those.
We thus define our empirical measure TQ as the normalized average trajectory return:

\begin{equation}
    \mathit{TQ}(\cD) \coloneqq \frac{\bar g(\cD) - \bar g(\cD_{\text{min}})}{\bar g(\cD_{\text{expert}}) - \bar g(\cD_{\text{min}})},
     \label{eq:tq}
\end{equation}

where $\cD_{\text{min}}$ is a dataset collected by a minimal performant policy $\pi_{\text{min}}$ and $\cD_{\text{expert}}$ is a dataset collected by an expert policy $\pi_{\text{expert}}$.
Throughout our experiments, we chose the dataset sampled from the best policy found during online training as $\cD_{\text{expert}}$ and the dataset sampled from a random policy as $\cD_{\text{min}}$.

\section{Dataset Generation}
\label{subsec: data_in_orl}

In Offline RL,
dataset generation is neither harmonized, nor thoroughly investigated \citep{Riedmiller:21}.
It has been shown in \cite{Kumar:20} that the performance of Conservative Q-learning~(CQL)
improved by changing the behavioral policy of the underlying dataset. 
Similarly, in \cite{Gulcehre:21}
exchanging an expert with a noisy expert behavioral policy changed the performance of the compared algorithms.
Furthermore, there is no consensus of which behavioral policy is the most representative for generating datasets to compare the performance of algorithms.
\cite{Kumar:20} claims that datasets generated by multiple different behavioral policies fit a real world setting best.
Contrary to that, \cite{Fujimoto:19a},
claim that a single behavioral policy is a better fit. 
Thus, there is an ambiguity in the Offline RL literature
on what may be the correct data generation scheme 
to test Offline RL algorithms. 
We review which datasets are utilized in the literature to compare among algorithms, with the aim to represent the most prominent dataset creation schemes in our empirical evaluation.

\cite{Agarwal:20} test on a dataset which consists of all 
training samples seen during online training of a DQN agent. 
\cite{Fujimoto:19a} generate data using a 
trained policy with an exploration factor.
\cite{Fujimoto:19b} evaluates on multiple datasets,
which include a dataset consisting of
all transitions a RL algorithm samples during online training and
data generated using a trained policy. 
\cite{Gulcehre:21} uses the RL Unplugged dataset \citep{Gulcehre:20}, which consists of different datasets collected by various policies. 
\cite{Kumar:20} uses three datasets generated by using 
a random, expert and a mixture of expert and random policy, 
generated from multiple different policies.
\cite{Wang:20} use datasets generated from a pre-trained model based on transfer learning from human experts.

\subsection{Dataset Generation Schemes}
\label{sec:dataset_generation}

We generate data in five different settings: 1) \emph{random}, 2) \emph{expert}, 3) \emph{mixed} 4) \emph{noisy} and 5) \emph{replay}. 
These generation schemes are designed to systematically cover and extend prior settings from literature. 
For each of the datasets, we have collected a predefined number of samples by interacting with the respective environment (see Sec.~\ref{sec:experiments}). 
The number of samples in a dataset is determined by
the number of environment interactions that are necessary to obtain expert policies 
through an Online RL algorithm.
We use DQN \citep{Mnih:13} as a baseline for the Online RL algorithm, which serves 
as an expert behavioral policy to create and collect samples, as described below. 
Details on how the online policy was trained
are given in Sec.~\ref{subsec:online_training} in the Appendix. 

\begin{itemize}
    
\item \textbf{Random Dataset.} 
This dataset is sampled by a random behavioral policy. 
Such a dataset was used for evaluation in \cite{Kumar:20}. 
It serves as a na\"{\i}ve baseline for data collection.

\item \textbf{Expert Dataset.}
We trained an online policy until convergence and sampled with the final greedy expert policy.
Such datasets are used in \cite{Fu:21, Gulcehre:21, Kumar:20}.

\item \textbf{Mixed Dataset.} The mixed dataset is generated using a mixture
of the random dataset ($80\%$) and the expert dataset ($20\%$). 
This is similar to \cite{Fu:21, Gulcehre:21}, where they refer to such a dataset as \emph{medium-expert}.

\item \textbf{Noisy Dataset.} The noisy dataset is generated with an expert policy that selects the actions $\epsilon$-greedy with $\epsilon=0.2$. Creating a dataset from a fixed noisy policy is similar to the dataset creation process in \cite{Fujimoto:19a, Fujimoto:19b, Kumar:20, Gulcehre:21}.

\item \textbf{Replay Dataset.} This dataset is the collection of all samples generated by the online policy during training, thus, consists of data generated by multiple policies. Such a dataset was used in \cite{Agarwal:20, Fujimoto:19b}.

\end{itemize}

\subsection{Dataset Generation from a domain shift perspective}
\label{sec:dist_shift_rl_datasets}

Generating multiple datasets in the same MDP from different policies can also be interpreted as a domain shift between the distributions of the policies.
Therefore, we want to further analyze which domain shifts are possible when generating RL datasets.
The interaction of a policy $\pi$ in a certain MDP generates transitions $(s,a,r,s')$, with a probability $p_\pi(s,a,r,s)$.
We define a domain shift as a change of the distribution $p_\pi(s,a,r,s)$ to $\breve p_\pi(s,a,r,s)$, analogous to the definition of domain shift in the supervised learning setting \citep{Widmer1996, Gama2014, Webb2018, Wouter2018, Kouw2019, Khetarpal:20,  Adler:20}.
Four distinct sources of domain shift can be disentangled by applying the chain rule of conditional probability on the transition probability
$~{p_\pi(s,a,r,s) = p(r \mid s, a, s') \; p(s' \mid s, a) \; p_\pi(a \mid s) \; \rho_\pi(s)}$, where $\rho_\pi(s) = \sum_a \rho_\pi(s, a)$ is the state-occupancy.
This separation is not unique, but yields the policy and the conditional probabilities used in the definition of the MDP.
Note that we separated the dynamics $p(r, s' \mid s, a)$ into the reward-dynamics $p(r \mid s, a, s')$ and the state-dynamics $p(s' \mid s, a)$ to further disentangle possible reasons for a domain shift.
We consider the following types of domain shifts in RL:

\begin{itemize}
    \item \textbf{Reward-Dynamics shift:} $p(r \mid s, a, s')$ is changed to $\breve p(r \mid s, a, s')$, while the state-dynamics $p(s' \mid s, a)$, the policy $p_\pi(a \mid s)$ and the state-occupancy $\rho_\pi(s)$ stay the same.
    This changes the expected return under the policy, but its occupancy stays the same.
    \item \textbf{State-Dynamics shift:} $p(s' \mid s, a)$ is changed to $\breve p(s' \mid s, a)$, while the reward-dynamics $p(r \mid s, a, s')$ and the policy $p_\pi(a \mid s)$ stay the same.
    This changes the occupancy under the policy, as the same behavior will result in a different distribution of outcomes.
    Although not necessary, the expected return under the policy is likely to change under a different occupancy.
    \item \textbf{Policy shift:} $p_\pi(a \mid s)$ is changed to $\breve p_\pi(a \mid s)$ while the reward-dynamics $p(r \mid s, a, s')$ and the state-dynamics $p(s' \mid s, a)$ stay the same. 
    In general, RL datasets are not sampled from the same behavioral policy. Therefore, this domain shift is inherent to RL datasets.
\end{itemize}

The state-occupancy $\rho_\pi(s)$ depends on the state-dynamics $p(s' \mid s, a)$ and the policy $p_\pi(a \mid s)$ in a non-trivial way.
This dependence is seen when formulating the state distribution recursively as $\rho_\pi(s') = \sum_{s} \rho_\pi(s) \sum_a p_\pi(a \mid s) \; p(s' \mid s, a)$.
Note that a shift in the initial-state distribution or the absorbing-state distribution can also cause domain shifts of the state-occupancy.
Every combination of the discussed shifts can occur as well, where we refer to the co-occurrence of all types of shifts as general domain shift.
Note, that a state-dynamics and policy shift in general induce a change of the state probabilities, if their effects do not counterbalance each other.
Finally, we conclude that generating datasets with different policies generally results in a shift of the dataset distribution, which we also observe empirically
(see Sec.~\ref{sec:results_mdp_relaxations} in the Appendix).

\section{Experiments}
\label{sec:experiments}

We conducted our study on six different deterministic environments from multiple suites. 
These are two classic control environments from the OpenAI gym suite \citep{Brockman:16}, two MiniGrid \citep{Chevalier-Boisvert:18} and two MinAtar environments \citep{Young:19}. For the first two suites, $10^5$ samples were collected for every dataset, whereas $2 \cdot 10^6$ samples were collected for the MinAtar environments. 
Over all environments (six), different data generation schemes (five) and seeds (five), 
we generated a total number of 150 datasets.

We train nine different algorithms popular in the Offline RL literature 
including Behavioral Cloning~(BC) \citep{Pomerleau:91} and 
variants of DQN,  
Quantile-Regression DQN~(QRDQN) \citep{Dabney:17} and Random Ensemble Mixture~(REM) \citep{Agarwal:20}. 
Furthermore, Behavior Value Estimation~(BVE) \citep{Gulcehre:21} and Monte-Carlo Estimation~(MCE) are used. 
Finally, three widely popular Offline RL algorithms that constrain the learned policy to be near in distribution to the behavioral policy that created the dataset, Batch-Constrained Q-learning~(BCQ) \citep{Fujimoto:19b}, Conservative Q-learning~(CQL) \citep{Kumar:19} and Critic Regularized Regression~(CRR) \citep{Wang:20} are considered.
Details on specific implementations are given in Sec.~\ref{sec:algorithms} in the Appendix. 

The considered algorithms
were executed on each of the $150$ datasets 
for five different seeds. 
Details on online and offline training are given in Sec.~\ref{sec:impl_details}. 
We relate the performance of the best policies found during offline training
 to the best policy found during online training. 
The performance $\omega$ of the best policy found during offline training is given by $
    ~{\omega(\cD_{\text{offline}}) = \left(\bar g(\cD_{\text{offline}}) - \bar g(\cD_{\text{min}})\right) / \left(\bar g(\cD_{\text{expert}}) - \bar g(\cD_{\text{min}})\right)}
$.
Policies are evaluated in the environment 
after fixed intervals during offline training.
The trajectories sampled for the evaluation step yielding the highest average return represent the dataset $\cD_{\text{offline}}$.

\paragraph{Results.}

\begin{figure}
    \centering
    \includegraphics[width=0.9\textwidth]{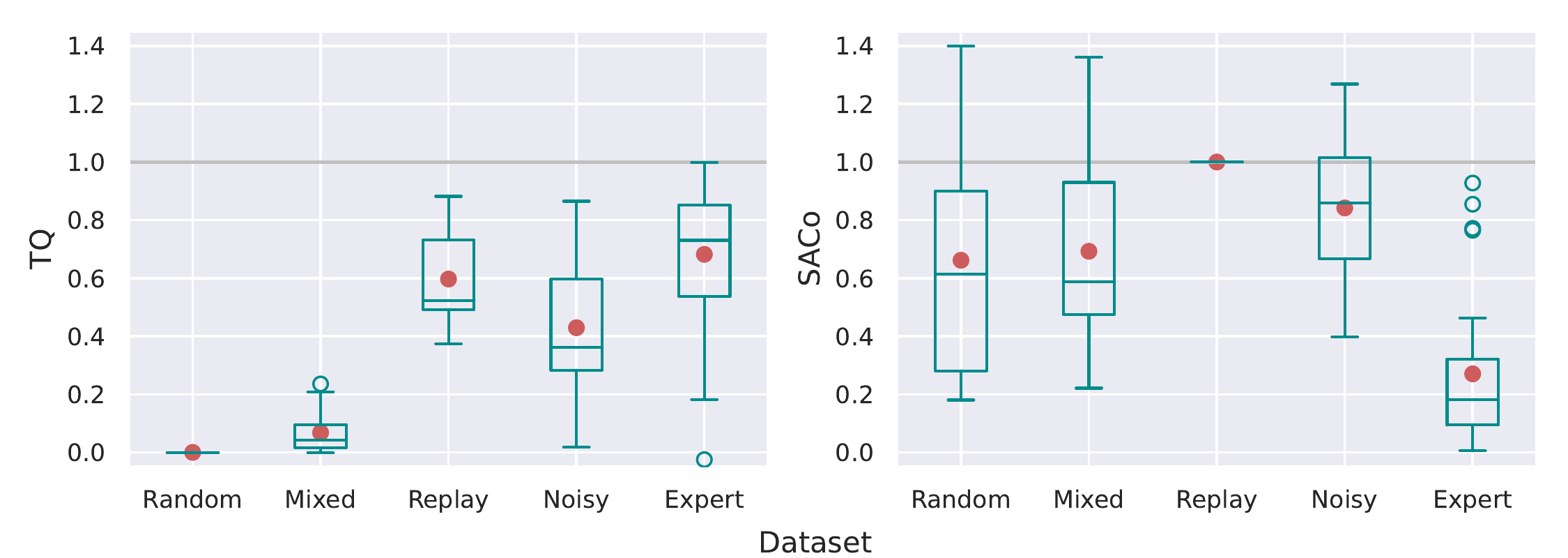}
    \caption{TQ and SACo over each dataset across dataset creation seeds and environments. Red dots represent mean values. The replay dataset exhibits a good balance between TQ and SACo, which is an explanation for the high performance across algorithms when using the replay dataset compared to other datasets.}
    \label{fig:datasets}
\end{figure}

\begin{figure}[h]
    \centering
    \includegraphics[width=\textwidth]{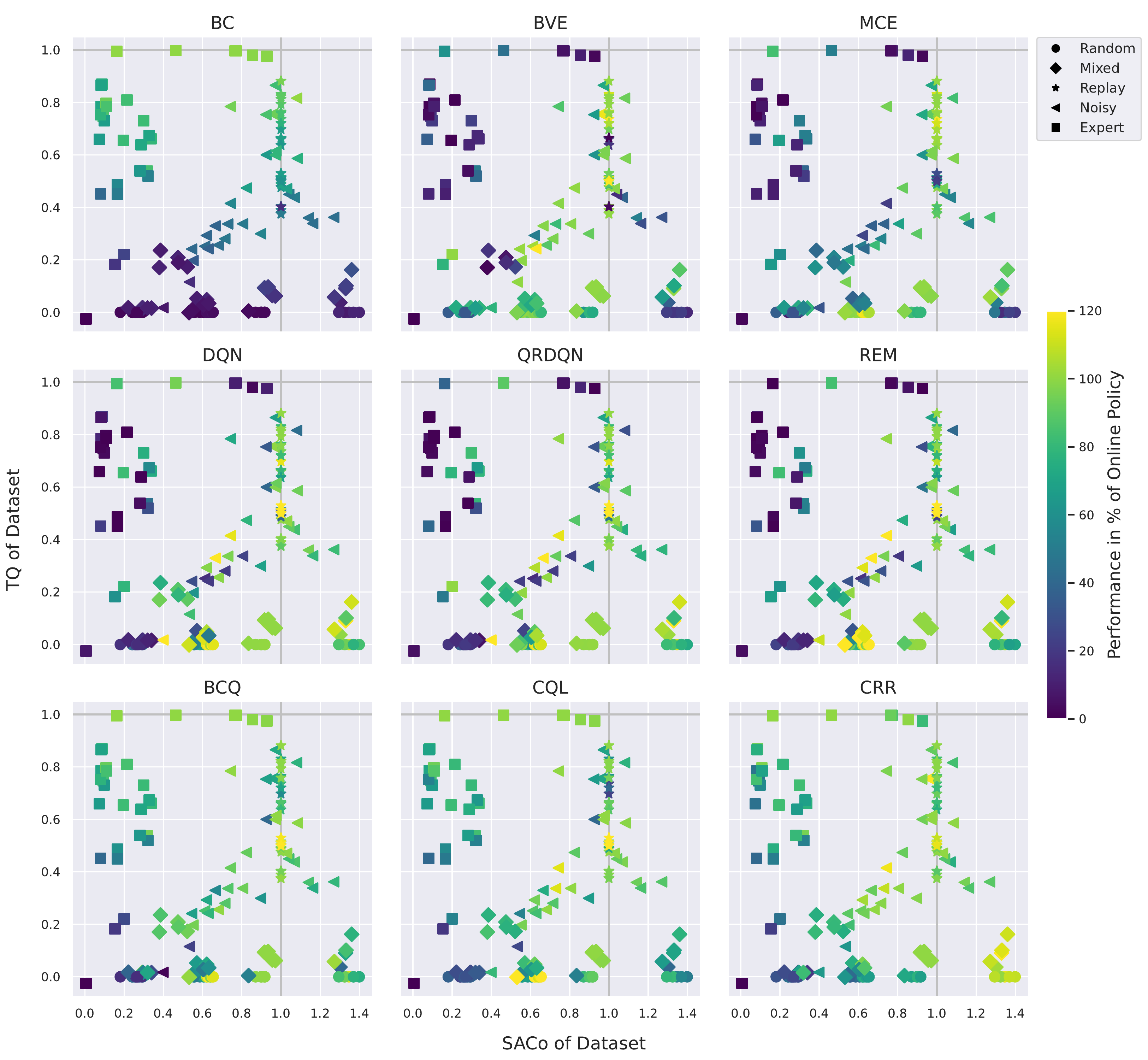}
    \caption{We characterize datasets using TQ and SACo, thus every point denotes one of the 150 datasets created for our evaluation. The position of points is thus the same in each subplot, as we trained every algorithm on every dataset. The performance of the best policy found during offline training on a dataset is denoted by the color of each individual point. We see: a) BC improves as TQ increases b) DQN variants (middle row) require high SACo to perform well c) Algorithms which constrain the learned policy towards the distribution of the behavioral policy (bottom row) perform well across datasets if they exhibit high TQ or SACo or both.}
    \label{fig:algos_all}
\end{figure}

We aim to analyze our experiments through the lens of dataset characteristics.
Fig.~\ref{fig:datasets} shows the TQ and the SACo of the gathered datasets, across dataset creation seeds and environments.
Random and mixed datasets exhibit low TQ, while expert data has the highest TQ on average.
In contrast, expert data exhibits low SACo on average, whereas random and mixed datasets attain higher values. 
The replay dataset provides a good balance between TQ and SACo. 
Fig.~\ref{fig:res:projections} in the Appendix visualizes how generating the dataset influences the covered state-action space.
In Fig.~\ref{fig:algos_all},
we characterize the generated datasets using TQ and SACo.
Each point represents a dataset, and each subplot contains the same datasets.
The performance of the best policy found during offline training using a specific dataset is denoted by the color of the respective point.

These results indicate that algorithms of the DQN family (DQN, QRDQN, REM) 
need datasets with high SACo to find a good policy. 
Furthermore, it was found that BC works well only if datasets have high TQ,
which is expected as it imitates behavior observed in the dataset. 
BVE and MCE were found to be very sensitive to the specific environment and dataset setting, favoring datasets with high SACo.
These algorithms are unconstrained in their final policy improvement step, but approximate the action-values of the behavioral policy instead of the optimal policy as in DQN.
Thus, they may not leverage datasets with high SACo as well as algorithms from the DQN family, but encounter the same limitations for datasets with low SACo.
BCQ, CQL and CRR enforce explicit or implicit constraints on the learned policy to be close to the distribution of the behavioral policy.
These algorithms were found to outperform algorithms of the DQN family on average, but especially for datasets with low SACo and high TQ. 
Furthermore, algorithms with constraints towards the behavioral policy were found to perform well if datasets exhibit high TQ or SACo or moderate values of TQ and SACo.

We observe from Fig.~\ref{fig:algos_all}, that algorithms that are closely related (e.g, DQN family, second row)
perform similarly when trained on similar datasets in terms of TQ and SACo
, which is an interesting property to investigate in future work. 
Scatterplots between pairs of TQ, SACo and the performance are given in Fig.~\ref{fig:ap_tq} and Fig.~\ref{fig:ap_saco} in the Appendix.
All scores for all environments and algorithms over datasets are given in Sec.~\ref{sec:perf_offline_algo} in the Appendix. 
We also analyzed these experiments with a logarithmic version of SACo (logarithmic in the nominator and denominator) in Sec.~\ref{app:lsaco} in the Appendix.
Furthermore, we analyzed these experiments using na\"{\i}ve entropy estimators instead of SACo in Sec.~\ref{sec:naive_emp} in the Appendix.
We found that the presented version with SACo yields the visually and conceptually simplest interpretation for our experiments, while all versions yield qualitatively similar results.

\section{Discussion}
\label{sec:analysis}

\paragraph{Limitations.}

We conducted our main experiments on discrete-action environments and showed initial results on continuous-action environments in Sec.~\ref{mujoco} in the Appendix.
Nevertheless, future work would need to investigate the effects of dataset characteristics for continuous action-spaces on a wider range of different tasks, algorithms and data collection schemes as done for discrete action-spaces.
Furthermore, this study is limited to model-free algorithms, while model-based algorithms have recently been reported to achieve superior results in complex continuous control industrial benchmarks \citep{Swazinna:22}.
For computational reasons, we limited this study to only use the off-policy algorithm DQN as an online policy. 
A comparison to an on-policy method as behavioral policy would be of interest, as well as exploration algorithms to generate even more diverse datasets; e.g. using generative flow networks \citep{Bengio:21a, Bengio:21b}.
Counting unique state-action pairs as an empirical exploration measure is simple and interpretable, but explicit entropy estimation would be especially beneficial on continuous state and action spaces.
Lastly, although we did define different types of domain shifts, the definition and analysis of an appropriate domain shift measure is beyond the scope of this work.

\paragraph{Conclusion.}

In this study, we derived two theoretical measures, the transition-entropy corresponding to explorativeness of a policy and the expected trajectory return, corresponding to the exploitativeness of a policy.
Furthermore, we analyzed stability traits of these measures under MDP homomorphisms.
Moreover, we implemented two empirical measures, TQ and SACo, which correspond to the expected trajectory return and the transition-entropy in deterministic MDPs.
We generated 150 datasets using six environments and five dataset sampling strategies over five seeds.
On these datasets, the performance of nine model-free algorithms was evaluated over independent runs, resulting in 6750 trained offline policies. 
The performance of the offline algorithms shows clear correlations with the exploratory and exploitative characteristics of the dataset, measured by TQ and SACo.
We found, that popular algorithms in Offline RL are strongly influenced by the characteristics of the dataset and the average performance across different datasets might not be enough for a fair comparison.
Our study thus provides a blueprint to characterize 
Offline RL datasets and understanding their effect on algorithms.

\paragraph{Acknowledgements}

The ELLIS Unit Linz, the LIT AI Lab, the Institute for Machine Learning, are supported by the Federal State Upper Austria. IARAI is supported by Here Technologies. We thank the projects AI-MOTION (LIT-2018-6-YOU-212), AI-SNN (LIT-2018-6-YOU-214), DeepFlood (LIT-2019-8-YOU-213), Medical Cognitive Computing Center (MC3), INCONTROL-RL (FFG-881064), PRIMAL (FFG-873979), S3AI (FFG-872172), DL for GranularFlow (FFG-871302), AIRI FG 9-N (FWF-36284, FWF-36235), ELISE (H2020-ICT-2019-3 ID: 951847). We thank Audi.JKU Deep Learning Center, TGW LOGISTICS GROUP GMBH, Silicon Austria Labs (SAL), FILL Gesellschaft mbH, Anyline GmbH, Google, ZF Friedrichshafen AG, Robert Bosch GmbH, UCB Biopharma SRL, Merck Healthcare KGaA, Verbund AG, Software Competence Center Hagenberg GmbH, T\"{U}V Austria, Frauscher Sensonic, AI~Austria~Reinforcement~Learning~Community, and the NVIDIA Corporation.

\bibliography{bibliography}
\bibliographystyle{plainnat}

\newpage
\onecolumn
\appendix

\section{Appendix}
\resumetocwriting

\setcounter{tocdepth}{3}
\setcounter{secnumdepth}{4}
\renewcommand\thefigure{\thesection.\arabic{figure}}    
\setcounter{figure}{0}    
\renewcommand\thetable{\thesection.\arabic{table}}    
\setcounter{table}{0}

\renewcommand{\contentsname}{Contents of the appendix}
\tableofcontents

\clearpage

\subsection{Explorativeness of policies in stochastic and deterministic MDPs}
\label{appendic_subsec_explorativeness}

In the following, we introduce a measure to quantify how well a given policy $\pi$ explores a given MDP on expectation.

\subsubsection{Explorativeness as expected information}

We consider an MDP with finite state, action and reward spaces throughout the following analysis.
We define \emph{explorativeness} of a policy $\pi$ as the expected information of the interaction between the policy and a given MDP.
In information theory, the expected information content of a measurement of a random variable $X$ is defined as its Shannon entropy, given by $~{H(X) \coloneqq - \sum_i p(x_i) \log (p(x_i))}$  \citep{Shannon:48}.
Corresponding to this, we define the expected information about policy MDP interactions through the transitions $(s,a,r,s')$, observed according to the transition probability $p_\pi(s,a,r,s')$ by a policy $\pi$ interacting with the environment.
We want to explicitly stress the interconnection of the policy and the MDP as a single transition generating process.
Without a given MDP, explorativeness of a policy can not be defined in this framework.
The expected information is thus given by the transition-entropy

\begin{equation}
    H(p_\pi(s,a,r,s')) \coloneqq - \sum_{\substack{s, a, r, s'\\p_\pi(s,a,r,s')>0}} p_\pi(s,a,r,s') \log(p_\pi(s,a,r,s')).
\end{equation}

To illustrate how $p_\pi(s, a,r, s')$ is influenced by the policy and the MDP dynamics, we rewrite the transition probability as $p_\pi(s, a,r, s') = p(r,s' \mid s,a) \; p_\pi(s,a)$. 
The dynamics $p(r,s' \mid s,a)$ depend solely on the MDP, while $p_\pi(s,a)$ is steered by the policy $\pi$.
The state-action probability $p_\pi(s,a)$ is often referred to as the occupancy measure $\rho^h_\pi(s, a) = \dP_\pi[s_h=s, a_h=a]$ \citep{Neu:20} induced by the policy $\pi$, where $h$ denotes the stage to specify changes of the dynamics. 
In the following, we only consider episodes consisting of a single stage and drop the index $h$ accordingly, but the following derivations would extend to episodes with multiple stages.
Consequently, $\rho_\pi(s,a)$ is used instead of $p_\pi(s,a)$ to emphasize that the state-action probability under a policy is the occupancy under this policy.

The occupancy depends not only on the policy alone, but also on the MDP dynamics.
To illustrate this fact, consider that the occupancy of a policy can be rewritten as $\rho_\pi(s,a) = p_\pi(a \mid s) \; p_\pi(s)$, thus the probability that a policy selects an action given a state $p_\pi(a \mid s)$, times the probability of being in a certain state under that policy $p_\pi(s)$.
The probability of being in a certain state can be recursively defined as $~{p_\pi(s') = \sum_s p_\pi(s) \sum_a p_\pi(a \mid s) \; p(s' \mid s, a)}$ via the MDP state-dynamics $p(s' \mid s, a)$.

Using the above, the transition-entropy can be further decomposed:
\begin{align}
 H&(p_\pi(s,a,r,s')) \nonumber\\
    &= - \sum\limits_{\substack{s,a,r,s'\\p_\pi(s,a,r,s') > 0}} p_\pi(s,a,r,s') \log(p_\pi(s,a,r,s')) \nonumber\\
    &= - \sum\limits_{\substack{s,a,r,s'\\p_\pi(s,a,r,s') > 0}} p(r,s' \mid s,a) \; \rho_\pi(s,a) \log(p(r,s' \mid s, a) \; \rho_\pi(s,a))\nonumber\\
    &= - \sum\limits_{\substack{s,a,r,s'\\p_\pi(s,a,r,s') > 0}} \rho_\pi(s,a) \; p(r,s' \mid s,a) \log(p(r,s' \mid s, a)) - \sum\limits_{\substack{s,a,r,s'\\\rho_\pi(s,a) > 0}} p(r,s' \mid s,a) \; \rho_\pi(s,a)  \log(\rho_\pi(s,a))\nonumber\\
    &= - \sum\limits_{\substack{s,a}} \rho_\pi(s,a) \sum\limits_{\substack{r,s'\\p_\pi(s,a,r,s') > 0}} p(r,s' \mid s, a) \; \log(p(r,s' \mid s, a)) - \sum\limits_{\substack{s,a,r,s'\\\rho_\pi(s,a) > 0}} p(r,s' \mid s,a) \; \rho_\pi(s,a) \log(\rho_\pi(s,a))\nonumber\\
    &= \sum\limits_{\substack{s,a}} \rho_\pi(s,a)  \; H(p(r,s' \mid s, a)) - \sum\limits_{\substack{s,a\\\rho_\pi(s,a) > 0}} \rho_\pi(s,a) \log(\rho_\pi(s,a)) \sum\limits_{r,s'} p(r,s' \mid s,a)\nonumber\\
    &= \sum\limits_{\substack{s,a}} \rho_\pi(s,a)  \; H(p(r,s' \mid s, a)) - \sum\limits_{\substack{s,a\\\rho_\pi(s,a) > 0}} \rho_\pi(s,a) \log(\rho_\pi(s,a))\nonumber\\
    &= \sum\limits_{\substack{s,a}} \rho_\pi(s,a) \; H(p(r,s' \mid s, a)) + H(\rho_\pi(s, a)).
    \label{eq:unified}
\end{align}

The transition-entropy thus equals the occupancy weighted sum of dynamics-entropies $H(p(r,s' \mid s,a))$ under every visitable state-action pair plus the occupancy-entropy $H(\rho_\pi(s, a))$ under the policy.

To maximize the transition-entropy, thus explore more on expectation, a tradeoff between two options is encountered, assuming a fixed visited state-action support under any candidate policy.
First, the occupancy can be distributed such that state-action pairs where transitions are more stochastic, and thus having high dynamics-entropy, are visited more often.
Second, the occupancy can be distributed evenly among all state-action pairs to have a high occupancy-entropy.
However, the transition-entropy is not straightforward to optimize without further assumptions, as there is a strong interplay between the policy and the MDP dynamics, which does not allow for smooth changes in the occupancy.

Furthermore, the transition-entropy generally increases if more state-action pairs are visitable under a policy.
Note that this is only strictly true, if additional assumptions on the distribution of MDP dynamics and occupancies under two policies that are compared are introduced.
This holds for instance for the upper bound on the entropy, where possible outcomes follow a uniform distribution.
The entropy of a random variable $X$ that follows a uniform distribution is $~{H(X) = \log(N)}$, which grows logarithmically by the number of possible outcomes $N$.

To summarize, the interaction of a policy in an MDP induces high transition-entropy if a large state-action support is occupied, more stochastic transitions are visited more frequently and all other possible transitions are visited evenly.

\subsubsection{Deterministic MDPs}

Deterministic MDPs differ in the sense that all information about the dynamics of a specific transition $~{p(r, s' \mid s, a)}$ is obtained the first time this transition is observed.
Note that for any deterministic MDP, there exists a function $~{\dT: \cS \times \cA \rightarrow \cR \times \cS; (s, a) \mapsto (r,s')}$ that maps state-action pairs to a reward and next state.
The dynamics of a deterministic MDP can thus be written as

\begin{equation}
    p_{\text{det}}(r,s' \mid s,a) \coloneqq 
    \begin{cases}
    1 & \text{if } (r,s') = \dT(s,a)\\
    0 & \text{otherwise.}\\
    \end{cases}
    \label{eq:det_dynamics}
\end{equation}

Starting from the general result of Eq.~\ref{eq:unified} and inserting the definition of deterministic dynamics from Eq. \ref{eq:det_dynamics}, the expected information about the dynamics in a deterministic MDP is given by 

\begin{align}
 H(p_\pi(s,a,r,s'))
    &= \sum\limits_{\substack{s,a}} \rho_\pi(s,a) \; H(p_{\text{det}}(r,s' \mid s, a)) + H(\rho_\pi(s,a)) \nonumber\\
    &= - \sum\limits_{\substack{s,a}} \rho_\pi(s,a)  \;  \sum\limits_{\substack{r, s'\\(r,s') = \dT(s,a)}} p_{\text{det}}(r,s'\mid s,a) \log(p_{\text{det}}(r,s'\mid s,a)) + H(\rho_\pi(s,a))\nonumber\\
    &= - \sum\limits_{\substack{s,a}} \rho_\pi(s,a)  \; 1 \log(1) + H(\rho_\pi(s,a))\nonumber\\
    &= H(\rho_\pi(s,a)).
\label{eq:exp_info_det}
\end{align}

The transition-entropy thus reduces to the occupancy-entropy under the policy for deterministic dynamics.
Therefore, the more uniform the policy visits the state-action space, the higher the transition-entropy and thus the more explorative the policy is on expectation.
Additionally, having a large state-action support is still beneficial to increase the transition-entropy as it is for stochastic MDPs.

\clearpage
\subsection{Bias of Entropy Estimators}\label{sec:estimators}

We discussed two possible estimators for the occupancy-entropy, which we want to characterize further in this section.
The first is the na\"{\i}ve entropy estimator for a dataset $\hat H(\cD) = - \sum_{i=1}^K \hat p_i \log(\hat p_i)$ with $\hat p_i = n_{s,a} / N$, where $n_{s,a}$ is the count of a specific state-action pairs $(s,a)$ in the dataset of size $N$.
The second is the upper bound of the na\"{\i}ve entropy estimator, $\log(u_{s,a}(\cD))$, which is a non-consistent estimator of $H(\rho_\pi)$.
We start with from the bias analysis of the na\"{\i}ve entropy estimator $\hat H(\cD)$ given by \cite{Basharin:59} and extended to the second order by \cite{Harris:75}:

\begin{align}
    H(\rho_\pi) - \dE[\hat H(\cD)] &= \frac{K-1}{2N} + \frac{1}{12 N^2} \left(\sum_k^K \frac{1}{p_k} - 1 \right) + \cO(N^{-3})
\end{align}

where $K$ denotes the number of visitable state-action pairs under the policy, thus $p_k = \rho_\pi(s,a)$ for the $k$-th visitable state-action pair.
For brevity, let

\begin{align}
    Z  &\coloneqq \frac{K-1}{2N} + \frac{1}{12 N^2} \left(\sum_k^K \frac{1}{p_k} - 1 \right) + \cO(N^{-3}) 
    \label{eq:bias}
\end{align}

We can rewrite Eq.~\ref{eq:bias} to arrive at an expression of the bias of $\log(u_{s,a}(\cD))$:

\begin{align}
    H(\rho_\pi) - \dE[\hat H(\cD)] &= Z \nonumber\\
    H(\rho_\pi) &= Z + \dE[\hat H(\cD)] \nonumber\\
    H(\rho_\pi) - \dE[\log(u_{s,a}(\cD))] &= Z + \dE[\hat H(\cD)] - \dE[\log(u_{s,a}(\cD))]
\end{align}

We use the fact that:

\begin{equation}
    0 \geq \dE[\hat H(\cD)] - \dE[\log(u_{s,a}(\cD))] \geq - \log(N)
\end{equation}

and arrive at:

\begin{equation}
    Z \geq H(\rho_\pi) - \dE[\log(u_{s,a}(\cD))] \geq Z - \log(N) \geq \frac{K-1}{2N} - \log(N)
\end{equation}

The bias of $\log(u_{s,a}(\cD))$ is thus smaller than the na\"{\i}ve entropy estimator $\hat H(\cD)$ as long as $\frac{K-1}{2N} - \log(N) \geq 0$.
This bias will not become negative, which would result in $\log(u_{s,a}(\cD))$ overestimating $H(\rho_\pi)$, if the following holds:

\begin{align}
    \frac{K-1}{2N} - \log(N) &\geq 0 \nonumber\\
    K &\geq 2N \log(N) + 1 .
\end{align}

Therefore, as long as $K \geq 2N \log(N) + 1$, the estimator $\log(u_{s,a}(\cD))$ is equally or less biased than the na\"{i}ve entropy estimator $\hat H(\cD)$.

\clearpage
\subsection{AMDP Definition}
\label{sec:amdp_definition}

In the following paragraphs, we extend the details for Definition~\ref{def:amdp_definition}:

\begin{definition*}
Given two MDPs $M = (\sS, \sA, \sR, p, \gamma)$ and $\tilde M = (\tilde \sS, \tilde \sA, \tilde \sR, \tilde p, \gamma)$, we assume there exists a common \textit{abstract MDP} (AMDP) $\hat{M} = (\hat{\sS}, \hat{\sA}, \hat{\sR}, \hat{p}, \gamma)$ (see Fig.~\ref{fig:amdp_mapping}), whereas $M$ and $\tilde M$ are homomorphic images of $\hat{M}$.
We base the definition of the AMDP on prior work from \cite{Sutton:99smdp, Jong:05, Li06mdp, Abel:2019, vanderPol:20, Abel:2022} by considering state and action abstractions. 
We define an MDP homomorphism by the surjective abstraction functions as $\phi : \sS \rightarrow \hat{\sS}$ and $\tilde \phi : \tilde \sS \rightarrow \hat{\sS}$, with $\phi(s), \tilde \phi(\tilde s) \in \hat{\sS}$ for the state abstractions and $\{ \psi_s : \sA \rightarrow \hat{\sA} \mid s \in \sS \}$ and $\{ \tilde \psi_{\tilde s} : \tilde \sA \rightarrow \hat{\sA} \mid \tilde s \in \tilde \sS \}$, with $\psi_s(a), \tilde \psi_{\tilde s}(\tilde a) \in \hat{\sA}$ for the action abstractions.

The inverse images $\phi^{-1}(\hat{s})$ with $\hat{s} \in \hat{\sS}$, and $\psi_s^{-1}(\hat{a})$ with $\hat{a} \in \hat{\sA}$, are the set of ground states and actions that correspond to $\hat{s}$ and $\hat{a}$, under abstraction function $\phi$ and $\psi_s$ respectively. 
Under these assumptions, $\{\phi^{-1}(\hat{s}) \mid \hat{s} \in \hat{\sS}\}$ and $\{\psi_s^{-1}(\hat{a}) \mid \hat{a} \in \hat{\sA}\}$ partition the ground state $\sS$ and ground action $\sA$. The inverse mappings hold equivalently for $\tilde \phi^{-1}(\hat s)$, $\tilde \psi_{\tilde s}^{-1}(\hat s)$.
\end{definition*}

The mappings are built to satisfy the following conditions:

\begin{align}
    \hat{p}(r \mid \phi(s), \psi_s(a), \phi(s')) 
    &\overset{\Delta}{=}
    p(r \mid s,a,s') 
    &\forall s, s' \in \sS, a \in \sA \label{eq:amdp_return_1}\\
    \hat{p}(r \mid \tilde \phi(\tilde s), \tilde \psi_{\tilde s}(\tilde a), \tilde \phi(\tilde s')) 
    &\overset{\Delta}{=} 
    \tilde p(r \mid \tilde s,\tilde a,\tilde s') 
    &\forall \tilde s, \tilde s' \in \tilde \sS, \tilde a \in \tilde \sA \label{eq:amdp_return_2}\\
    \hat{p}(\phi(s') \mid \phi(s), \psi_s(a)) 
    &\overset{\Delta}{=}
    \sum_{\bar s' \in \phi^{-1}(\phi(s'))} p(\bar s' \mid s, a) &\forall  s, s' \in \sS, a \in \sA\\
    \hat{p}(\tilde \phi(\tilde s') \mid \tilde \phi(\tilde s), \tilde \psi_{\tilde s}(\tilde a)) 
    &\overset{\Delta}{=}
    \sum_{\bar s' \in \tilde \phi^{-1}(\tilde \phi(\tilde s'))} \tilde p(\bar s' \mid \tilde s, \tilde a) 
    &\forall \tilde s, \tilde s' \in \tilde \sS, \tilde a \in \tilde \sA
\end{align}

Note that $\phi$ and $\psi_s$ are surjective functions, therefore expressions such as $\phi^{-1}(\phi(s'))$ return a set.

\subsection{Bound transition-entropies of homomorphic images}
\label{th-occ_entr_lb} 
\label{app:derivation of measures}

In the following, we give the missing proof of Theorem \ref{th-max_abs_diff_ub} in Sec.~\ref{common_structures_of_MDPs}.

\begin{lemma}
Let $\pi(a \mid s)$ be a policy of a homomorphic image $M$ of an AMDP $\hat{M}$ with corresponding abstract policy $\hat{\pi}(\hat{a} \mid \hat{s})$.
Let $p(s,a,r,s')$ and $\hat{p}(s,a,r,s')$ be the transition probabilities induced by $\pi(a \mid s)$ and $\hat{\pi}(\hat{a} \mid \hat{s})$ respectively.
The transition-entropy of the common AMDP $H(\hat{p}(\hat{s}, \hat{a}, r, \hat{s}'))$ provides a lower-bound to $H(p(s,a,r,s'))$.
\end{lemma}

\begin{proof}
We start with an intuition. The space states, actions and next-states of $M$ can be partitioned in a way, such that each partition maps to a single state-action-next-state tuple $(\hat{s}, \hat{a}, \hat{s}')$ in $\hat{M}$. This is illustrated in Fig.~\ref{fig:amdp_mapping}. 
For this mapping, the probability mass is conserved, formally 
$~{\hat p(\hat{s}, \hat{a}, \hat{s}') = \sum_{s \in \phi^{-1}(\hat{s}), a \in \psi_s^{-1}(\hat{a}), s' \in \phi^{-1}(\hat{s}')} p(s,a,s')}$. 
By these mappings from $\hat{M}$ to $M$, the entropy can only be increased or be equal.

Here is a formal proof.
We start with the definitions of the entropies. 

\begin{align}
    H(\hat{p}(s,a,r,s')) &= - \sum_{\hat{s}, \hat{a}, r, \hat{s}'} p(\hat{s}, \hat{a}, r, \hat{s}') \cdot \log(\hat{p}(\hat{s}, \hat{a}, r, \hat{s}'))
    \label{eq:entropy_1}
    \\
    H(p(s,a,r,s')) &= - \sum_{s, a, r, s'} p(s,a,r,s') \cdot \log(p(s,a,r,s'))
    \label{eq:entropy_2}
\end{align}

First we look at the function $f(x) = - x \log(x)$ for $x > 0$. We can show that $f(x)$ is concave by showing that its second derivative is non-positive.

\begin{align}
    f'(x) &= - \log(x) - 1 \\
    f''(x) &= - \frac{1}{x}
\end{align}

Since $f(x) > 0$ for $x > 0$ and $f(x)$ is concave, it follows that $f(x)$ is \emph{sub-additive} meaning that for every $x_1, x_2 > 0$ that $f(x_1) + f(x_2) \geq f(x_1 + x_2)$. More generally, we can say that for any $x_i > 0$ that $\sum_i f(x_i) \geq f(\sum_i x_i)$. 
From the assumptions of the MDP homomorphism it follows that the probability mass of an abstract state-action-next-state 3-tuple is equal to the sum of probability masses of the state-action-next-state 3-tuples in its inverse images.

\begin{align}
    \hat{p}(\hat{s}, \hat{a}, r, \hat{s}') &= 
    \sum\limits_{\substack{s \in \phi^{-1}(\hat{s})\\a \in \psi_s^{-1}(\hat{a})\\s' \in \phi^{-1}(\hat{s}')}} 
    p(s,a,r,s')
\end{align}

We fix an arbitrary abstract transition $(\hat{s},\hat{a},r,\hat{s}')$. From sub-additivity of $f(x)$ follows that

\begin{align}
    \sum\limits_{\substack{s \in \phi^{-1}(\hat{s})\\a \in \psi_s^{-1}(\hat{a})\\s' \in \phi^{-1}(\hat{s}')}} 
    f(p(s,a,r,s'))
    &\geq 
    f(\hat{p}(\hat{s},\hat{a},r,\hat{s}'))
    \\
    -\sum\limits_{\substack{s \in \phi^{-1}(\hat{s})\\a \in \psi_s^{-1}(\hat{a})\\s' \in \phi^{-1}(\hat{s}')}} 
    p(s,a,r,s') \cdot \log(p(s,a,r,s'))
    &\geq 
    -\hat{p}(\hat{s},\hat{a},r,\hat{s}') \cdot \log(\hat{p}(\hat{s},\hat{a},r,\hat{s}'))
    \label{eq:pproof}
\end{align}

As Eq.~\ref{eq:pproof} holds for \emph{every} abstract transition, the inequality also holds for the sum over all transitions: 

\begin{align}
    -
    \sum_{\hat{s},\hat{a},r,\hat{s}'} 
    \sum\limits_{\substack{s \in \phi^{-1}(\hat{s})\\a \in \psi_s^{-1}(\hat{a})\\s' \in \phi^{-1}(\hat{s}')}} 
    p(s,a,r,s') \cdot \log(p(s,a,r,s'))
    &\geq 
    - \sum_{\hat{s},\hat{a},r,\hat{s}'} 
    p(\hat{s},\hat{a},r,\hat{s}') \cdot \log(p(\hat{s},\hat{a},r,\hat{s}'))
    \label{eq:entropy_proof}
\end{align}

These are the transition-entropy for the abstract MDP (Eq.~\ref{eq:entropy_1}) and the MDP that is an homomorphic image of the abstract MDP (Eq.~\ref{eq:entropy_2}).
Therefore Eq.~\ref{eq:entropy_proof} results in

\begin{align}
    H(p(s,a,r,s')) &\geq H(p(\hat{s},\hat{a},r,\hat{s}')).
\end{align}

This concludes the proof.
\end{proof}

\clearpage
\subsection{Analyzing AMDPs}
\label{sec:amdp_overview}

This section provides an illustrative overview of the properties of abstract MDPs (AMDPs). 

\subsubsection{Common AMDP Transformations}

This section provides an illustrative overview on the core properties of AMDPs. 
\FloatBarrier
\begin{figure}[H]
    \centering
    \includegraphics[width=0.3\textwidth]{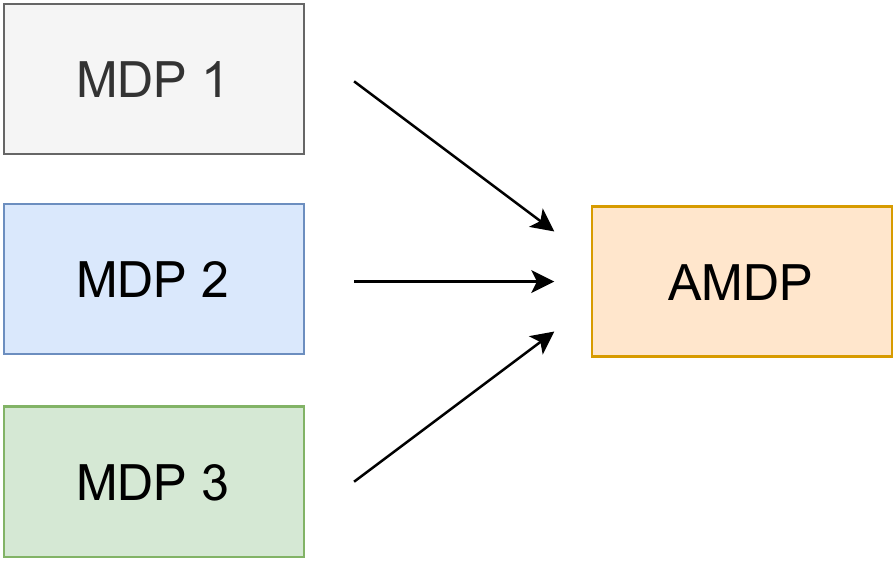}
    \caption{Schematic depiction of three homomorphic images of an abstract MDP (AMDP).}
    \label{fig:original}
\end{figure}

\begin{figure}
    \centering
    \includegraphics[width=1\textwidth]{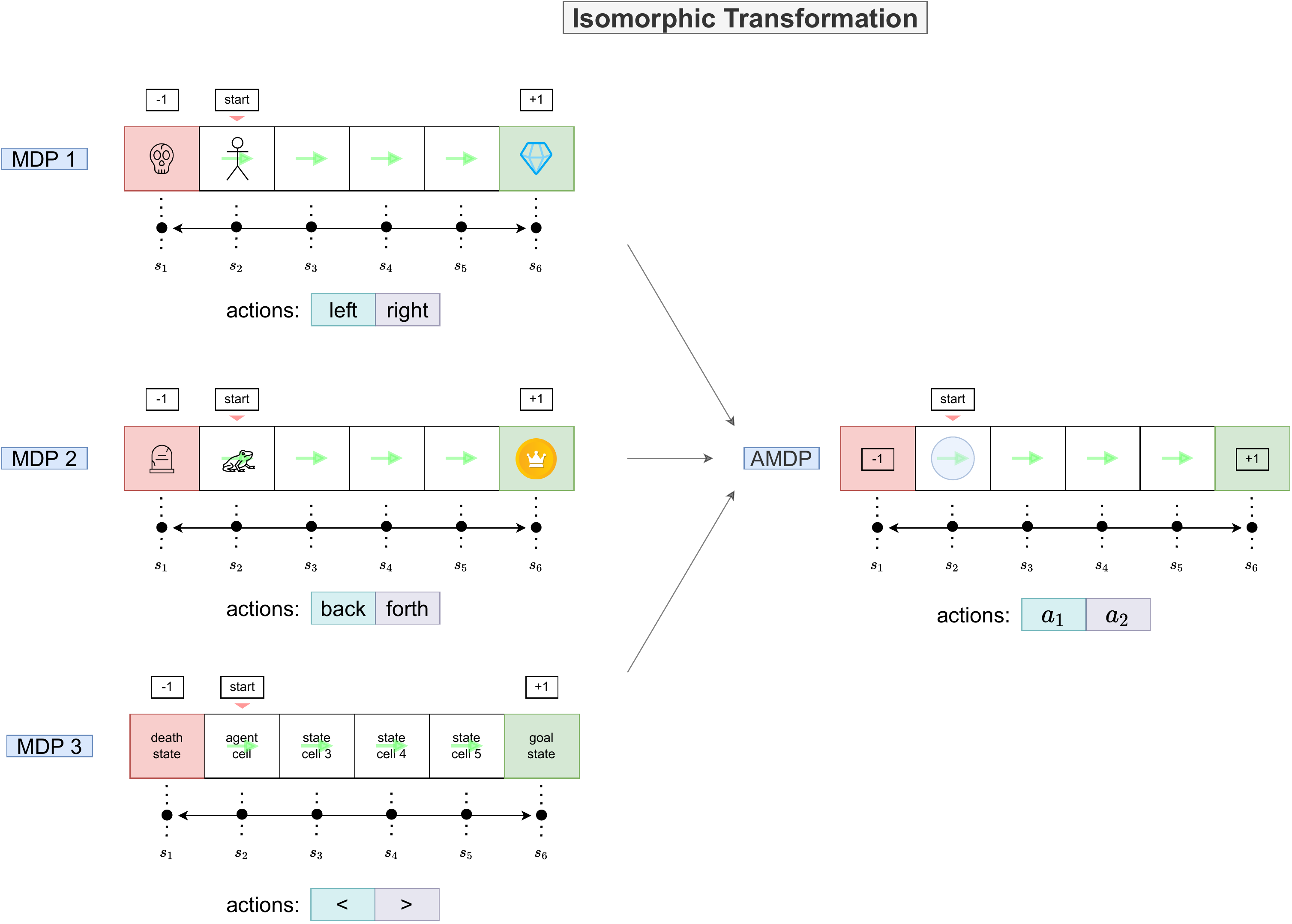}
    \caption{An example of three homomorphic images of an abstract MDP (AMDP). MDP 1 to 3 are distinct in their state and/or action space.}
    \label{fig:iso_transform_appendix}
\end{figure}

\begin{figure}
    \centering
    \includegraphics[width=1\textwidth]{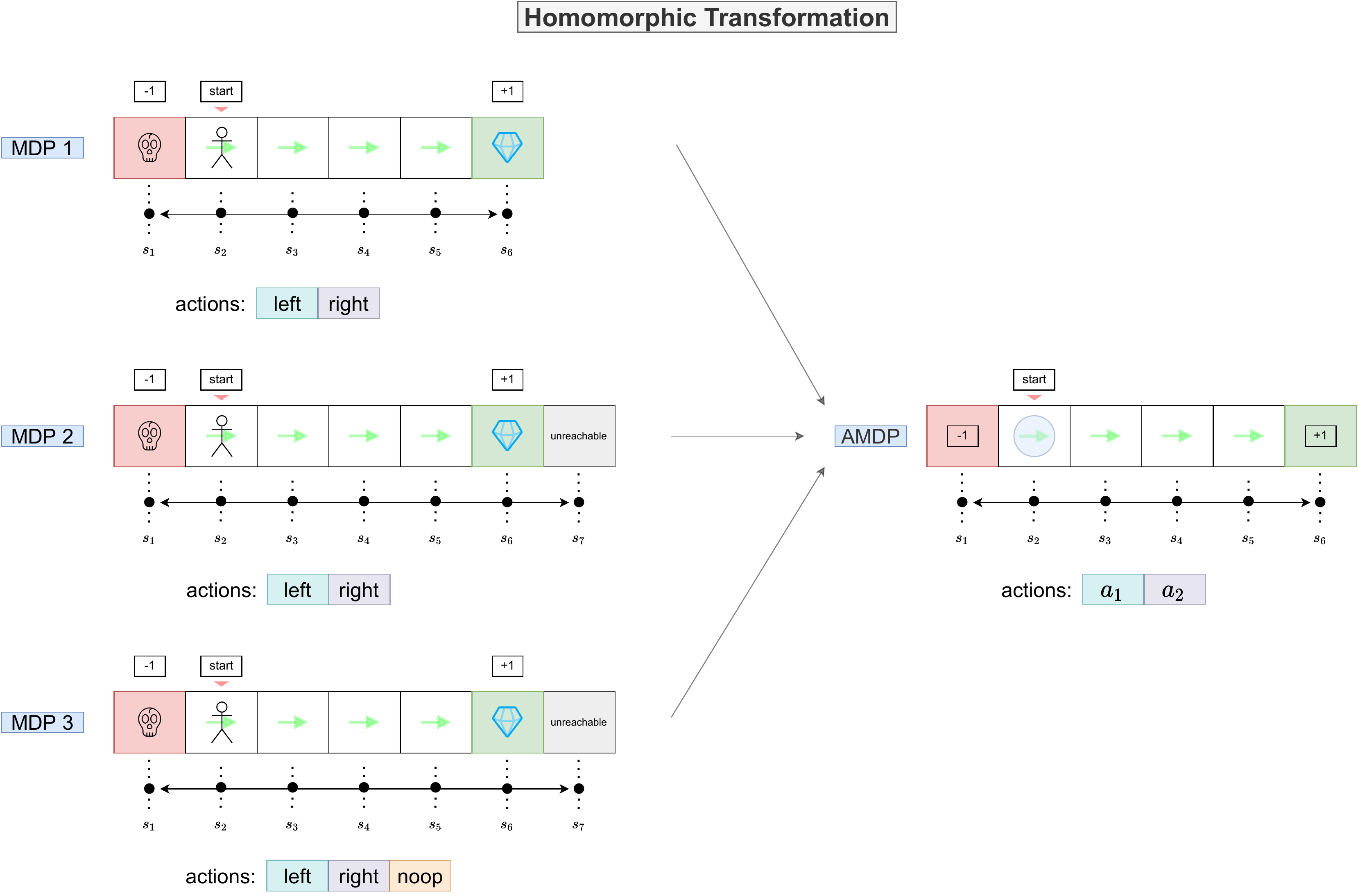}
    \caption{An example of three homomorphic images of an abstract MDP (AMDP). All MDPs are assumed to have the same state-action support. Their realization only show minor differences, e.g. in different symbols for the agent or the goal.}
    \label{fig:homo_transform_appendix}
\end{figure}

\begin{figure}
    \centering
    \includegraphics[width=1\textwidth]{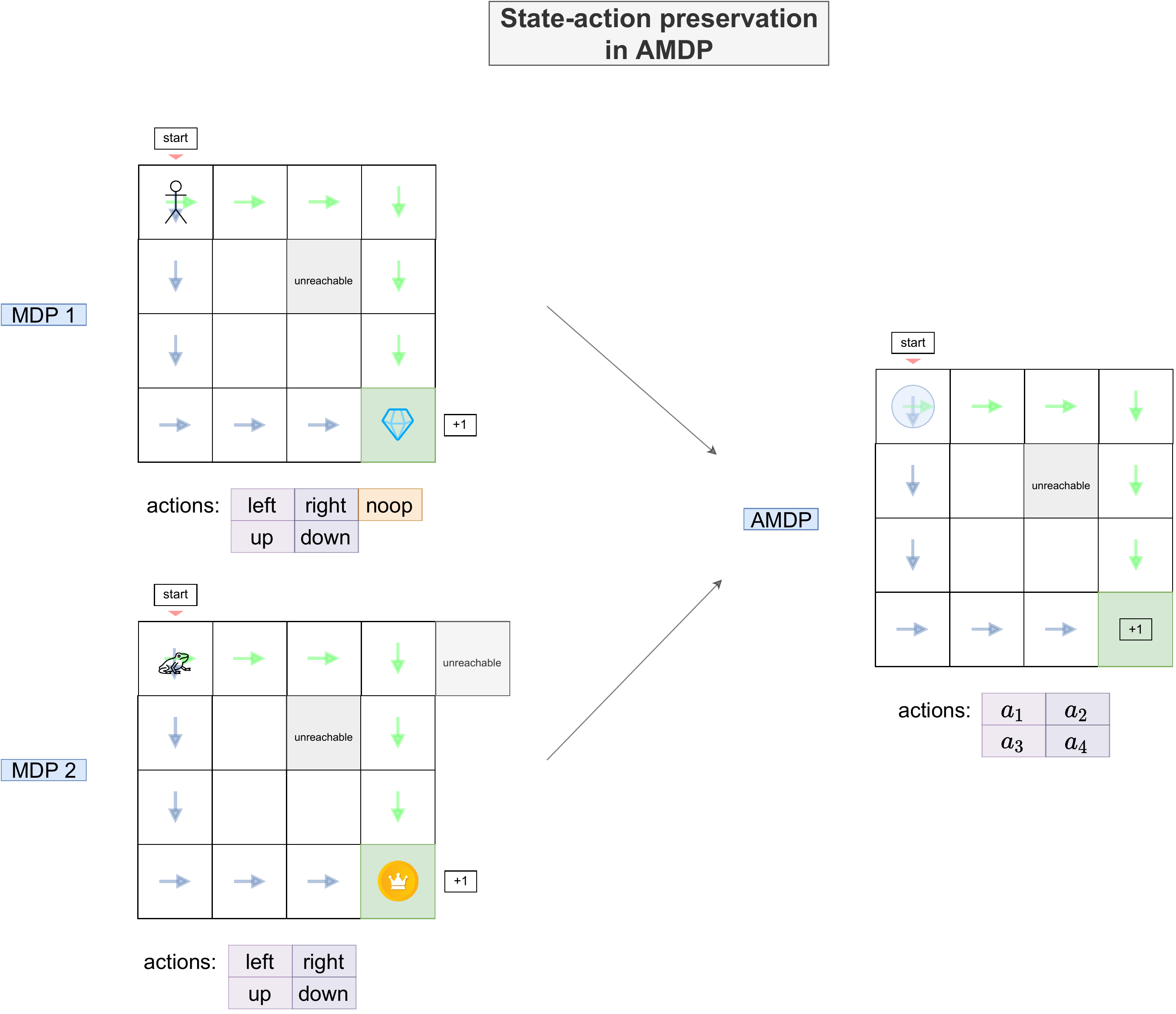}
    \caption{An example of state-action preservation for two policy paths given two MDPs. MDP 1 and 2 are homomorphic images of an abstract MDP (AMDP). The AMDP preserves the paths from all its homomorphic images.}
    \label{fig:amdp_dist_shift_overview}
\end{figure}

\clearpage
\subsubsection{AMDPs with domain shift}

This section provides an illustrative overview of the core properties of abstract MDPs (AMDPs) with domain shift.  
 
\begin{figure}
    \centering
    \includegraphics[width=0.5\textwidth]{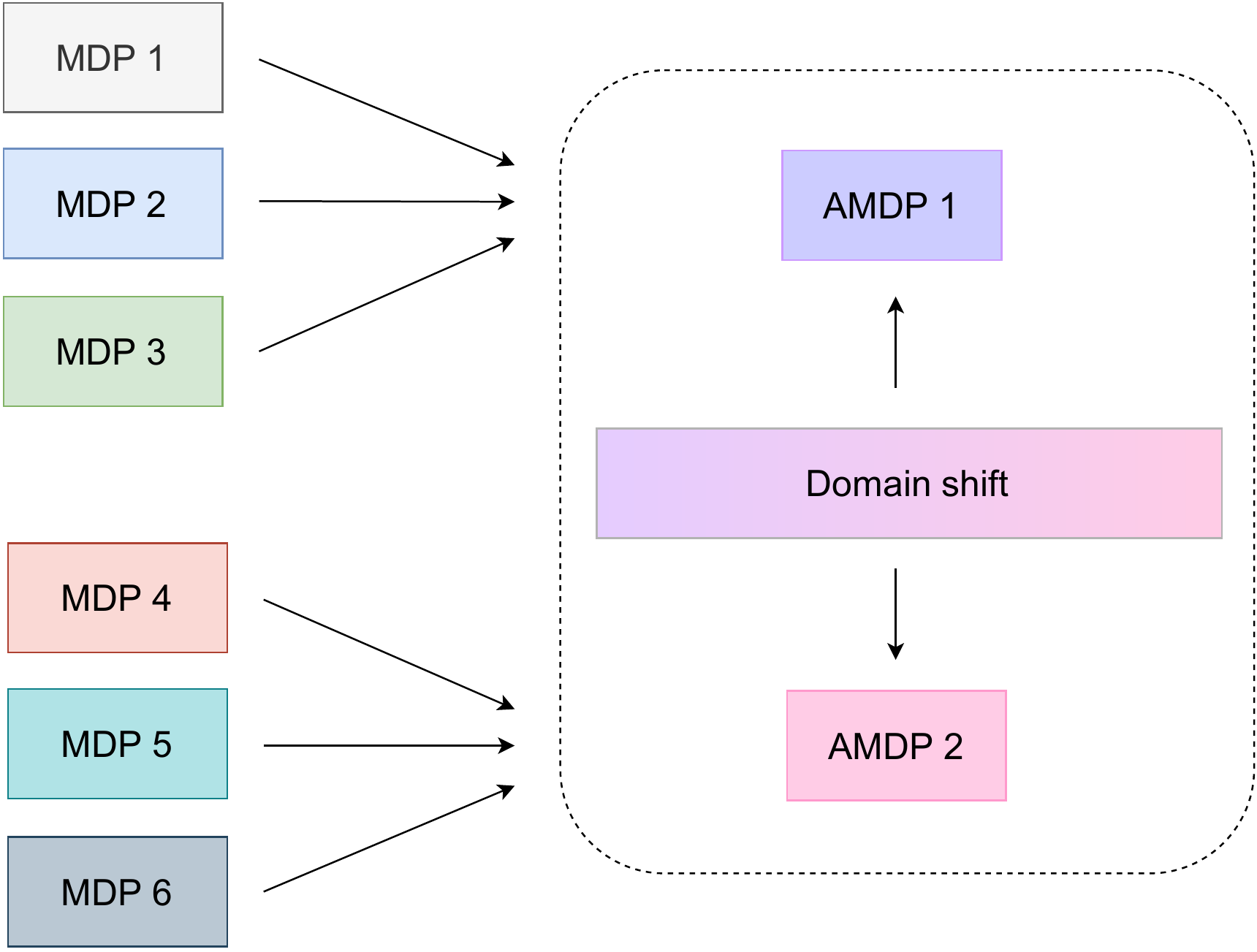}
    \caption{We show six MDPs that are mapped to their respective abstract MDP (AMDP). The AMDPs are assumed to have the same state-action support, and are realized as domain shifted transformations of each other. 
Under such transformation, the dynamics of the MDPs may differ and therefore abstract policies succeeding in one AMDP may fail in the other.}
    \label{fig:dist_shift_damdp}
\end{figure}

\begin{figure}
    \centering
    \includegraphics[width=.85\textwidth]{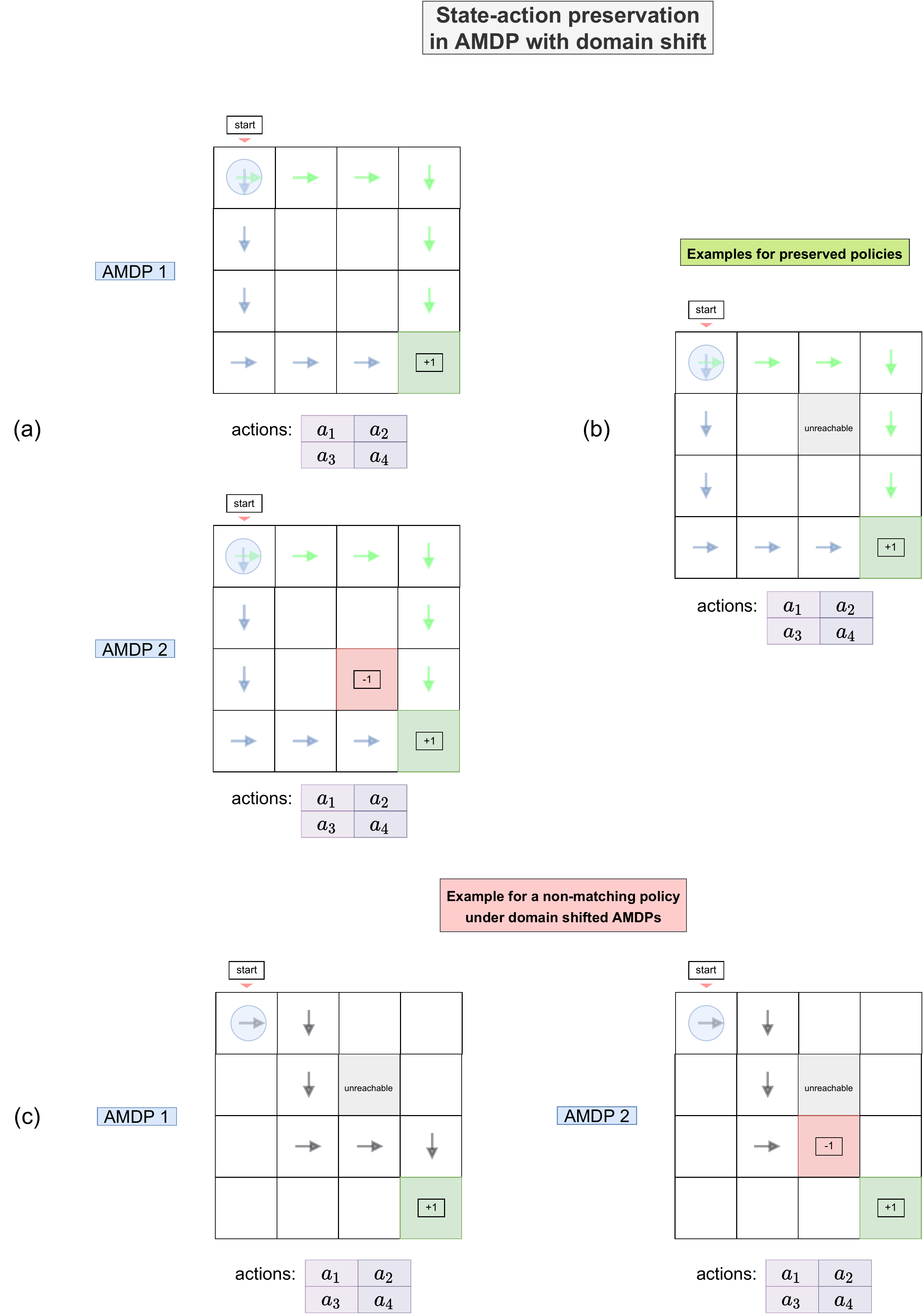}
    \caption{(a) We consider two abstract MDPs (AMDPs) with domain shift between them, and with two policy paths. (b) In this example, we see that the two defined AMDPs  partially preserve the state-action paths. (c) In this example, we see on the same AMDPs conflicting policy paths, since both AMDPs have different policies.}
    \label{fig:amdp_domain_shift_state-action_preservation}
\end{figure}

\begin{table}
\centering
\begin{tabular}{c}
\includegraphics[width=0.5\textwidth]{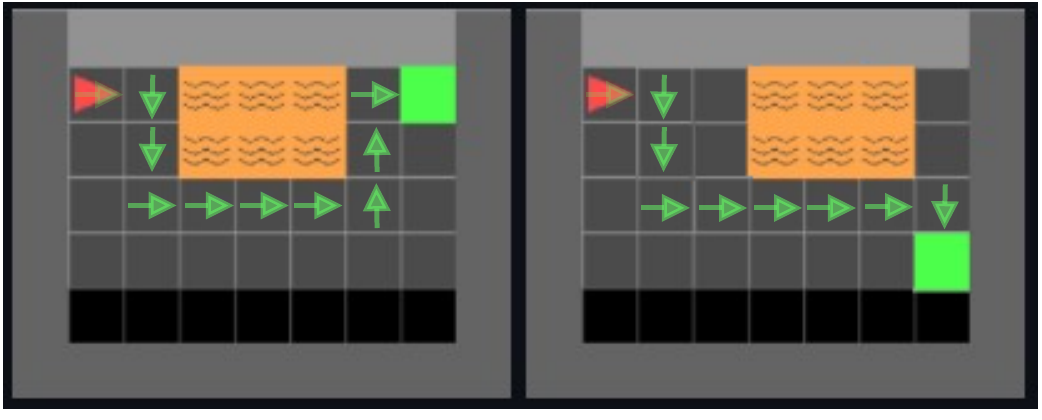}
\end{tabular}\\
\begin{tabular}{cc}
\includegraphics[width=0.42\textwidth]{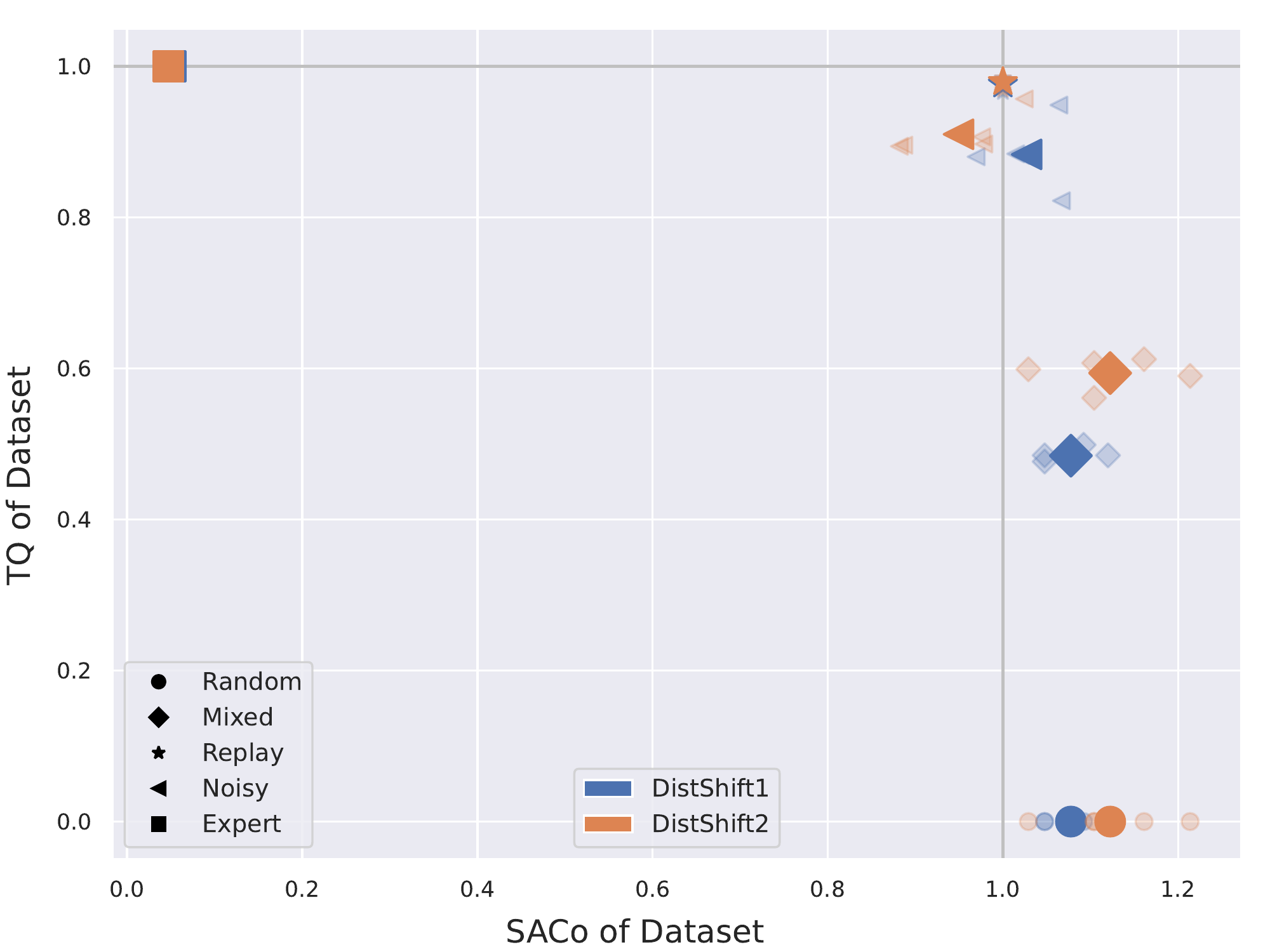} & \includegraphics[width=0.58\textwidth]{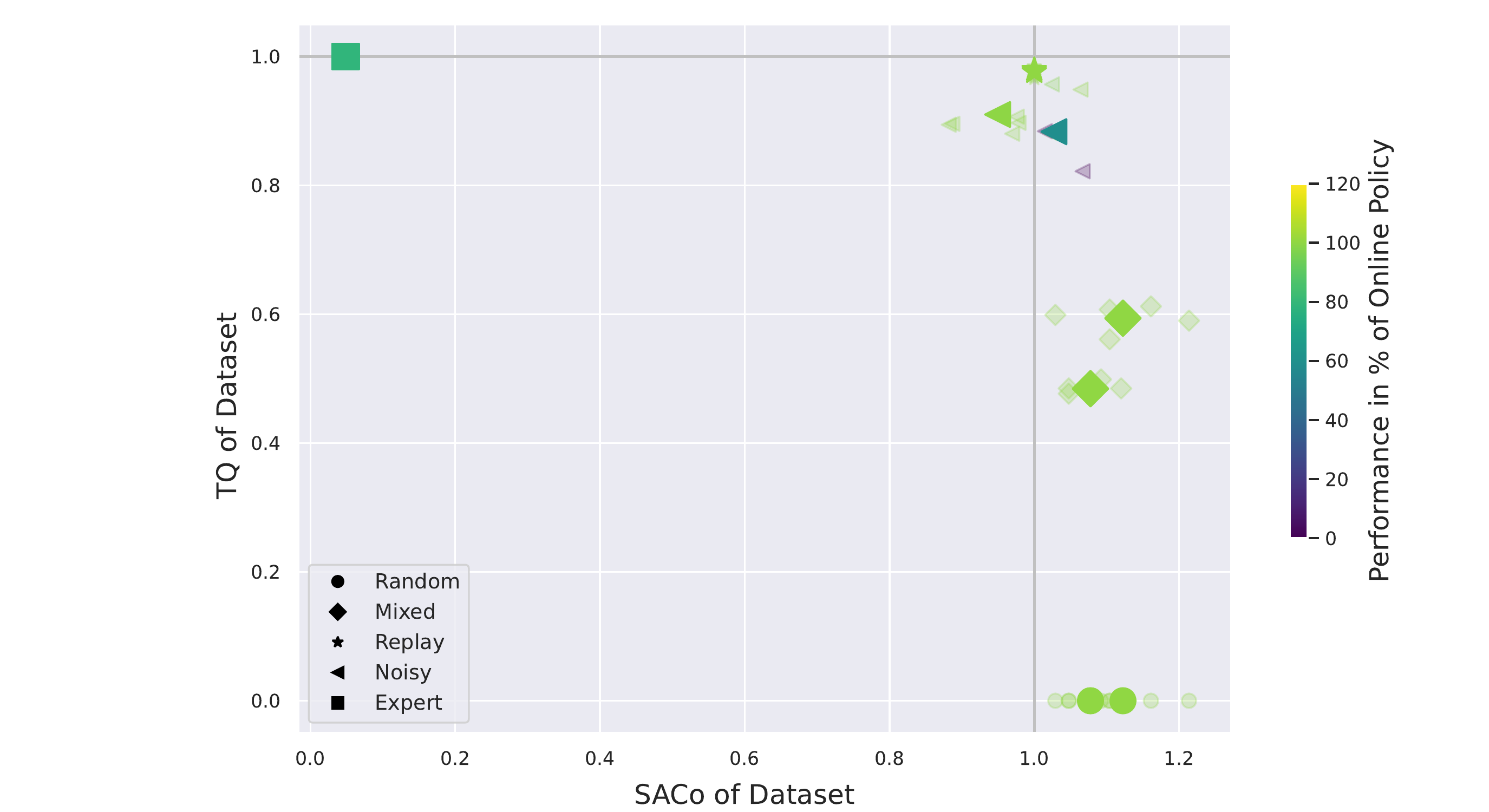} \\
\end{tabular}
\captionof{figure}{In this experiment we evaluate MiniGrid with two settings, the original MDP and a domain shifted version of the orignal MDP (top, left: original, and top, right: domain shifted). In the lower left figure, we see SACo and TQ for the two generated datasets of the MDPs (blue: original, orange: domain shifted; 5 seeds each). The generated datasets measures are similar since the dynamics are similar, and therefore, are slightly shifted versions of each other. On the lower right, we see the performance of offline trained DQN on the generated datasets. We observe that the performances match almost everywhere, with slight deviation of DistShift1, which is to be expected, since it is the slightly more difficult environment. In DistShift1 the agent has to longer paths next to the lava to get toward the goal state in the corner.}
\label{fig:ablations_1}
\end{table}

\clearpage
\subsubsection{Assessing SACo in continuous state spaces}
\label{sec:state_action_continous}

\begin{table}[H]
\centering
\begin{tabular}{c}
\includegraphics[width=0.45\textwidth]{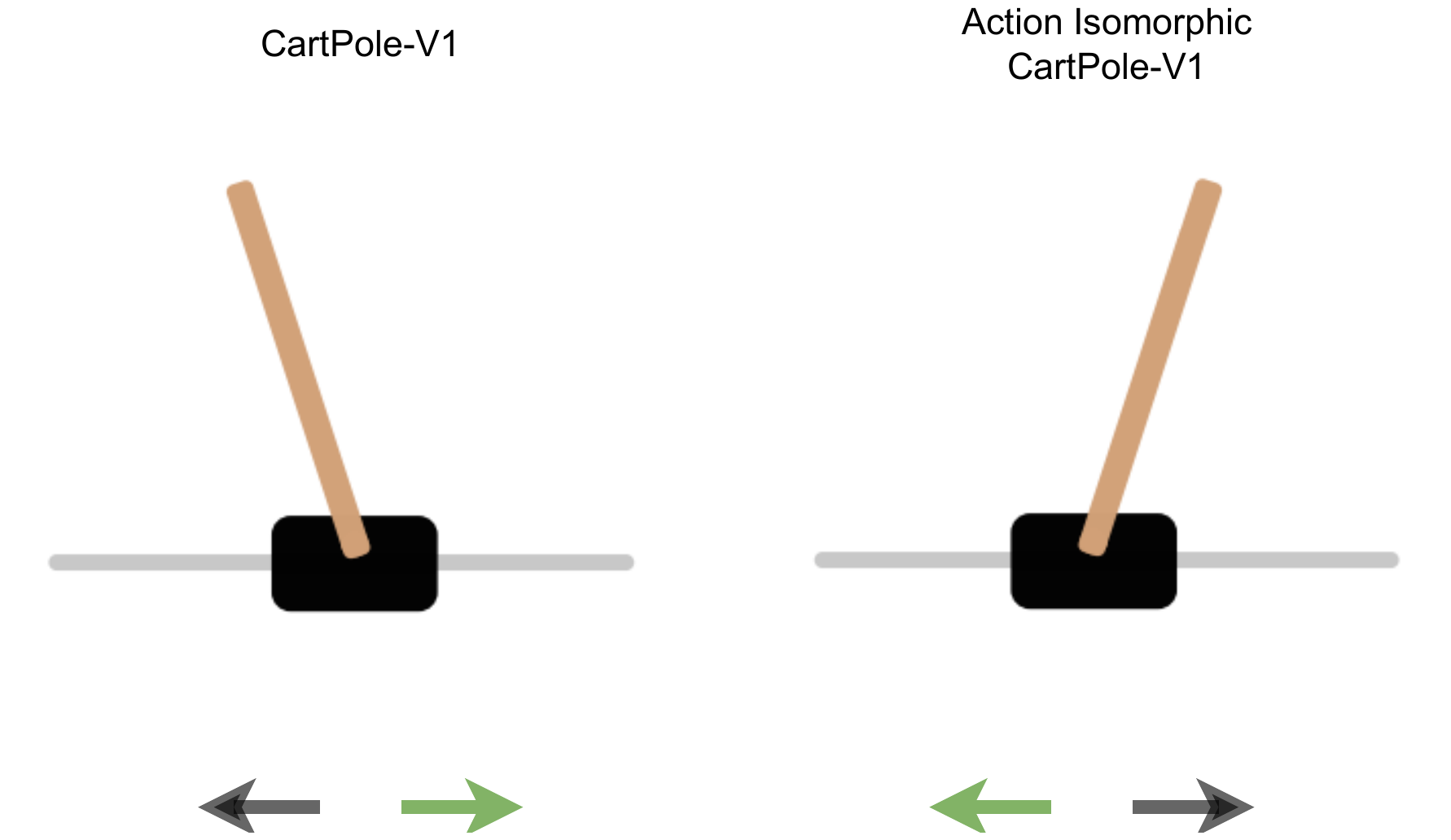}
\end{tabular}\\
\begin{tabular}{c}
\includegraphics[width=0.55\textwidth]{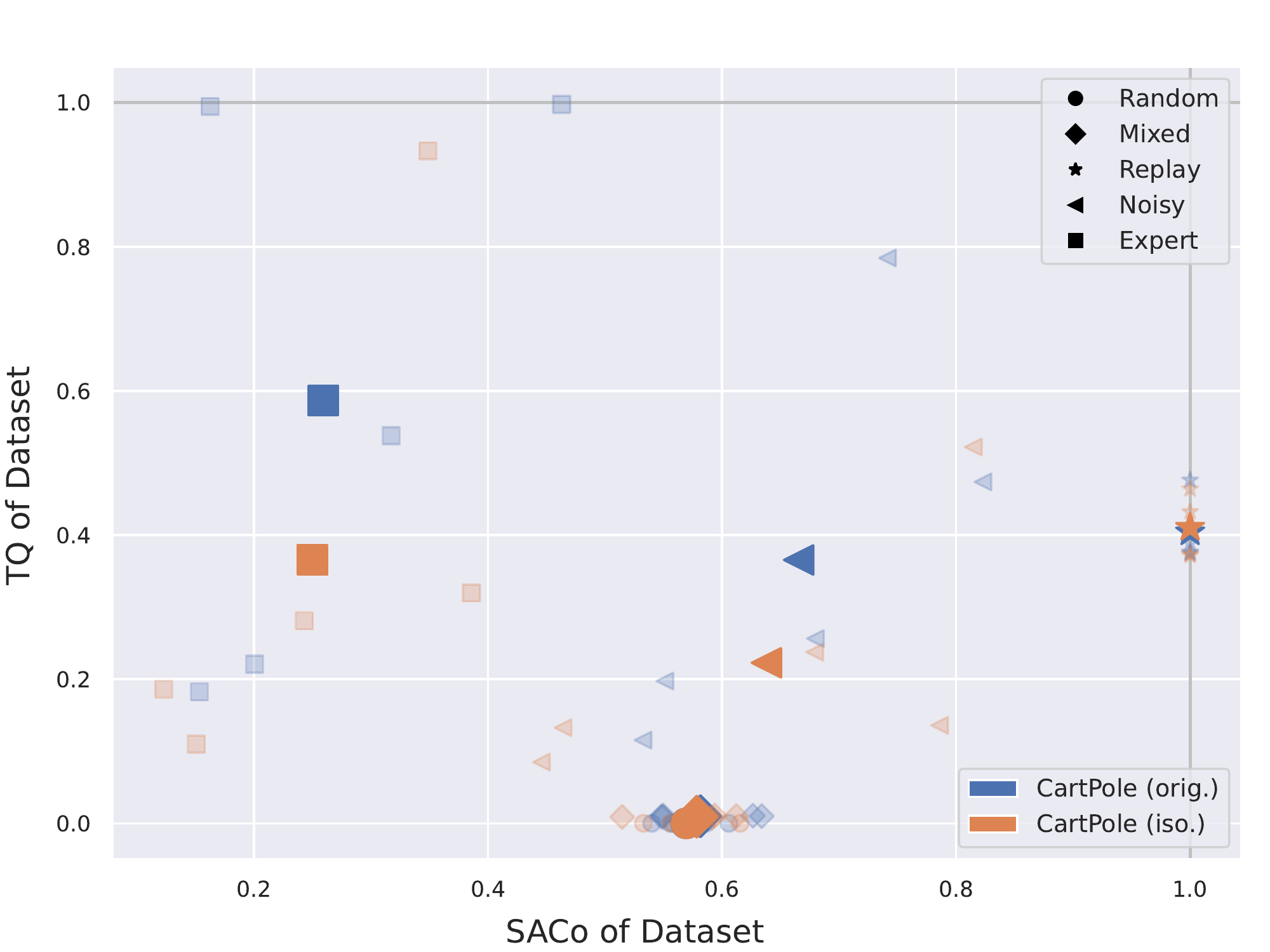} \\
\end{tabular}
\captionof{figure}{In this experiment we have evaluated CartPole with an isomorphically transformed MDP that flips the actions. We conducted this experiment to assess the stability of our measures, even for continuous states. Although CartPole has a continuous state space, the values for TQ and SACo are concentrated in distinctive regions (though with high variance), which we have observed also in experiments with discrete state and action spaces. We evaluated five seeds for each MDP.}
\label{fig:state_action_continuous}
\end{table}

\clearpage
\subsubsection{Assessing DQN vs QRDQN dataset collection properties}
\label{sec:sym_state_action_continous}

\begin{figure}[H]
    \centering
    \includegraphics[width=0.55\textwidth]{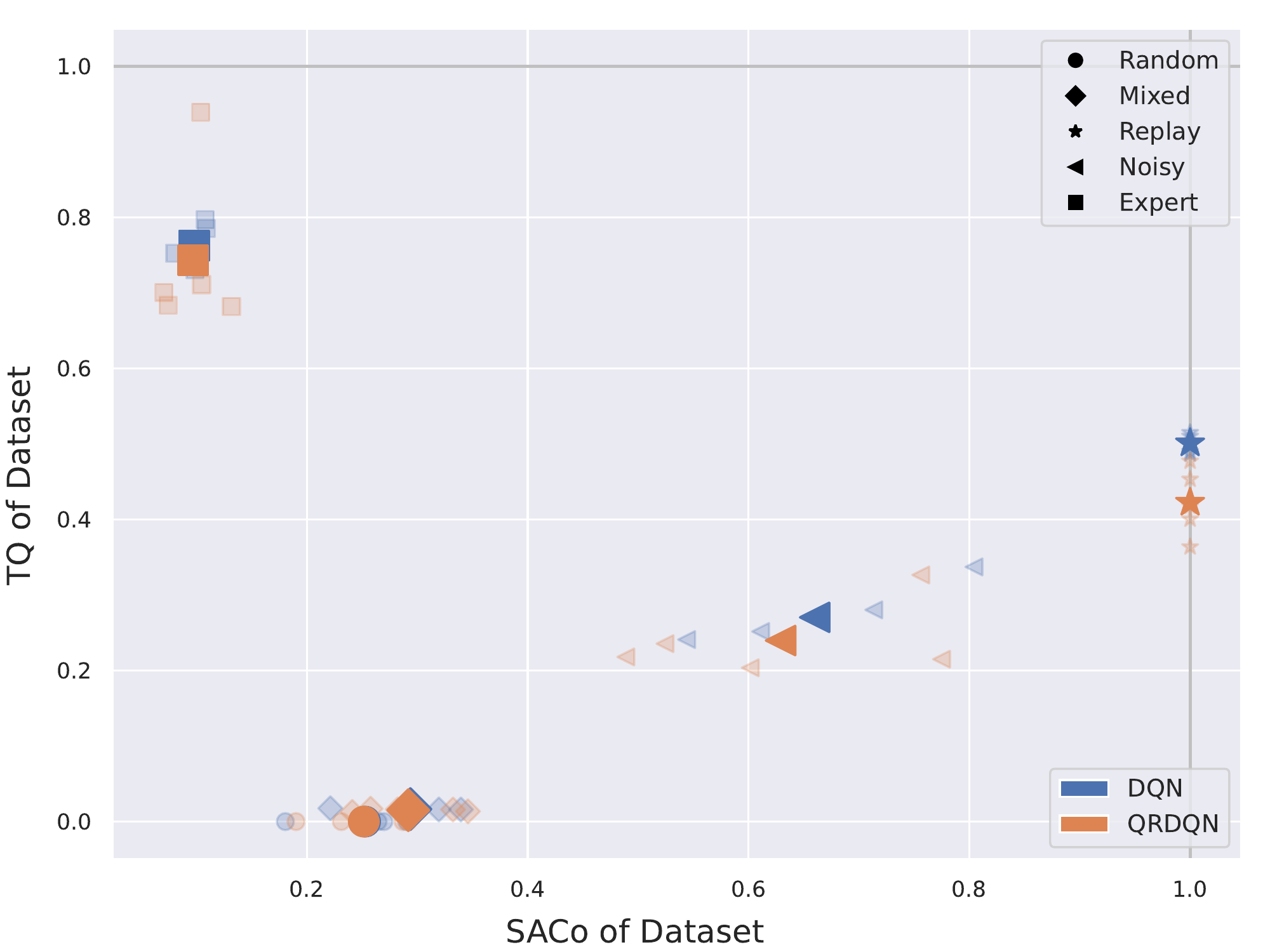}
    \caption{In this experiment, we also considered evaluating DQN to contrast it with QRDQN when generating the datasets. We see that both methods are projected to similar cluster regions. This result is expected since DQN and QRDQN are closely related algorithms. We evaluated five seeds for each DQN variant.}
    \label{fig:dist_shift_similar_structure}
\end{figure}

\subsection{Empirical Evaluations of Domain Shifts in Different MDP Settings} \label{sec:results_mdp_relaxations}

We investigate our implemented measures, TQ and SACo, regarding their properties under domain shifts.
Specifically, we are interested in the stability of TQ and SACo and whether they indicate domain shifts in the underlying datasets.
Our experimental setup utilizes the dataset generation schemes described in Sec.~\ref{sec:dataset_generation}.
We create datasets using these schemes on a total of six environments, that are transformations (isomorphic, homomorphic) of each other, or exhibit different domain shifts.
The results are presented in Fig.~\ref{fig:ablations_2}.

We see, that on the same environment, the TQ and SACo of different types of behavioral policies (random, expert, ...) result in different clusters of datasets.
This matches the intuition that changing the policy, thus introducing a policy shift, changes the dataset distribution drastically.
How prevalent the changes are becomes apparent, when comparing the results between different MDP settings.
The locations of clusters, representing the same behavioral policy types, change only minor when switching between the original MDP (Breakout) to isomorphic and homomorphic transformations of it.
Furthermore, we present results on general domain shifts, where we are not assured that TQ and SACo indicate the domain shift, since the joint probability distribution and thus both the policy and the environment changes.

\begin{table}[h]
\centering
\begin{tabular}{ccc}
\begin{tabular}[c]{@{}c@{}}Same MDP\\(Breakout)\end{tabular}& 
\begin{tabular}[c]{@{}c@{}}Isomorphically transformed \\ MDP (Breakout)\end{tabular} & 
\begin{tabular}[c]{@{}c@{}}Homorphically transformed \\ MDP (Breakout)\end{tabular} \\
\cr
\includegraphics[width=0.15\textwidth]{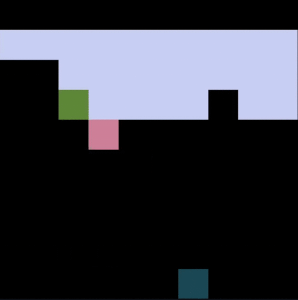} & \includegraphics[width=0.15\textwidth]{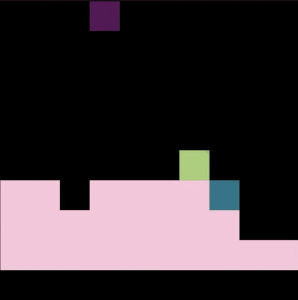} & \includegraphics[width=0.15\textwidth]{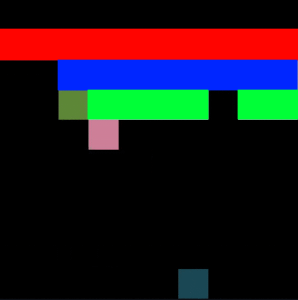} \\
\cr
\includegraphics[width=0.32\textwidth]{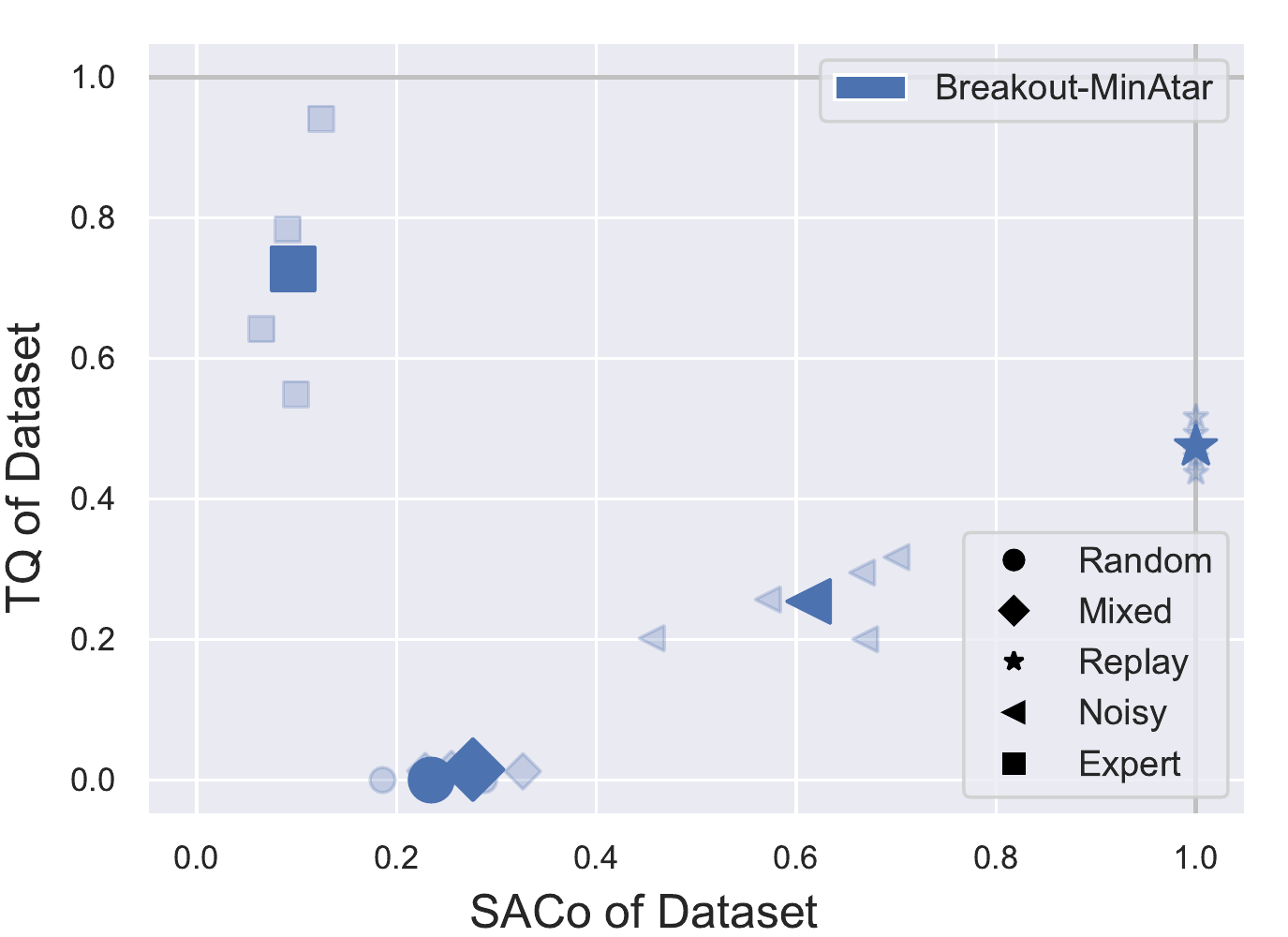} & \includegraphics[width=0.32\textwidth]{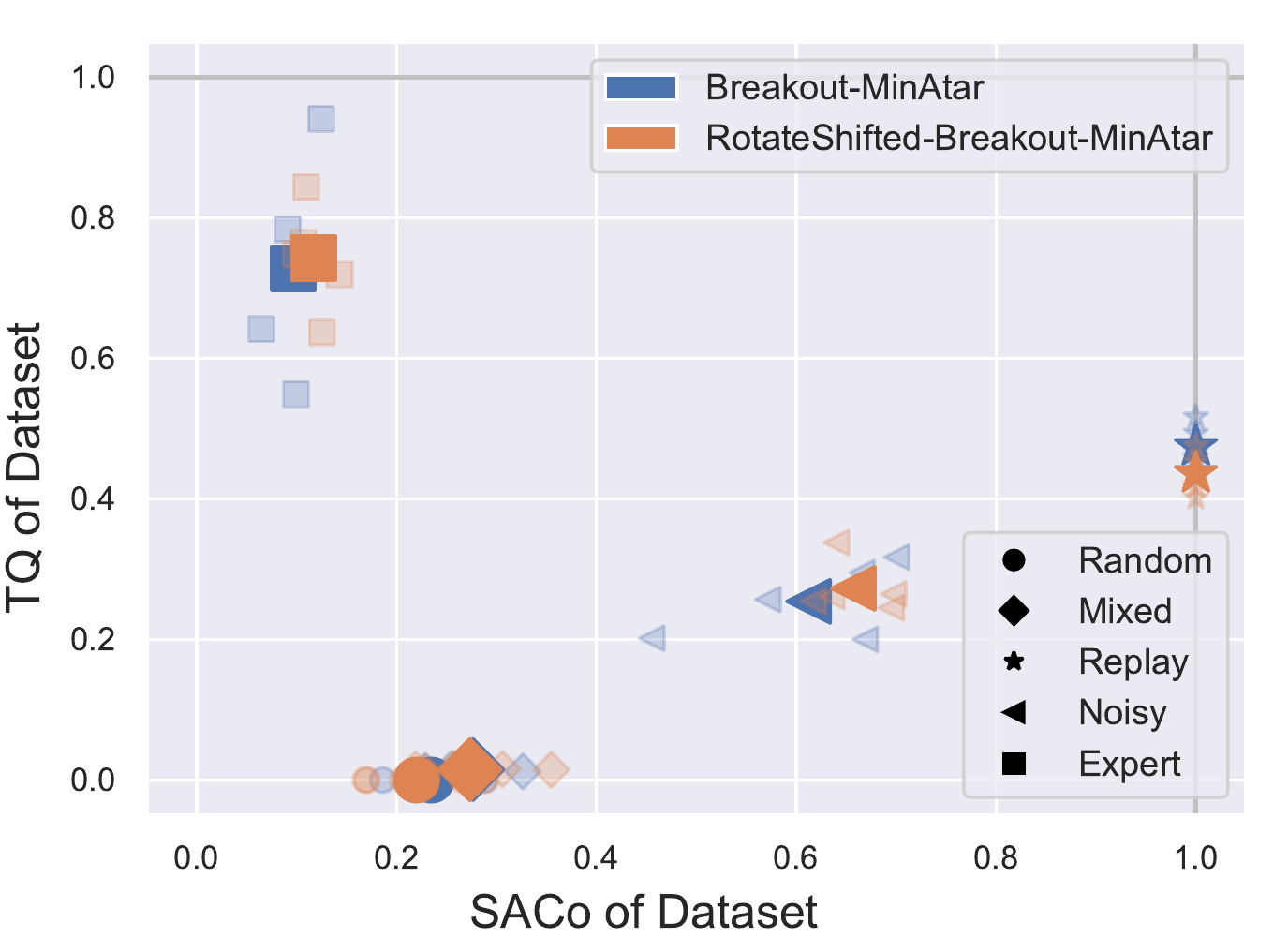} & \includegraphics[width=0.32\textwidth]{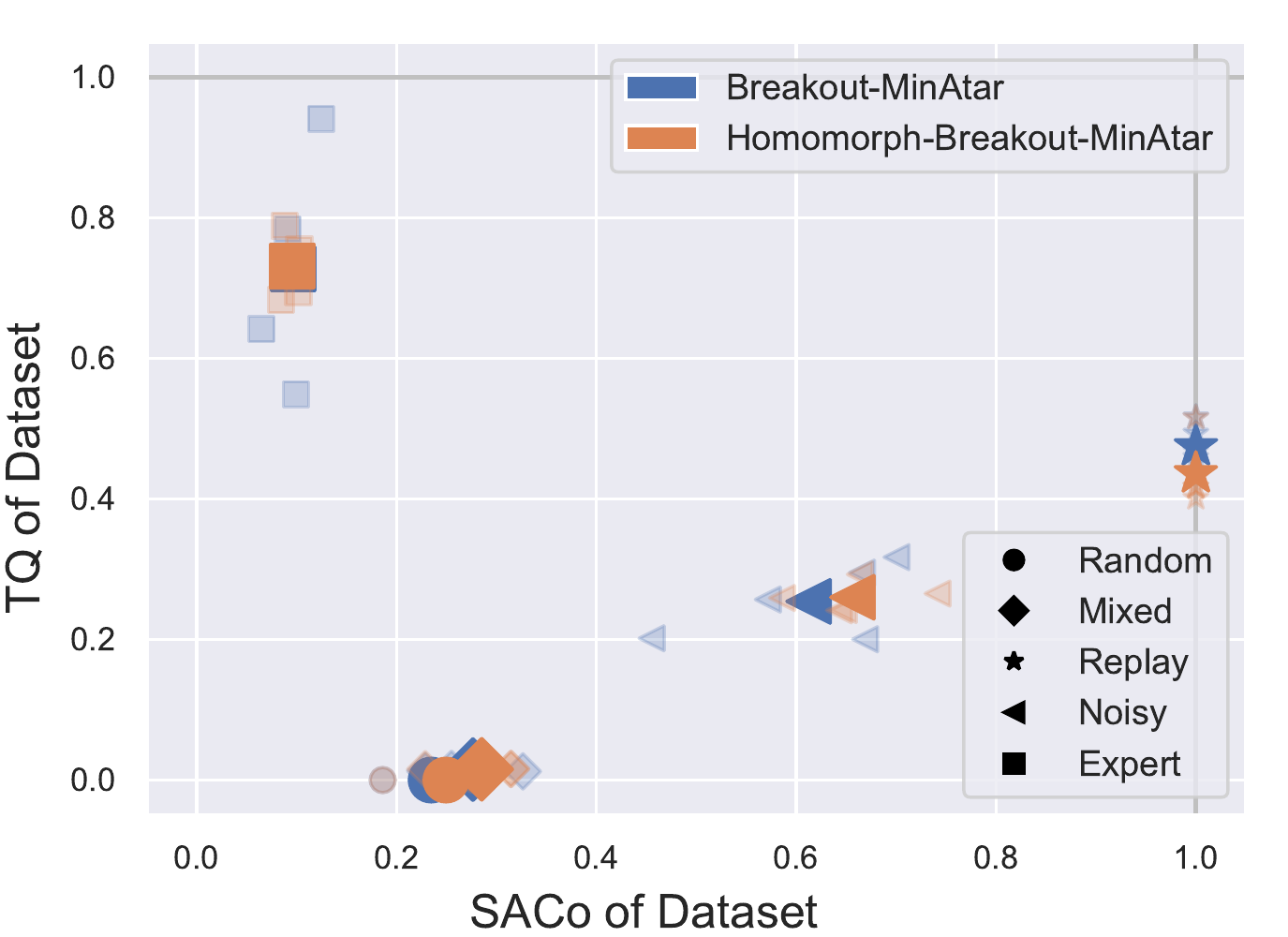}
\cr
\cr
\cr
\begin{tabular}[c]{@{}c@{}}Same MDP with\\ shift in state-occupancy (Breakout)\end{tabular}
 & \begin{tabular}[c]{@{}c@{}}Different MDP (also from MinAtar) with\\ same state-action support \& \\general domain shift (Space Invaders)\end{tabular} 
 & \begin{tabular}[c]{@{}c@{}}Different MDP (Minigrid)\\from different environment suite\end{tabular} \\
 \cr
\includegraphics[width=0.15\textwidth]{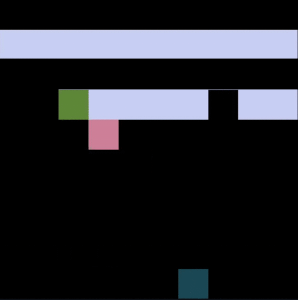} & \includegraphics[width=0.15\textwidth]{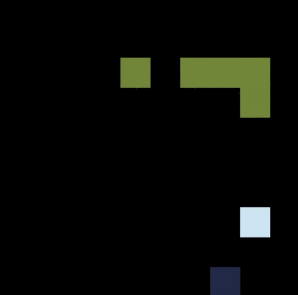} & \includegraphics[width=0.15\textwidth]{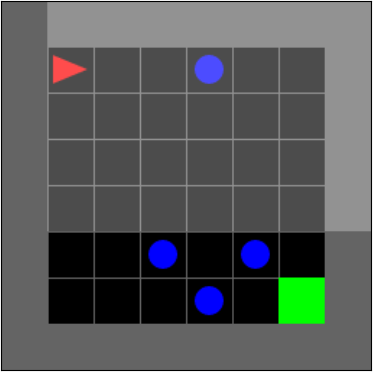} \\
\cr
\includegraphics[width=0.32\textwidth]{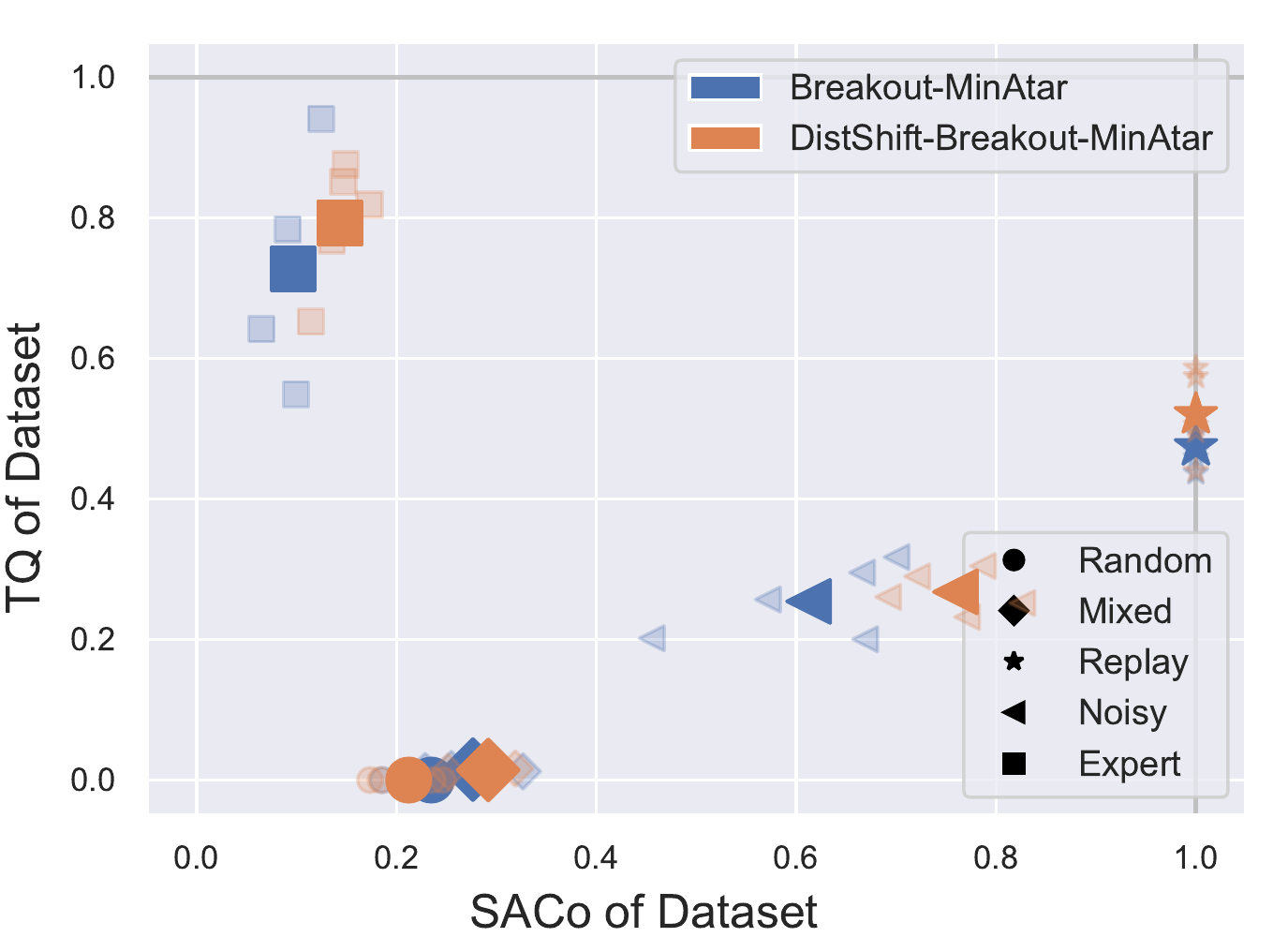} & \includegraphics[width=0.32\textwidth]{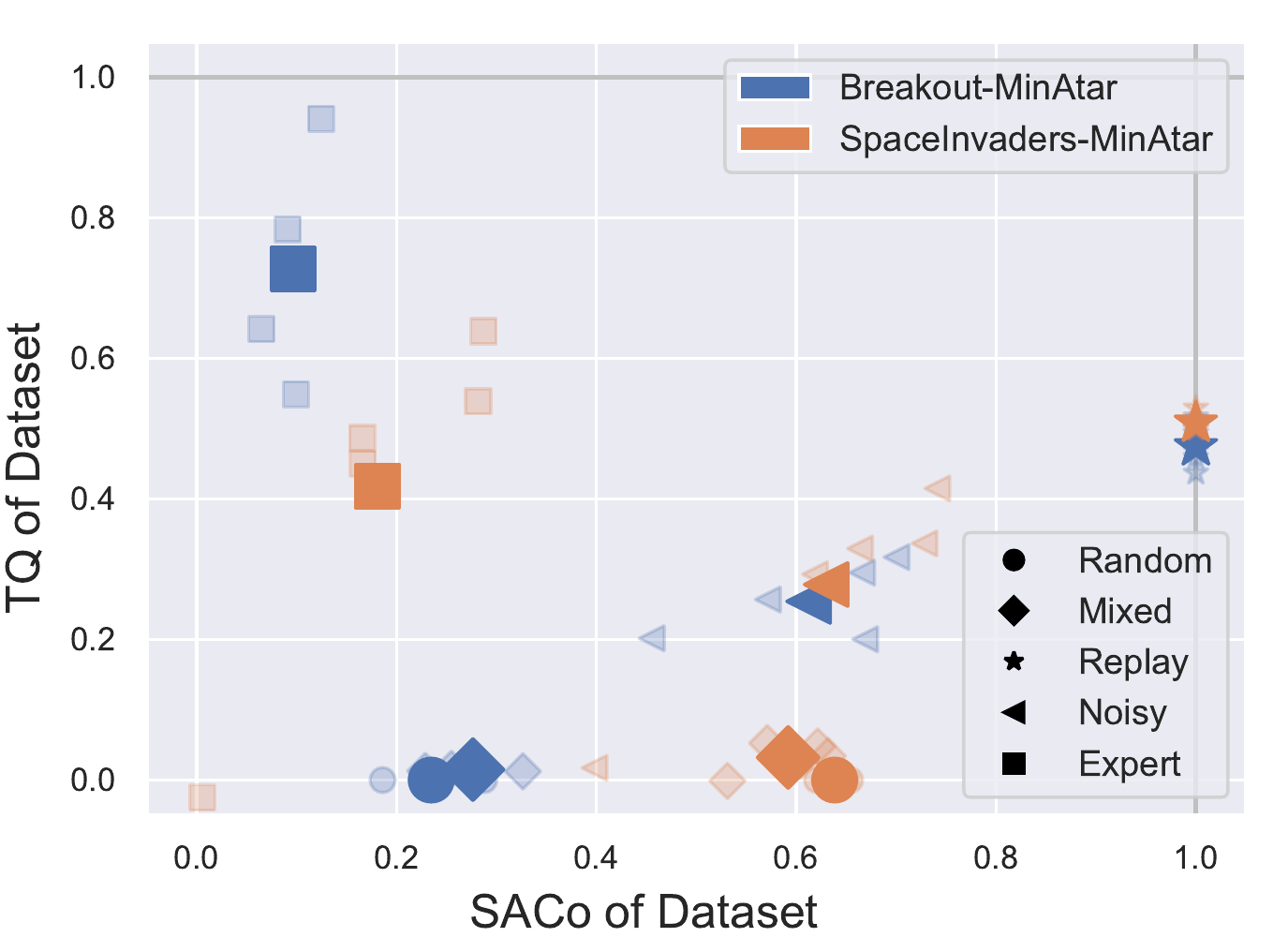} & \includegraphics[width=0.32\textwidth]{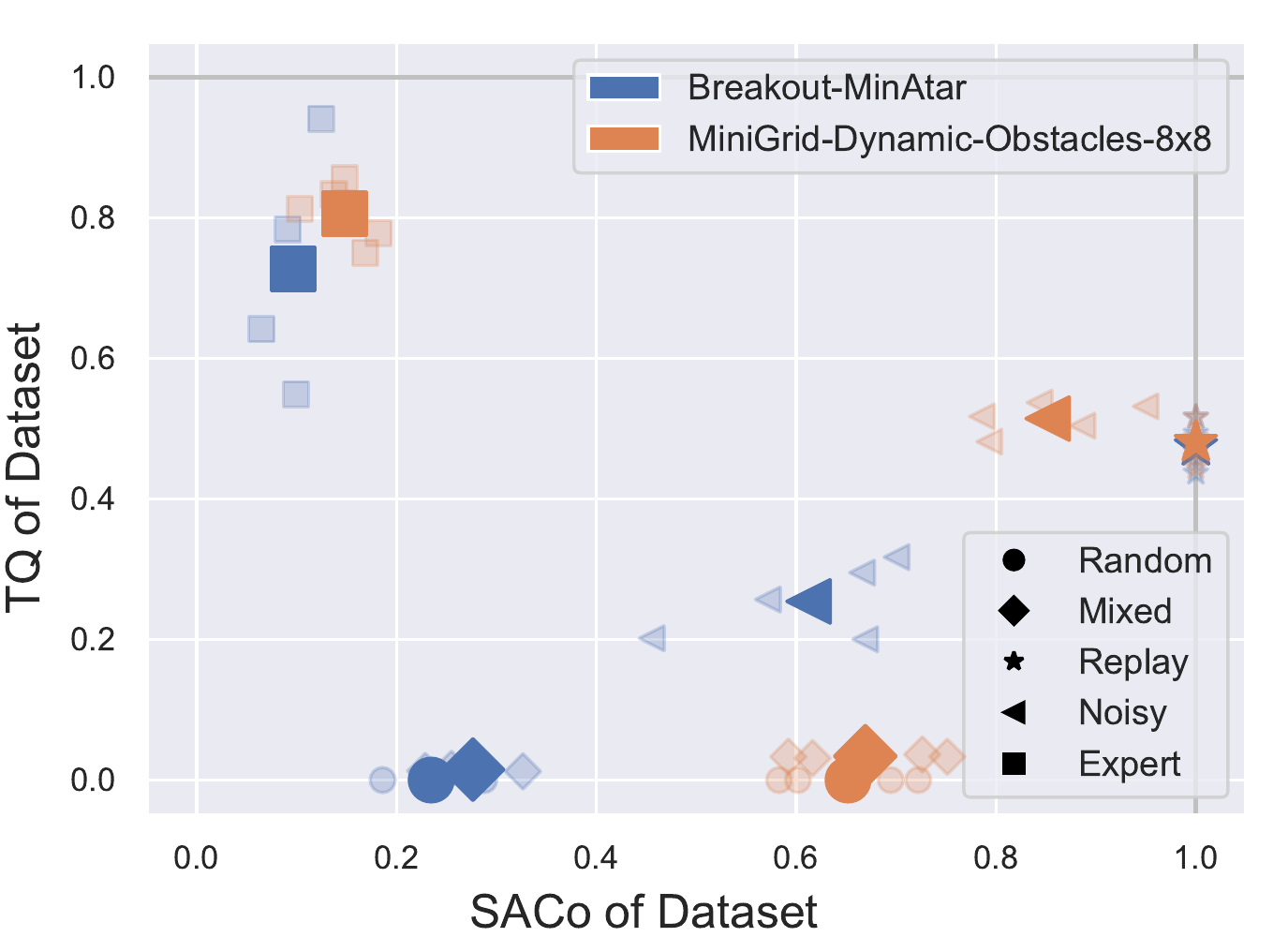}
\end{tabular}
\captionof{figure}{We show TQ and SACo of datasets sampled from different behavioral policies in different MDP settings. 
Sampling from different behavioral policies corresponds to a policy shift, which is expected to result in different values for TQ and SACo.
However, we also show that small deviations between datasets may occur by sampling with different seeds, even with the same type of behavioral policy. 
Breakout (upper left) is the reference environment all other environments are compared to.
Environments in the first row exhibit isomorphic and homomorphic transformations.
As expected, datasets sampled from the same behavioral policies are very similar in terms of TQ and SACo and deviate only slightly under isomorphic and homomorphic transformations of the environment.
Environments in the second row exhibit domain shifts and reflect different MDP settings, which results in stronger deviations between datasets even among the same behavioral policies.}
\label{fig:ablations_2}
\end{table}

\clearpage
\subsection{Environments}
\label{sec:environments}

Although the dynamics of the environments used throughout the main experiments are rather different and range from motion equations to predefined game rules, they share common traits. 
This includes the dimension of the state $dim(\bm s)$, the number of eligible actions $|\mathcal{A}|$, the maximum episode length $T_{max}$ as well as the minimum and maximum expected return $g_{min}, g_{max}$ of an episode. 
Furthermore, the discount factor $\gamma$ is fixed for every environment regardless of the specific experiment executed on it and is thus listed with the other parameters in Tab.~\ref{tab:env_params}.
An overview of the environments, depicted by their graphical user interfaces is given in Fig.~\ref{fig:environments}.\\

Two environments contained in the MinAtar suite, Breakout and SpaceInvaders, do not have an explicit maximum episode length, as the episode termination is ensured through the game rules. 
Breakout terminates either if the ball is hitting the ground or two rows of bricks were destroyed, which results in the maximum of $60$ reward. 
An optimal agent could attain infinite reward for SpaceInvaders, as the aliens always reset if they are eliminated entirely by the player and there is a speed limit that aliens can maximally attain. 
Nevertheless, returns much higher than $200 - 300$ are very unlikely due to stochasticity in the environment dynamics that is introduced through sticky actions with a probability of $0.1$ for all MinAtar environments.\\

\begin{table}[h]
\centering
\begin{tabular}{lcccccc}
\hline
Environment           & $dim(\bm s)$ & $|\mathcal{A}|$ & $T_{max}$ & $g_{min}$ & $g_{max}$ & $\gamma$ \\ \hline
\texttt{CartPole-v1}              & $4$            & $2$               & $500$       & $9^*$        & $500$       & $0.95$     \\
\texttt{MountainCar-v0}           & $2$            & $3$               & $200$       & $-200$      & $-90^*$       & $0.99$     \\
\texttt{MiniGrid-LavaGapS7-v0}             & $98$           & $3$               & $196$       & $0$         & $0.945^*$     & $0.95$     \\
\texttt{MiniGrid-Dynamic-Obstacles-8x8-v0} & $98$           & $3$               & $256$       & $-1$        & $0.935^*$     & $0.95$     \\
\texttt{Breakout-MinAtar-v0}              & $100$          & $3$               & -         & $0$         & $60$        & $0.99$     \\
\texttt{SpaceInvaders-MinAtar-v0}       & $100$          & $4$               & -         & $0$         & $\infty$  & $0.99$     \\ \hline
\end{tabular}
\caption{Environment specific characteristics and parameters. $^*$Minimum or maximum expected returns depend on the starting state.}
\label{tab:env_params}
\end{table}

Action-spaces for MiniGrid and MinAtar are reduced to the number of eligible actions and state representations simplified. 
Specifically, the third layer in the symbolical state representation of MiniGrid environments was removed as it contained no information for the chosen environments. 
The state representation of MinAtar environments was collapsed into one layer, where the respective entries have been set to the index of the layer, divided by the total number of layers. 
The resulting two-dimensional state representations are flattened for MiniGrid as well as for MinAtar environments.

\begin{figure}[h]
    \centering
    \includegraphics[width=0.7\textwidth]{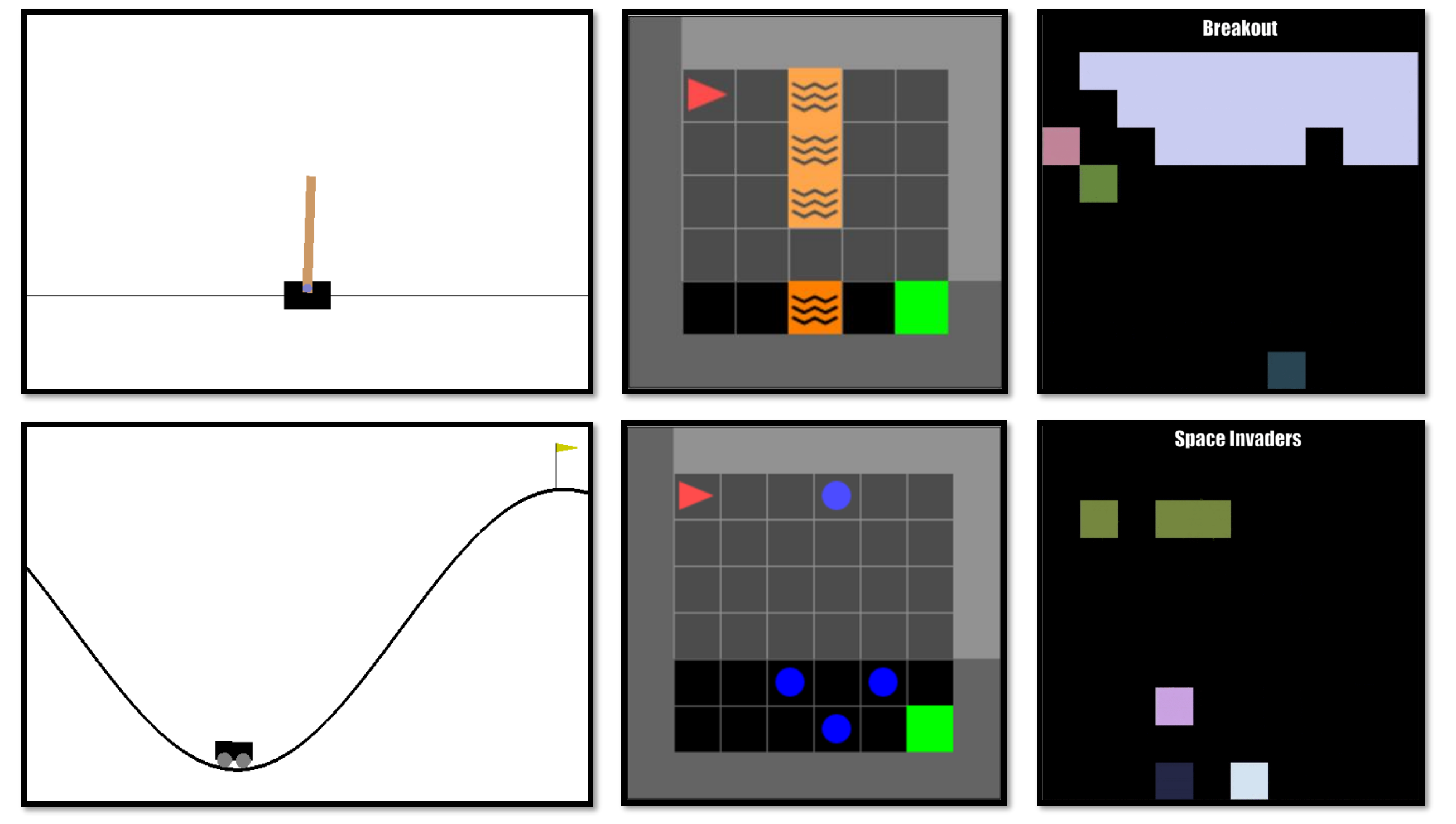}
    \caption{Graphical interfaces of environments used throughout the main experiments. Enumerating from top to bottom, left to right, these environments are \texttt{CartPole-v1}, \texttt{MountainCar-v0}, \texttt{MiniGrid-LavaGapS7-v0}, \texttt{MiniGrid-Dynamic-Obstacles-8x8-v0}, \texttt{Breakout-MinAtar-v0} and \texttt{SpaceInvaders-MinAtar-v0}.}
    \label{fig:environments}
\end{figure}

\subsection{Algorithms}
\label{sec:algorithms}

We conducted an evaluation of the different dataset compositions using nine different algorithms applicable in an Offline RL setting. 
The selection covers recent advances in the field, as well as off-policy methods not specifically designed for Offline RL that are often utilized for comparison.
Behavioral cloning (BC) \citep{Pomerleau:91} serves as a baseline algorithm, as it mimics the behavioral policy used to create the dataset. 
Consequently, its performance is expected to be strongly correlated with the TQ of the dataset.

Behavior Value Estimation (BVE) \citep{Gulcehre:21} is utilized without the ranking regularization that it was proposed to be coupled with. 
This way, extrapolation errors are circumvented during training as the action-value of the behavioral policy is evaluated.
Policy improvement only happens during inference, when the action is greedily selected on the learned action-values.
BVE uses SARSA updates where the next state and action are sampled from the dataset, utilizing temporal difference updates to evaluate the policy.\\
As a comparison, Monte-Carlo Estimation (MCE) evaluates the behavioral policy that created the dataset from the observed returns.
Again, actions are greedily selected on the action-values obtained from Monte-Carlo estimates.

Deep Q-Network (DQN) \citep{Mnih:13} is used to obtain the online policy, but can be applied in the Offline RL setting as well, as it is an off-policy algorithm.
The dataset serves as a replay buffer in this case, which remains constant throughout training.
As it is not originally designed for the Offline RL setting, there are no countermeasures to the erroneous extrapolation of action-values during training nor during inference.\\
Quantile-Regression DQN (QRDQN) \citep{Dabney:17} approximates a set of $K$ quantiles of the action-value distribution instead of a point estimate during training.
During inference, the action is selected greedy through the mean values of the action-value distribution.\\
Random Ensemble Mixture (REM) \citep{Agarwal:20} utilizes an ensemble of $J$ action-value approximations to attain a more robust estimate.
During training, the influence of each approximation on the overall loss is weighted through a randomly sampled categorical distribution.
Selecting an action is done greedy on the average of the action-value estimates.
Batch-Constrained Deep Q-learning (BCQ) \citep{Fujimoto:19a} for discrete action-spaces is based on DQN, but uses a BC policy on the dataset to constrain eligible actions during training and inference.
A relative threshold $\tau$ is utilized for this constraint, where eligible actions must attain at least $\tau$ times the probability of the most probable action under the BC policy.\\
Conservative Q-learning (CQL) \citep{Kumar:20} introduces a regularization term to  policy evaluation. 
The general framework might be applied to any off-policy algorithm that approximates action-values, therefore we based it on DQN as used for the online policy.
Furthermore, the particular regularizer has to be chosen, where we used the KL-divergence against a uniform prior distribution, referred to as CQL($\cH$) by the authors.
The influence of the regularizing term is controlled by a temperature parameter $\alpha$.\\
Critic Regularized Regression (CRR) \citep{Wang:20} aims to ameliorate the problem that the performance of BC suffers from low-quality data, by filtering actions based on action-value estimates.
Two filters which can be combined with several advantage functions were proposed by the authors, where the combination referred to as binary max was utilized in this study.
Furthermore, DQN is used instead of a distributional action-value estimator for obtaining the $m$ action-value samples in the advantage estimate.

\subsection{Implementation Details}
\label{sec:impl_details}

\subsubsection{Network Architectures}

The state input space is as defined in Tab.~\ref{tab:env_params}, followed by 3 linear layers with a hidden size of 256. The number of output actions for the final linear layer is defined by the number of eligible actions for action-value networks. For QRDQN and REM, the number of actions times the number of quantiles or estimators respectively is used as output size. All except the last linear layer use the SELU activation function \citep{klambauer:17} with proper initialization of weights, whereas the final one applies a linear activation. Behavioral cloning networks use the softmax activation in the last layer to output a proper probability distribution, but are otherwise identical to the action-value networks.

\subsubsection{Online Training}
\label{subsec:online_training}

For every environment, a single online policy is obtained through training with DQN. This policy is the one used to generate the datasets under the different settings described in Sec.~\ref{sec:dataset_generation}. All hyperparameters are listed in Tab.~\ref{tab:online_hyperparams}.

Initially, as many samples as the batch size are collected by a random policy to pre-populate the experience replay buffer. 
Rather than training for a fixed amount of episodes, the number of policy-environment interactions is used as training steps. 
Consequently, the number of training steps is independent of the agent's intermediate performance and comparable across environments.
The policy is updated in every of those steps, after a single interaction with the environment, where tuples $(s, a, r, s')$ are collected and stored in the buffer. 
After the buffer has reached the maximum size, the oldest tuple is discarded for every new one. 
Action selection during environment interactions to collect samples starts out with an initial $\epsilon$ that linearly decays over a period of steps towards the minimal $\epsilon$, which remains fixed throughout the rest of the training procedure. 
Training batches are sampled randomly from the experience replay buffer.
The Adam optimizer was used for all algorithms and the target network parameters $\theta'$ is updated to match the parameters $\theta$ of the current action-value estimator every $100$ training steps.

The policy is evaluated periodically after a certain number of training steps, depending on the used environment. 
It interacts greedy based on the current value estimate with the environment for $10$ episodes, averaging over the returns to estimate its performance.

\begin{table}[h]
\centering
\begin{tabular}{ll}
\hline
Hyperparameter                & Value      \\ \hline
Algorithm                      & DQN        \\
Learning rate                  & $0.0001$     \\
Batch size                     & $32$         \\
Optimizer                      & Adam    \\
Loss                           & Huber with $\lambda=1$      \\
Initial $\epsilon$             & $1.0$        \\
Linear $\epsilon$ decay period & $1 \, 000$ steps \\
Minimal $\epsilon$             & $0.01$       \\
Target update frequency    & $100$ steps  \\
Training steps             & $100 \, 000$ ($2 \, 000 \, 000$)    \\
Network update frequency  & $1$ step     \\
Experience-Replay Buffer size                 & $50 \, 000$ ($500 \, 000$)   \\ 
Evaluation frequency       & $200$ ($4 \, 000$) steps  \\ \hline
\end{tabular}
\caption{Online training hyperparameters, values in parenthesis apply for MinAtar environments.}
\label{tab:online_hyperparams}
\end{table}

\subsubsection{Offline Training}

If not stated otherwise, the hyperparameters for offline training are identical to the ones used during online training, stated in Tab.~\ref{tab:online_hyperparams}.
All others which differ in an Offline RL setting are listed in Tab.~\ref{tab:res:offlinetraining}.
Furthermore, parameters specific to the used algorithms are stated as well, relying on the parameters provided by the original authors.

Five times as many training steps as in the online case are used for training, which is common in Offline RL since one is interested in asymptotic performance on the fixed dataset.
Algorithms are evaluated after a certain number of training steps through $10$ interaction episodes with the environment, as it is done during the online training.
Resulting returns for each of those evaluation steps are averaged over five independent runs, given an algorithm and a dataset.
The maximum of this returns is then compared to the online policy to obtain the performance of the algorithm on a specific dataset.

\begin{table}[H]
\center
\begin{tabular}{llc}
\hline
Algorithm & Hyperparameter          & Value              \\ \hline
All       & Evaluation frequency     &  $1 \, 000$ ($20 \, 000$) steps              \\
All       & Training steps  & $500 \, 000$ ($10 \, 000 \, 000$) \\
All       & Batch size      &   $128$    \\
QRDQN    & Number of quantiles   $K$  &  $50$               \\
REM       & Number of estimators  $J$  &  $200$               \\
BCQ       & Threshold $\tau$ &   $0.3$           \\
CQL       & Temperature parameter $\alpha$ &   $0.1$         \\
CRR       & samples for advantage estimate $m$      & $4$ \\\hline
\end{tabular}
\caption{Offline training hyperparameters, values in parenthesis apply for MinAtar environments.}
\label{tab:res:offlinetraining}
\end{table}

\subsubsection{Counting Unique State-Action Pairs}
\label{sec:count_unique_saco}
Counting unique state-action pairs of large datasets is often infeasible due to time and memory restrictions. 
Therefore, we evaluate several methods to enable counting on large benchmark datasets.
We compared 1) a simple list-based approach to store all state-action pairs, 2) a Hash-Table and 3) HyperLogLog \citep{Flajolet:07}, a probabilistic counting method.
We specifically chose the HyperLogLog approach, because it can be optimally parallelized or distributed across machines and can be adapted to a "sliding window" usage \citep{Flajolet:07}, which makes it especially useful to RL scenarios.
HyperLogLog has a worst-case time complexity of $\mathcal{O}(N)$ and worst-case memory complexity of $\Theta(\log_2 \log_2 N)$, as there is no need to store a list of unique values.
Even for large $N > 10^9$, estimations typically deviate by a maximum of $2 \%$ from the true counts, as shown in \cite{Flajolet:07}.
An overview of the time and memory complexities of all methods are provided in Tab.~\ref{tab:complexities}.

\begin{table}[h]
\centering
\begin{tabular}{lcc}
\hline
Algorithm       & Time complexity    & Memory complexity     \\
\hline
List of uniques & $\mathcal{O}(N^2)$ & $\Theta(N)$           \\
Hash Table       & $\mathcal{O}(N)$   & $\Theta(N)$           \\
HyperLogLog     & $\mathcal{O}(N)$   & $\Theta(\log_2 \log_2 N)$ \\
\hline
\end{tabular}
\caption{Time and Memory complexities of different algorithms that count unique state-action pairs.}
\label{tab:complexities}
\end{table}

Based on the presented findings, we chose HyperLogLog as a probabilistic counting method to determine the number of unique state-action pairs for each dataset.

\subsubsection{Hardware and Software Specifications}
\label{sec:specs}

Throughout the experiments, PyTorch 1.8 \citep{pytorch2019} with CUDA toolkit 11 \citep{Nickolls:08} on Python 3.8 \citep{Rossum:09} was used. Plots are created using Matplotlib 3.4 \citep{Hunter:2007}.

We used a mixture of 27 GPUs, including GTX 1080 Ti, TITAN X, and TITAN V. Runs for Classic Control and MiniGrid environments took 96 hours in total, the executed runs for MinAtar environments took around 20 days.

\FloatBarrier
\subsection{Calculating TQ and SACo}
\label{sec:env_tq_saco}

All necessary measurements for calculating the TQ and SACo are listed in this section.
The maximum returns attained by the online policy are listed in Tab.~\ref{tab:res:gonline}, the average return attained by the random policy in Tab.~\ref{tab:res:grandom}. 
Furthermore, the average return and unique state-action pairs of each dataset are given in Tab.~\ref{tab:res:gd} and Tab.~\ref{tab:res:ud}. 

\begin{table}[h]
\centering
\begin{tabular}{lrrrrr}
\hline
Environment                    & \multicolumn{5}{c}{Maximum return of online policy $\bar g({\cD_{\text{expert}}})$} \\
                               & Run 1      & Run 2      & Run 3      & Run 4      & Run 5      \\ \hline
\texttt{CartPole-v1}                       & $500.00$   & $500.00$   & $500.00$   & $500.00$   & $500.00$  \\
\texttt{MountainCar-v0}                    & $-99.78$   & $-102.07$  & $-102.70$  & $-100.19$  & $-99.82$  \\
\texttt{MiniGrid-LavaGapS7-v0}             & $0.80$     & $0.91$     & $0.86$     & $0.81$     & $0.85$    \\
\texttt{MiniGrid-Dynamic-Obstacles-8x8-v0} & $0.93$     & $0.93$     & $0.93$     & $0.93$     & $0.92$    \\
\texttt{Breakout-MinAtar-v0}               & $18.02$    & $19.46$    & $17.00$    & $18.47$    & $19.32$   \\
\texttt{SpaceInvaders-MinAtar-v0}        & $26.31$    & $25.17$    & $28.45$    & $28.09$    & $28.08$   \\ \hline
\end{tabular}
\caption{Maximum return of the policy trained online.}
\label{tab:res:gonline}
\end{table}

\begin{table}[h]
\centering
\begin{tabular}{lrrrrr}
\hline
Environment                       & \multicolumn{5}{c}{Average return of the random policy $\bar g({\cD_{\text{min}}})$} \\
                                  & Run 1         & Run 2        & Run 3        & Run 4        & Run 5        \\ \hline
\texttt{CartPole-v1}                       & $22.23$       & $22.12$      & $22.04$      & $22.51$      & $22.05$      \\
\texttt{MountainCar-v0}                    & $-200.00$     & $-200.00$    & $-200.00$    & $-200.00$    & $-200.00$    \\
\texttt{MiniGrid-LavaGapS7-v0}             & $0.02$        & $0.02$       & $0.02$       & $0.02$       & $0.03$       \\
\texttt{MiniGrid-Dynamic-Obstacles-8x8-v0} & $-1.00$       & $-1.00$      & $-1.00$      & $-1.00$      & $-1.00$      \\
\texttt{Breakout-MinAtar-v0}               & $0.51$        & $0.51$       & $0.51$       & $0.51$       & $0.51$       \\
\texttt{SpaceInvaders-MinAtar-v0}          & $2.84$        & $2.83$       & $2.84$       & $2.85$       & $2.85$       \\ \hline
\end{tabular}
\caption{Average return of the random policy.}
\label{tab:res:grandom}
\end{table}

\begin{table}[h]
\centering
\begin{tabular}{llrrrrr}
\hline
Environment         & Dataset & \multicolumn{5}{c}{Average return of dataset trajectories $\bar g({\cD})$} \\
                    &         & Run 1        & Run 2       & Run 3       & Run 4       & Run 5       \\ \hline
\texttt{CartPole-v1}         & Random  & $22.23$      & $22.12$     & $22.04$     & $22.51$     & $22.05$     \\
                    & Mixed   & $27.47$      & $27.15$     & $26.79$     & $26.87$     & $26.17$     \\
                    & Replay  & $208.05$     & $249.72$    & $201.13$    & $215.27$    & $201.98$    \\
                    & Noisy   & $397.03$     & $144.62$    & $248.48$    & $116.77$    & $77.21$     \\
                    & Expert  & $497.48$     & $498.82$    & $279.08$    & $127.98$    & $109.23$    \\ \hline
\texttt{MountainCar-v0}      & Random  & $-200.00$    & $-200.00$   & $-200.00$   & $-200.00$   & $-200.00$   \\
                    & Mixed   & $-176.36$    & $-183.04$   & $-179.71$   & $-182.96$   & $-181.01$   \\
                    & Replay  & $-159.69$    & $-135.38$   & $-135.44$   & $-133.67$   & $-136.20$   \\
                    & Noisy   & $-156.13$    & $-164.55$   & $-164.98$   & $-155.13$   & $-166.10$   \\
                    & Expert  & $-118.90$    & $-135.23$   & $-128.93$   & $-134.63$   & $-132.52$   \\ \hline
\texttt{MiniGrid}            & Random  & $0.02$       & $0.02$      & $0.02$      & $0.02$      & $0.03$      \\
\texttt{-LavaGapS7-v0}       & Mixed   & $0.09$       & $0.05$      & $0.16$      & $0.10$      & $0.08$      \\
                    & Replay  & $0.59$       & $0.70$      & $0.71$      & $0.57$      & $0.62$      \\
                    & Noisy   & $0.61$       & $0.56$      & $0.70$      & $0.70$      & $0.65$      \\
                    & Expert  & $0.63$       & $0.42$      & $0.75$      & $0.70$      & $0.57$      \\ \hline
\texttt{MiniGrid-Dynamic}    & Random  & $-1.00$      & $-1.00$     & $-1.00$     & $-1.00$     & $-1.00$     \\
\texttt{-Obstacles-8x8-v0}   & Mixed   & $-0.87$      & $-0.88$     & $-0.82$     & $-0.81$     & $-0.99$     \\
                    & Replay  & $0.58$       & $0.71$      & $0.53$      & $0.57$      & $0.46$      \\
                    & Noisy   & $-0.09$      & $0.14$      & $0.16$      & $0.19$      & $-0.42$     \\
                    & Expert  & $0.89$       & $0.89$      & $0.92$      & $0.93$      & $0.00$      \\ \hline
\texttt{Breakout-MinAtar-v0} & Random  & $0.51$       & $0.51$      & $0.51$      & $0.51$      & $0.51$      \\
                    & Mixed   & $0.80$       & $0.81$      & $0.80$      & $0.80$      & $0.81$      \\
                    & Replay  & $9.53$       & $10.04$     & $8.92$      & $9.25$      & $9.72$      \\
                    & Noisy   & $4.91$       & $6.90$      & $4.48$      & $4.86$      & $5.78$      \\
                    & Expert  & $13.59$      & $15.61$     & $12.56$     & $14.02$     & $15.28$     \\ \hline
\texttt{SpaceInvaders}       & Random  & $2.84$       & $2.83$      & $2.84$      & $2.85$      & $2.85$      \\
\texttt{-MinAtar-v0}         & Mixed   & $4.07$       & $2.81$      & $4.07$      & $3.54$      & $3.71$      \\
                    & Replay  & $14.85$      & $14.66$     & $15.62$     & $15.46$     & $15.48$     \\
                    & Noisy   & $9.71$       & $3.22$      & $11.46$     & $11.16$     & $13.32$     \\
                    & Expert  & $14.26$      & $2.28$      & $16.65$     & $14.21$     & $18.96$     \\ \hline
\end{tabular}
\caption{Average return of dataset trajectories per environment and dataset creation setting for every run.}
\label{tab:res:gd}
\end{table}

\begin{table}[h]
\centering
\begin{tabular}{llrrrrr}
\hline
Environment         & Dataset & \multicolumn{5}{c}{Unique state-action pairs in dataset $u_{s, a}(\cD)$}          \\
                    &         & Run 1         & Run 2         & Run 3         & Run 4         & Run 5         \\ \hline
\texttt{CartPole-v1}         & Random  & $55\,916$     & $52\,888$     & $58\,127$     & $52\,100$     & $54\,085$     \\
                    & Mixed   & $52\,409$     & $59\,350$     & $60\,820$     & $52\,896$     & $53\,467$     \\
                    & Replay  & $95\,384$     & $94\,749$     & $95\,950$     & $96\,499$     & $97\,263$     \\
                    & Noisy   & $70\,710$     & $64\,392$     & $78\,952$     & $53\,173$     & $51\,771$     \\
                    & Expert  & $15\,496$     & $43\,860$     & $30\,434$     & $19\,349$     & $14\,909$     \\ \hline
\texttt{MountainCar-v0}      & Random  & $3\,315$      & $3\,294$      & $3\,448$      & $3\,015$      & $3\,212$      \\
                    & Mixed   & $5\,294$      & $5\,838$      & $5\,891$      & $4\,725$      & $5\,980$      \\
                    & Replay  & $13\,740$     & $11\,183$     & $12\,411$     & $12\,444$     & $12\,549$     \\
                    & Noisy   & $14\,669$     & $14\,187$     & $14\,138$     & $12\,934$     & $14\,575$     \\
                    & Expert  & $2\,947$      & $3\,768$      & $3\,709$      & $2\,432$      & $4\,123$      \\ \hline
\texttt{MiniGrid}            & Random  & $1\,842$      & $1\,847$      & $1\,879$      & $1\,919$      & $1\,840$      \\
\texttt{-LavaGapS7-v0}       & Mixed   & $1\,819$      & $1\,813$      & $1\,827$      & $1\,866$      & $1\,808$      \\
                    & Replay  & $1\,368$      & $1\,394$      & $1\,343$      & $1\,401$      & $1\,421$      \\
                    & Noisy   & $1\,310$      & $1\,288$      & $1\,450$      & $1\,360$      & $1\,311$      \\
                    & Expert  & $114$         & $112$         & $116$         & $116$         & $104$         \\ \hline
\texttt{MiniGrid-Dynamic}    & Random  & $41\,497$     & $40\,791$     & $40\,843$     & $41\,591$     & $41\,110$     \\
\texttt{-Obstacles-8x8-v0}   & Mixed   & $43\,278$     & $44\,118$     & $42\,968$     & $43\,164$     & $37\,401$     \\
                    & Replay  & $45\,283$     & $45\,423$     & $46\,916$     & $46\,191$     & $44\,801$     \\
                    & Noisy   & $46\,571$     & $49\,191$     & $45\,526$     & $44\,998$     & $40\,115$     \\
                    & Expert  & $38\,704$     & $42\,140$     & $36\,202$     & $35\,331$     & $14\,435$     \\ \hline
\texttt{Breakout-MinAtar-v0} & Random  & $16\,218$     & $15\,915$     & $16\,459$     & $16\,247$     & $16\,182$     \\
                    & Mixed   & $18\,351$     & $18\,608$     & $20\,175$     & $18\,179$     & $19\,166$     \\
                    & Replay  & $62\,737$     & $54\,810$     & $91\,183$     & $61\,433$     & $59\,980$     \\
                    & Noisy   & $38\,326$     & $44\,074$     & $49\,592$     & $38\,527$     & $42\,789$     \\
                    & Expert  & $5\,809$      & $5\,914$      & $9\,006$      & $4\,950$      & $6\,535$      \\ \hline
\texttt{SpaceInvaders}       & Random  & $935\,920$    & $920\,093$    & $925\,641$    & $934\,557$    & $933\,024$    \\
\texttt{-MinAtar-v0}         & Mixed   & $860\,935$    & $777\,601$    & $898\,787$    & $869\,611$    & $901\,127$    \\
                    & Replay  & $1\,507\,798$ & $1\,463\,980$ & $1\,446\,246$ & $1\,439\,305$ & $1\,426\,702$ \\
                    & Noisy   & $933\,016$    & $582\,548$    & $1\,053\,096$ & $955\,208$    & $1\,057\,379$ \\
                    & Expert  & $250\,085$    & $8\,163$      & $407\,306$    & $239\,007$    & $409\,359$    \\ \hline
\end{tabular}
\caption{Unique state-action pairs per environment and dataset creation setting for every run.}
\label{tab:res:ud}
\end{table}

\clearpage
\subsection{Correlations between TQ, SACo, Policy-Entropy and Agent Performance}
\label{sec:correlations}

In the following, detailed plots showing all correlations between TQ, SACo, the entropy of the approximated behavioral policy that created the dataset, and the performance of offline agents are given.
The policy's entropy is calculated by the probabilities of sampling actions for a state, given by the final network output of the same network used for BC.
We included this measure as a comparison to TQ and SACo, as it is easy to obtain by a practitioner.
Nevertheless, the entropy of the behavioral policy has no direct connection to exploitation or exploration as we have shown for our other measures. 
To give the intuition why this is the case, acting as random as possible does not guarantee to explore the state-action space thoroughly as there could be a bottleneck such as the door of a room that prevents to explore the environment further if the policy does not navigate through the door by chance.
Similarly, low entropy thus acting very deterministic, does not necessarily correspond to high exploitation as it could be bad deterministic behaviour as well.
Other possible measures such as using the reward distribution, episode length distribution \citep{Monier:20}, reward sparsity rate, etc. suffer from the same issue.

Scatterplots of TQ, SACo and the entropy estimate of the behavioral policy are depicted in Fig.~\ref{fig:tq_saco_entropy}, where we observe that there is no correlation between TQ and SACo, a stronger negative correlation between the entropy and TQ and a medium positive correlation between the entropy and SACo.

\begin{figure}[h]
    \centering
    \includegraphics[width=\textwidth]{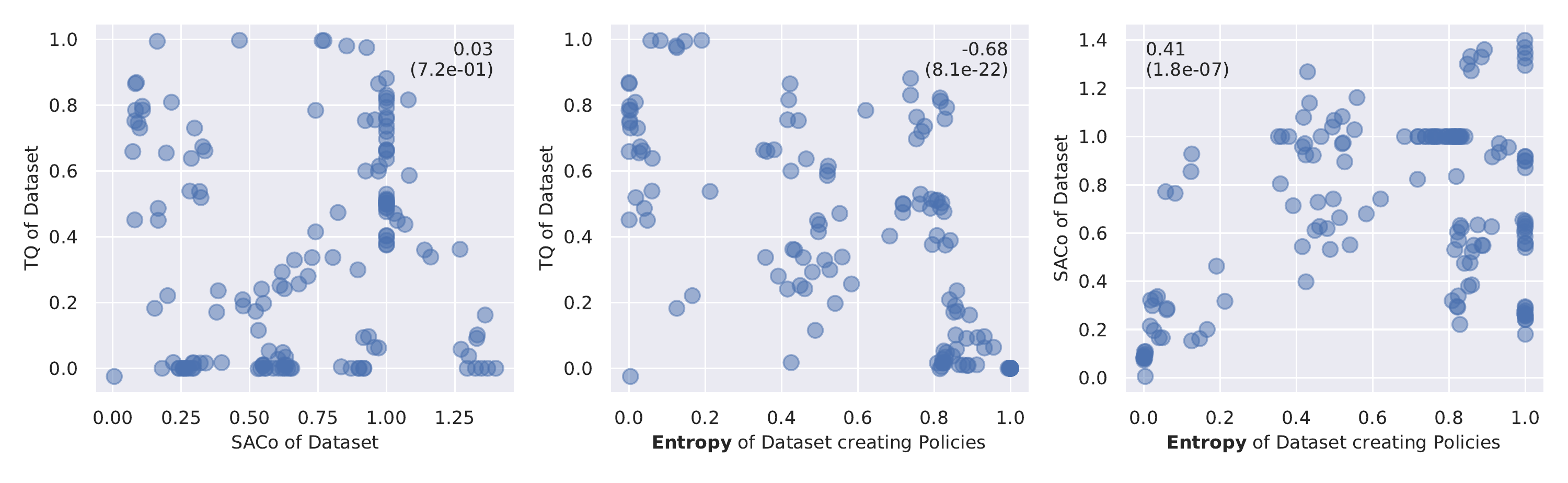}
    \caption{Scatterplots of TQ, SACo and Entropy. Annotation insets state correlation coefficients with corresponding p-value in brackets.}
    \label{fig:tq_saco_entropy}
\end{figure}

Scatterplots between offline agent performance and the TQ are given in Fig.\ref{fig:ap_tq}. As expected, BC has a very strong positive correlation with the TQ of the underlying dataset. 
We observe that the performance of off-policy algorithms DQN, QRDQN and REM that do not constrain the learned policy towards the behavioral policy, exhibit weak negative correlations with the TQ.
Conversely, algorithms that do constrain the learned policy towards the behavioral policy, BCQ, CQL and CRR, exhibit weak positive correlations.

Scatterplots between offline agent performance and the SACo are given in Fig.\ref{fig:ap_saco}. While BC shows very weak if any correlation with the SACo of the underlying dataset, all other algorithms exhibit medium to high positive correlations.

Scatterplots between offline agent performance and the entropy are given in Fig.\ref{fig:ap_entropy}. 
As shown in Fig.\ref{fig:tq_saco_entropy}, TQ and Entropy as well as SACo and Entropy are correlated.
Therefore, BC exhibits a medium negative correlation with entropy, as lower-entropy datasets were created by high performing policies in our experiments.

Algorithms that constrain the learned policy towards the behavioral policy were found to be uncorrelated with the entropy of the behavioral policy, whereas all other algorithms that do not enforce such a constraint have medium positive correlations.
An intuitive explanation would be that algorithms that constrain towards the behavioral policy implicitly use the entropy of the behavioral policy during training, to adjust between searching for the optimal policy and staying close to the behavioral policy.
This is especially easy to see in the case of BCQ, which results in DQN if the behavioral policy is the random policy or to BC if the behavioral policy selected actions greedily.

\begin{figure}[h]
    \centering
    \includegraphics[width=\textwidth]{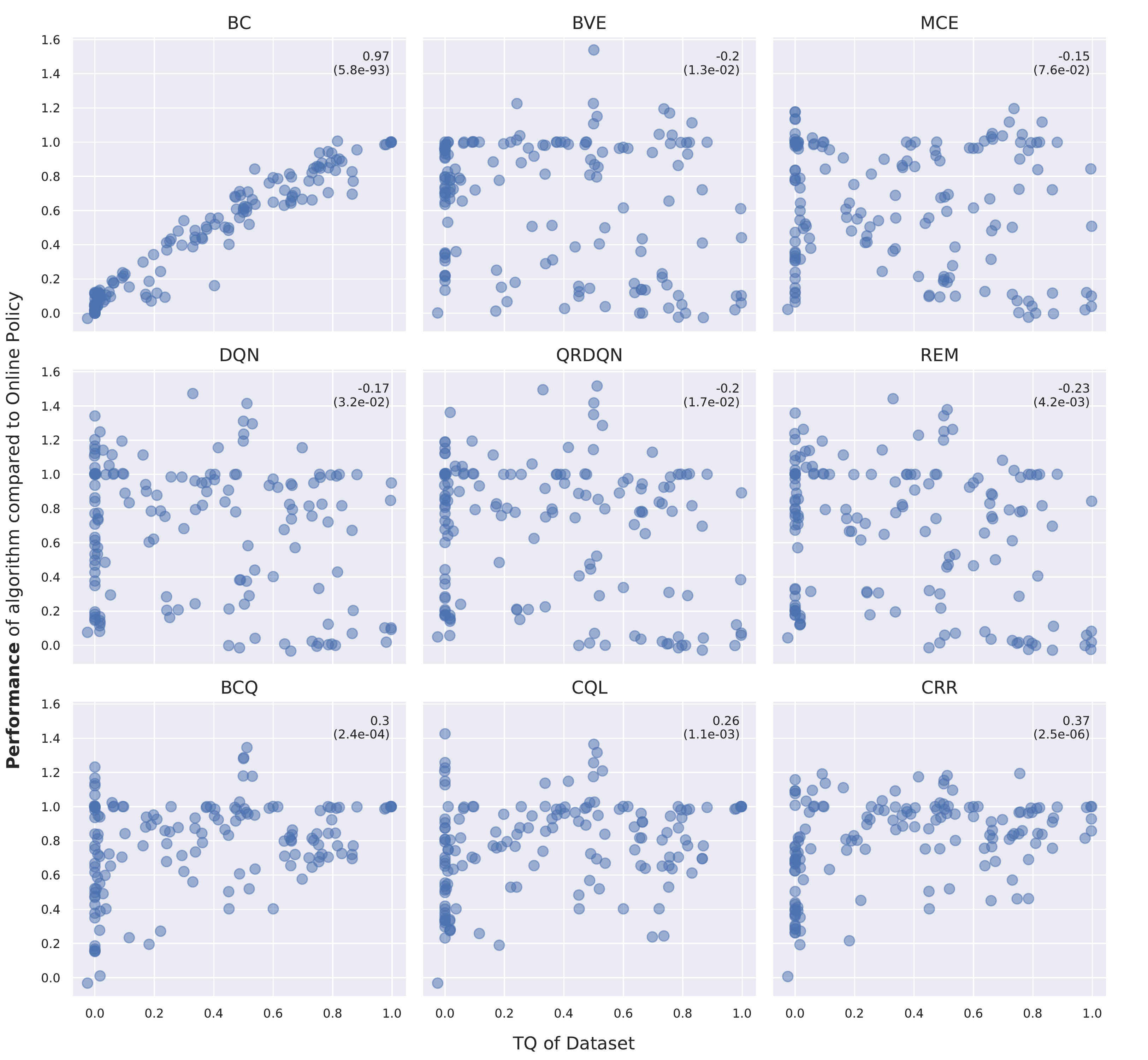}
    \caption{Scatterplots of performance of algorithms and TQ. Annotation insets state correlation coefficients with corresponding p-value in brackets.}
    \label{fig:ap_tq}
\end{figure}

\begin{figure}[h]
    \centering
    \includegraphics[width=\textwidth]{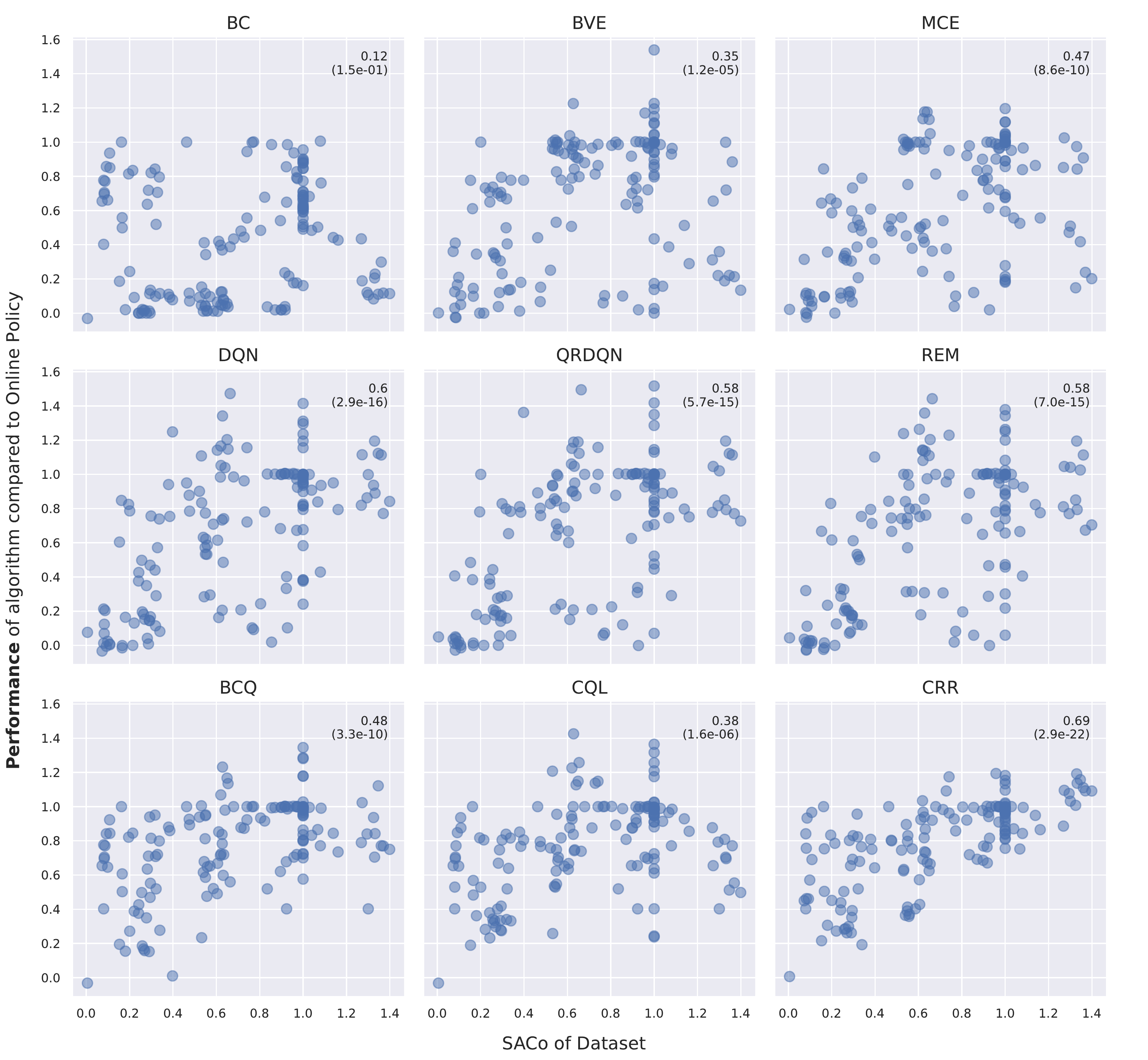}
    \caption{Scatterplots of performance of algorithms and SACo. Annotation insets state correlation coefficients with corresponding p-value in brackets.}
    \label{fig:ap_saco}
\end{figure}

\begin{figure}[h]
    \centering
    \includegraphics[width=\textwidth]{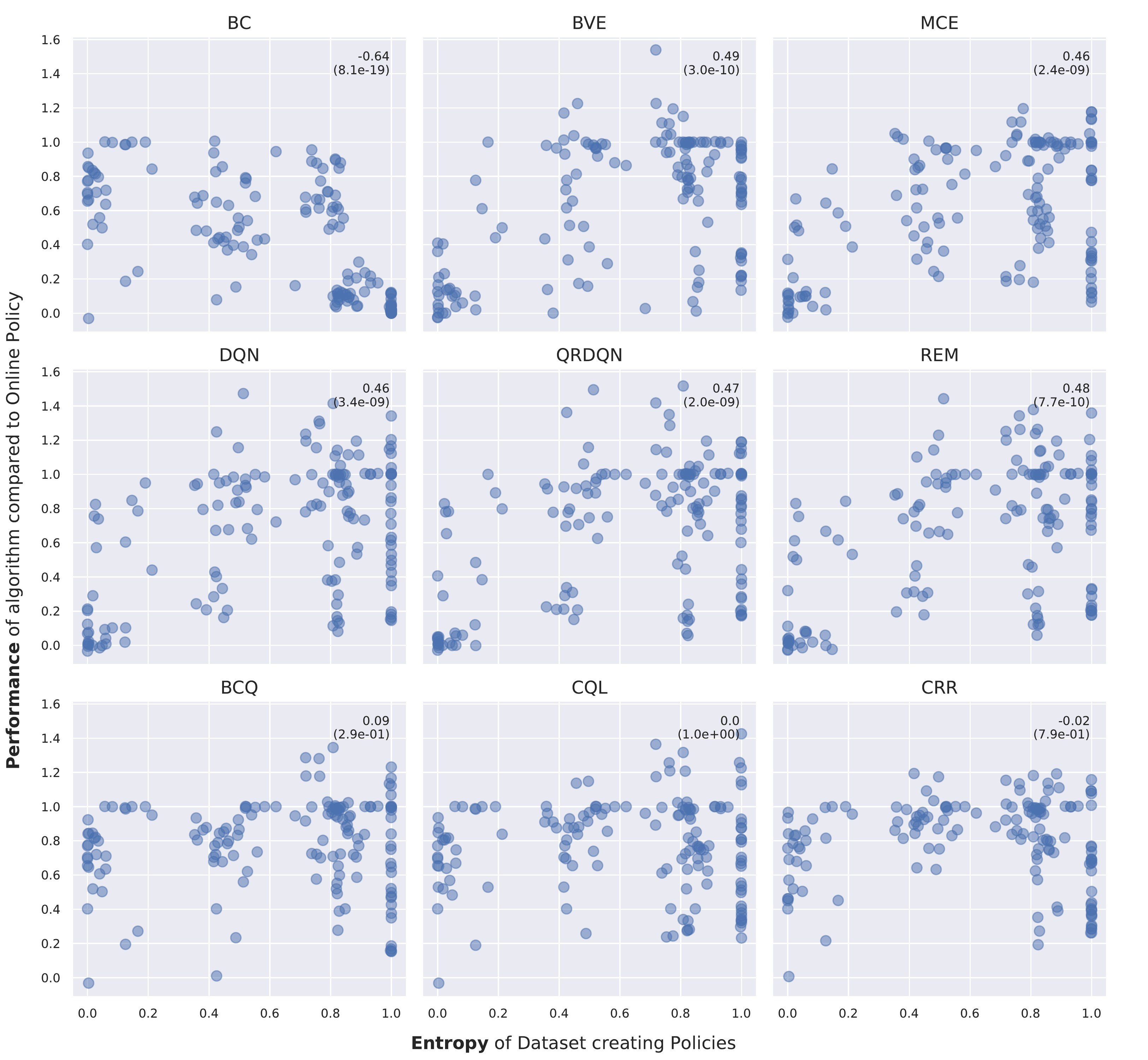}
    \caption{Scatterplots of performance of algorithms and Entropy. Annotation insets state correlation coefficients with corresponding p-value in brackets.}
    \label{fig:ap_entropy}
\end{figure}

\clearpage
\subsection{Performance of Offline Algorithms}
\label{sec:perf_offline_algo}

Results for the best policies learned during the offline training given the generation scheme of the dataset used for training are provided in Fig.~\ref{fig:buffertypes}.

\begin{figure}[h]
    \centering
    \includegraphics[width=\textwidth]{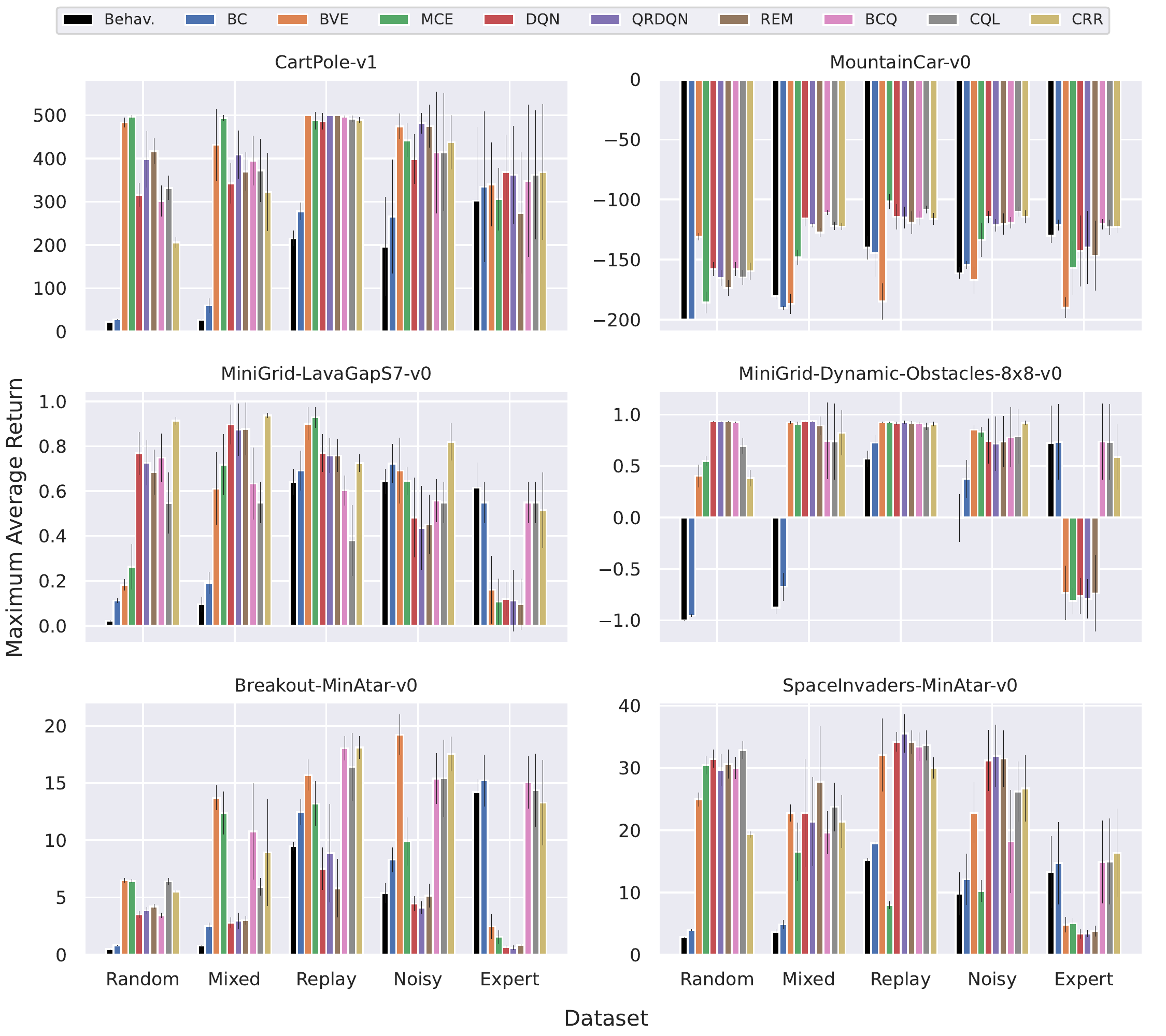}
    \caption{Maximum average return of policies learned during the offline training. Error bars denote the standard deviation over runs and datasets created by the same dataset creation scheme on the same environment. Behav. denotes the behavioral policy used to generate the dataset, thus is the average return of episodes contained in the dataset. }
    \label{fig:buffertypes}
\end{figure}

Performances $\omega$ for every algorithm with the respective dataset settings are given in Tab.~\ref{tab:res:performance_1}.
The results are averaged over different dataset creation seeds and multiple runs carried out with each algorithm, compared to the respective online policy used to create the dataset.

\begin{table}[h]
 \centering 
 \begin{tabular}{lccccccccc} \hline
Dataset 	 & 	 BC	 & 	BVE	 & 	MCE	 & 	DQN	 & 	QRDQN	 & 	REM	 & 	BCQ	 & 	CQL	 & 	CRR 	 \\
\hline \multicolumn{10}{c}{\texttt{CartPole-v1}} \\ 
Random 	 & 	 \begin{tabular}[c]{@{}c@{}}$0.01 $ \\ $\pm 0.00$\end{tabular}	 & 	\begin{tabular}[c]{@{}c@{}}$\mathbf{0.96} $ \\ $\pm 0.02$\end{tabular}	 & 	\begin{tabular}[c]{@{}c@{}}$\mathbf{0.99} $ \\ $\pm 0.01$\end{tabular}	 & 	\begin{tabular}[c]{@{}c@{}}$0.61 $ \\ $\pm 0.06$\end{tabular}	 & 	\begin{tabular}[c]{@{}c@{}}$0.79 $ \\ $\pm 0.14$\end{tabular}	 & 	\begin{tabular}[c]{@{}c@{}}$0.83 $ \\ $\pm 0.06$\end{tabular}	 & 	\begin{tabular}[c]{@{}c@{}}$0.59 $ \\ $\pm 0.07$\end{tabular}	 & 	\begin{tabular}[c]{@{}c@{}}$0.65 $ \\ $\pm 0.06$\end{tabular}	 & 	\begin{tabular}[c]{@{}c@{}}$0.38 $ \\ $\pm 0.03$\end{tabular} 	 \\
Mixed 	 & 	 \begin{tabular}[c]{@{}c@{}}$0.08 $ \\ $\pm 0.04$\end{tabular}	 & 	\begin{tabular}[c]{@{}c@{}}$0.86 $ \\ $\pm 0.17$\end{tabular}	 & 	\begin{tabular}[c]{@{}c@{}}$\mathbf{0.99} $ \\ $\pm 0.02$\end{tabular}	 & 	\begin{tabular}[c]{@{}c@{}}$0.67 $ \\ $\pm 0.10$\end{tabular}	 & 	\begin{tabular}[c]{@{}c@{}}$0.81 $ \\ $\pm 0.12$\end{tabular}	 & 	\begin{tabular}[c]{@{}c@{}}$0.73 $ \\ $\pm 0.09$\end{tabular}	 & 	\begin{tabular}[c]{@{}c@{}}$0.78 $ \\ $\pm 0.12$\end{tabular}	 & 	\begin{tabular}[c]{@{}c@{}}$0.73 $ \\ $\pm 0.15$\end{tabular}	 & 	\begin{tabular}[c]{@{}c@{}}$0.63 $ \\ $\pm 0.19$\end{tabular} 	 \\
Replay 	 & 	 \begin{tabular}[c]{@{}c@{}}$0.54 $ \\ $\pm 0.04$\end{tabular}	 & 	\begin{tabular}[c]{@{}c@{}}$\mathbf{1.00} $ \\ $\pm 0.00$\end{tabular}	 & 	\begin{tabular}[c]{@{}c@{}}$\mathbf{0.97} $ \\ $\pm 0.04$\end{tabular}	 & 	\begin{tabular}[c]{@{}c@{}}$\mathbf{0.97} $ \\ $\pm 0.04$\end{tabular}	 & 	\begin{tabular}[c]{@{}c@{}}$\mathbf{1.00} $ \\ $\pm 0.00$\end{tabular}	 & 	\begin{tabular}[c]{@{}c@{}}$\mathbf{1.00} $ \\ $\pm 0.00$\end{tabular}	 & 	\begin{tabular}[c]{@{}c@{}}$\mathbf{0.99} $ \\ $\pm 0.01$\end{tabular}	 & 	\begin{tabular}[c]{@{}c@{}}$\mathbf{0.98} $ \\ $\pm 0.02$\end{tabular}	 & 	\begin{tabular}[c]{@{}c@{}}$\mathbf{0.98} $ \\ $\pm 0.01$\end{tabular} 	 \\
Noisy 	 & 	 \begin{tabular}[c]{@{}c@{}}$0.51 $ \\ $\pm 0.28$\end{tabular}	 & 	\begin{tabular}[c]{@{}c@{}}$\mathbf{0.95} $ \\ $\pm 0.06$\end{tabular}	 & 	\begin{tabular}[c]{@{}c@{}}$0.88 $ \\ $\pm 0.08$\end{tabular}	 & 	\begin{tabular}[c]{@{}c@{}}$0.79 $ \\ $\pm 0.12$\end{tabular}	 & 	\begin{tabular}[c]{@{}c@{}}$\mathbf{0.96} $ \\ $\pm 0.05$\end{tabular}	 & 	\begin{tabular}[c]{@{}c@{}}$\mathbf{0.95} $ \\ $\pm 0.10$\end{tabular}	 & 	\begin{tabular}[c]{@{}c@{}}$0.82 $ \\ $\pm 0.30$\end{tabular}	 & 	\begin{tabular}[c]{@{}c@{}}$0.82 $ \\ $\pm 0.28$\end{tabular}	 & 	\begin{tabular}[c]{@{}c@{}}$0.87 $ \\ $\pm 0.13$\end{tabular} 	 \\
Expert 	 & 	 \begin{tabular}[c]{@{}c@{}}$\mathbf{0.65} $ \\ $\pm 0.36$\end{tabular}	 & 	\begin{tabular}[c]{@{}c@{}}$\mathbf{0.67} $ \\ $\pm 0.20$\end{tabular}	 & 	\begin{tabular}[c]{@{}c@{}}$\mathbf{0.59} $ \\ $\pm 0.15$\end{tabular}	 & 	\begin{tabular}[c]{@{}c@{}}$\mathbf{0.73} $ \\ $\pm 0.18$\end{tabular}	 & 	\begin{tabular}[c]{@{}c@{}}$\mathbf{0.71} $ \\ $\pm 0.24$\end{tabular}	 & 	\begin{tabular}[c]{@{}c@{}}$0.53 $ \\ $\pm 0.29$\end{tabular}	 & 	\begin{tabular}[c]{@{}c@{}}$\mathbf{0.68} $ \\ $\pm 0.37$\end{tabular}	 & 	\begin{tabular}[c]{@{}c@{}}$\mathbf{0.71} $ \\ $\pm 0.31$\end{tabular}	 & 	\begin{tabular}[c]{@{}c@{}}$\mathbf{0.72} $ \\ $\pm 0.33$\end{tabular} 	 \\
\hline \multicolumn{10}{c}{\texttt{MountainCar-v0}} \\ 
Random 	 & 	 \begin{tabular}[c]{@{}c@{}}$0.00 $ \\ $\pm 0.00$\end{tabular}	 & 	\begin{tabular}[c]{@{}c@{}}$\mathbf{0.70} $ \\ $\pm 0.03$\end{tabular}	 & 	\begin{tabular}[c]{@{}c@{}}$0.14 $ \\ $\pm 0.09$\end{tabular}	 & 	\begin{tabular}[c]{@{}c@{}}$0.42 $ \\ $\pm 0.06$\end{tabular}	 & 	\begin{tabular}[c]{@{}c@{}}$0.35 $ \\ $\pm 0.06$\end{tabular}	 & 	\begin{tabular}[c]{@{}c@{}}$0.26 $ \\ $\pm 0.06$\end{tabular}	 & 	\begin{tabular}[c]{@{}c@{}}$0.42 $ \\ $\pm 0.06$\end{tabular}	 & 	\begin{tabular}[c]{@{}c@{}}$0.35 $ \\ $\pm 0.07$\end{tabular}	 & 	\begin{tabular}[c]{@{}c@{}}$0.41 $ \\ $\pm 0.07$\end{tabular} 	 \\
Mixed 	 & 	 \begin{tabular}[c]{@{}c@{}}$0.10 $ \\ $\pm 0.02$\end{tabular}	 & 	\begin{tabular}[c]{@{}c@{}}$0.13 $ \\ $\pm 0.08$\end{tabular}	 & 	\begin{tabular}[c]{@{}c@{}}$0.52 $ \\ $\pm 0.07$\end{tabular}	 & 	\begin{tabular}[c]{@{}c@{}}$\mathbf{0.85} $ \\ $\pm 0.07$\end{tabular}	 & 	\begin{tabular}[c]{@{}c@{}}$0.80 $ \\ $\pm 0.02$\end{tabular}	 & 	\begin{tabular}[c]{@{}c@{}}$0.73 $ \\ $\pm 0.04$\end{tabular}	 & 	\begin{tabular}[c]{@{}c@{}}$\mathbf{0.90} $ \\ $\pm 0.03$\end{tabular}	 & 	\begin{tabular}[c]{@{}c@{}}$0.79 $ \\ $\pm 0.03$\end{tabular}	 & 	\begin{tabular}[c]{@{}c@{}}$0.78 $ \\ $\pm 0.03$\end{tabular} 	 \\
Replay 	 & 	 \begin{tabular}[c]{@{}c@{}}$0.56 $ \\ $\pm 0.20$\end{tabular}	 & 	\begin{tabular}[c]{@{}c@{}}$0.15 $ \\ $\pm 0.15$\end{tabular}	 & 	\begin{tabular}[c]{@{}c@{}}$\mathbf{0.99} $ \\ $\pm 0.07$\end{tabular}	 & 	\begin{tabular}[c]{@{}c@{}}$0.86 $ \\ $\pm 0.11$\end{tabular}	 & 	\begin{tabular}[c]{@{}c@{}}$0.86 $ \\ $\pm 0.10$\end{tabular}	 & 	\begin{tabular}[c]{@{}c@{}}$0.81 $ \\ $\pm 0.10$\end{tabular}	 & 	\begin{tabular}[c]{@{}c@{}}$0.85 $ \\ $\pm 0.05$\end{tabular}	 & 	\begin{tabular}[c]{@{}c@{}}$\mathbf{0.92} $ \\ $\pm 0.03$\end{tabular}	 & 	\begin{tabular}[c]{@{}c@{}}$0.85 $ \\ $\pm 0.05$\end{tabular} 	 \\
Noisy 	 & 	 \begin{tabular}[c]{@{}c@{}}$0.46 $ \\ $\pm 0.03$\end{tabular}	 & 	\begin{tabular}[c]{@{}c@{}}$0.33 $ \\ $\pm 0.12$\end{tabular}	 & 	\begin{tabular}[c]{@{}c@{}}$0.67 $ \\ $\pm 0.15$\end{tabular}	 & 	\begin{tabular}[c]{@{}c@{}}$\mathbf{0.86} $ \\ $\pm 0.06$\end{tabular}	 & 	\begin{tabular}[c]{@{}c@{}}$0.79 $ \\ $\pm 0.05$\end{tabular}	 & 	\begin{tabular}[c]{@{}c@{}}$0.80 $ \\ $\pm 0.09$\end{tabular}	 & 	\begin{tabular}[c]{@{}c@{}}$0.81 $ \\ $\pm 0.05$\end{tabular}	 & 	\begin{tabular}[c]{@{}c@{}}$\mathbf{0.91} $ \\ $\pm 0.04$\end{tabular}	 & 	\begin{tabular}[c]{@{}c@{}}$\mathbf{0.86} $ \\ $\pm 0.06$\end{tabular} 	 \\
Expert 	 & 	 \begin{tabular}[c]{@{}c@{}}$\mathbf{0.79} $ \\ $\pm 0.05$\end{tabular}	 & 	\begin{tabular}[c]{@{}c@{}}$0.10 $ \\ $\pm 0.09$\end{tabular}	 & 	\begin{tabular}[c]{@{}c@{}}$0.43 $ \\ $\pm 0.23$\end{tabular}	 & 	\begin{tabular}[c]{@{}c@{}}$0.58 $ \\ $\pm 0.30$\end{tabular}	 & 	\begin{tabular}[c]{@{}c@{}}$0.61 $ \\ $\pm 0.31$\end{tabular}	 & 	\begin{tabular}[c]{@{}c@{}}$0.54 $ \\ $\pm 0.29$\end{tabular}	 & 	\begin{tabular}[c]{@{}c@{}}$\mathbf{0.80} $ \\ $\pm 0.04$\end{tabular}	 & 	\begin{tabular}[c]{@{}c@{}}$\mathbf{0.78} $ \\ $\pm 0.07$\end{tabular}	 & 	\begin{tabular}[c]{@{}c@{}}$\mathbf{0.78} $ \\ $\pm 0.06$\end{tabular} 	 \\
\hline
\end{tabular}
\caption{Performance of algorithms averaged over dataset creation seeds and offline runs, where $\pm$ captures the standard deviation. Results are for Classic Control environments on all nine algorithms.}
\label{tab:res:performance_1}
\end{table}

\begin{table}[h]
 \centering 
 \begin{tabular}{lccccccccc} \hline
Dataset 	 & 	 BC	 & 	BVE	 & 	MCE	 & 	DQN	 & 	QRDQN	 & 	REM	 & 	BCQ	 & 	CQL	 & 	CRR 	 \\
\hline \multicolumn{10}{c}{\texttt{MiniGrid-LavaGapS7-v0}} \\ 
Random 	 & 	 \begin{tabular}[c]{@{}c@{}}$0.11 $ \\ $\pm 0.01$\end{tabular}	 & 	\begin{tabular}[c]{@{}c@{}}$0.20 $ \\ $\pm 0.03$\end{tabular}	 & 	\begin{tabular}[c]{@{}c@{}}$0.30 $ \\ $\pm 0.13$\end{tabular}	 & 	\begin{tabular}[c]{@{}c@{}}$0.91 $ \\ $\pm 0.12$\end{tabular}	 & 	\begin{tabular}[c]{@{}c@{}}$0.86 $ \\ $\pm 0.14$\end{tabular}	 & 	\begin{tabular}[c]{@{}c@{}}$0.80 $ \\ $\pm 0.13$\end{tabular}	 & 	\begin{tabular}[c]{@{}c@{}}$0.88 $ \\ $\pm 0.14$\end{tabular}	 & 	\begin{tabular}[c]{@{}c@{}}$0.63 $ \\ $\pm 0.14$\end{tabular}	 & 	\begin{tabular}[c]{@{}c@{}}$\mathbf{1.09} $ \\ $\pm 0.05$\end{tabular} 	 \\
Mixed 	 & 	 \begin{tabular}[c]{@{}c@{}}$0.21 $ \\ $\pm 0.06$\end{tabular}	 & 	\begin{tabular}[c]{@{}c@{}}$0.72 $ \\ $\pm 0.22$\end{tabular}	 & 	\begin{tabular}[c]{@{}c@{}}$0.85 $ \\ $\pm 0.18$\end{tabular}	 & 	\begin{tabular}[c]{@{}c@{}}$\mathbf{1.06} $ \\ $\pm 0.11$\end{tabular}	 & 	\begin{tabular}[c]{@{}c@{}}$1.03 $ \\ $\pm 0.13$\end{tabular}	 & 	\begin{tabular}[c]{@{}c@{}}$1.04 $ \\ $\pm 0.13$\end{tabular}	 & 	\begin{tabular}[c]{@{}c@{}}$0.75 $ \\ $\pm 0.20$\end{tabular}	 & 	\begin{tabular}[c]{@{}c@{}}$0.65 $ \\ $\pm 0.13$\end{tabular}	 & 	\begin{tabular}[c]{@{}c@{}}$\mathbf{1.11} $ \\ $\pm 0.05$\end{tabular} 	 \\
Replay 	 & 	 \begin{tabular}[c]{@{}c@{}}$0.81 $ \\ $\pm 0.08$\end{tabular}	 & 	\begin{tabular}[c]{@{}c@{}}$\mathbf{1.07} $ \\ $\pm 0.08$\end{tabular}	 & 	\begin{tabular}[c]{@{}c@{}}$\mathbf{1.10} $ \\ $\pm 0.06$\end{tabular}	 & 	\begin{tabular}[c]{@{}c@{}}$0.91 $ \\ $\pm 0.13$\end{tabular}	 & 	\begin{tabular}[c]{@{}c@{}}$0.90 $ \\ $\pm 0.12$\end{tabular}	 & 	\begin{tabular}[c]{@{}c@{}}$0.90 $ \\ $\pm 0.13$\end{tabular}	 & 	\begin{tabular}[c]{@{}c@{}}$0.71 $ \\ $\pm 0.07$\end{tabular}	 & 	\begin{tabular}[c]{@{}c@{}}$0.43 $ \\ $\pm 0.17$\end{tabular}	 & 	\begin{tabular}[c]{@{}c@{}}$0.85 $ \\ $\pm 0.04$\end{tabular} 	 \\
Noisy 	 & 	 \begin{tabular}[c]{@{}c@{}}$\mathbf{0.85} $ \\ $\pm 0.12$\end{tabular}	 & 	\begin{tabular}[c]{@{}c@{}}$0.82 $ \\ $\pm 0.21$\end{tabular}	 & 	\begin{tabular}[c]{@{}c@{}}$0.76 $ \\ $\pm 0.10$\end{tabular}	 & 	\begin{tabular}[c]{@{}c@{}}$0.57 $ \\ $\pm 0.25$\end{tabular}	 & 	\begin{tabular}[c]{@{}c@{}}$0.51 $ \\ $\pm 0.26$\end{tabular}	 & 	\begin{tabular}[c]{@{}c@{}}$0.53 $ \\ $\pm 0.18$\end{tabular}	 & 	\begin{tabular}[c]{@{}c@{}}$0.66 $ \\ $\pm 0.13$\end{tabular}	 & 	\begin{tabular}[c]{@{}c@{}}$0.65 $ \\ $\pm 0.13$\end{tabular}	 & 	\begin{tabular}[c]{@{}c@{}}$\mathbf{0.97} $ \\ $\pm 0.12$\end{tabular} 	 \\
Expert 	 & 	 \begin{tabular}[c]{@{}c@{}}$\mathbf{0.65} $ \\ $\pm 0.13$\end{tabular}	 & 	\begin{tabular}[c]{@{}c@{}}$0.17 $ \\ $\pm 0.19$\end{tabular}	 & 	\begin{tabular}[c]{@{}c@{}}$0.10 $ \\ $\pm 0.12$\end{tabular}	 & 	\begin{tabular}[c]{@{}c@{}}$0.12 $ \\ $\pm 0.09$\end{tabular}	 & 	\begin{tabular}[c]{@{}c@{}}$0.10 $ \\ $\pm 0.15$\end{tabular}	 & 	\begin{tabular}[c]{@{}c@{}}$0.08 $ \\ $\pm 0.13$\end{tabular}	 & 	\begin{tabular}[c]{@{}c@{}}$\mathbf{0.65} $ \\ $\pm 0.13$\end{tabular}	 & 	\begin{tabular}[c]{@{}c@{}}$\mathbf{0.65} $ \\ $\pm 0.13$\end{tabular}	 & 	\begin{tabular}[c]{@{}c@{}}$\mathbf{0.60} $ \\ $\pm 0.21$\end{tabular} 	 \\
\hline \multicolumn{10}{c}{\texttt{MiniGrid-Dynamic-Obstacles-8x8-v0}} \\ 
Random 	 & 	 \begin{tabular}[c]{@{}c@{}}$0.02 $ \\ $\pm 0.01$\end{tabular}	 & 	\begin{tabular}[c]{@{}c@{}}$0.73 $ \\ $\pm 0.06$\end{tabular}	 & 	\begin{tabular}[c]{@{}c@{}}$0.80 $ \\ $\pm 0.03$\end{tabular}	 & 	\begin{tabular}[c]{@{}c@{}}$\mathbf{1.00} $ \\ $\pm 0.00$\end{tabular}	 & 	\begin{tabular}[c]{@{}c@{}}$\mathbf{1.00} $ \\ $\pm 0.00$\end{tabular}	 & 	\begin{tabular}[c]{@{}c@{}}$\mathbf{1.00} $ \\ $\pm 0.00$\end{tabular}	 & 	\begin{tabular}[c]{@{}c@{}}$\mathbf{1.00} $ \\ $\pm 0.00$\end{tabular}	 & 	\begin{tabular}[c]{@{}c@{}}$0.88 $ \\ $\pm 0.04$\end{tabular}	 & 	\begin{tabular}[c]{@{}c@{}}$0.72 $ \\ $\pm 0.04$\end{tabular} 	 \\
Mixed 	 & 	 \begin{tabular}[c]{@{}c@{}}$0.17 $ \\ $\pm 0.07$\end{tabular}	 & 	\begin{tabular}[c]{@{}c@{}}$\mathbf{1.00} $ \\ $\pm 0.01$\end{tabular}	 & 	\begin{tabular}[c]{@{}c@{}}$\mathbf{0.99} $ \\ $\pm 0.01$\end{tabular}	 & 	\begin{tabular}[c]{@{}c@{}}$\mathbf{1.00} $ \\ $\pm 0.00$\end{tabular}	 & 	\begin{tabular}[c]{@{}c@{}}$\mathbf{1.00} $ \\ $\pm 0.00$\end{tabular}	 & 	\begin{tabular}[c]{@{}c@{}}$\mathbf{0.98} $ \\ $\pm 0.05$\end{tabular}	 & 	\begin{tabular}[c]{@{}c@{}}$0.90 $ \\ $\pm 0.19$\end{tabular}	 & 	\begin{tabular}[c]{@{}c@{}}$0.90 $ \\ $\pm 0.19$\end{tabular}	 & 	\begin{tabular}[c]{@{}c@{}}$0.94 $ \\ $\pm 0.11$\end{tabular} 	 \\
Replay 	 & 	 \begin{tabular}[c]{@{}c@{}}$0.90 $ \\ $\pm 0.04$\end{tabular}	 & 	\begin{tabular}[c]{@{}c@{}}$\mathbf{1.00} $ \\ $\pm 0.00$\end{tabular}	 & 	\begin{tabular}[c]{@{}c@{}}$\mathbf{1.00} $ \\ $\pm 0.00$\end{tabular}	 & 	\begin{tabular}[c]{@{}c@{}}$\mathbf{0.99} $ \\ $\pm 0.01$\end{tabular}	 & 	\begin{tabular}[c]{@{}c@{}}$\mathbf{1.00} $ \\ $\pm 0.01$\end{tabular}	 & 	\begin{tabular}[c]{@{}c@{}}$\mathbf{1.00} $ \\ $\pm 0.01$\end{tabular}	 & 	\begin{tabular}[c]{@{}c@{}}$\mathbf{0.99} $ \\ $\pm 0.01$\end{tabular}	 & 	\begin{tabular}[c]{@{}c@{}}$\mathbf{0.98} $ \\ $\pm 0.02$\end{tabular}	 & 	\begin{tabular}[c]{@{}c@{}}$\mathbf{0.99} $ \\ $\pm 0.01$\end{tabular} 	 \\
Noisy 	 & 	 \begin{tabular}[c]{@{}c@{}}$0.71 $ \\ $\pm 0.09$\end{tabular}	 & 	\begin{tabular}[c]{@{}c@{}}$\mathbf{0.96} $ \\ $\pm 0.02$\end{tabular}	 & 	\begin{tabular}[c]{@{}c@{}}$\mathbf{0.95} $ \\ $\pm 0.03$\end{tabular}	 & 	\begin{tabular}[c]{@{}c@{}}$0.90 $ \\ $\pm 0.11$\end{tabular}	 & 	\begin{tabular}[c]{@{}c@{}}$0.89 $ \\ $\pm 0.14$\end{tabular}	 & 	\begin{tabular}[c]{@{}c@{}}$0.90 $ \\ $\pm 0.13$\end{tabular}	 & 	\begin{tabular}[c]{@{}c@{}}$0.92 $ \\ $\pm 0.15$\end{tabular}	 & 	\begin{tabular}[c]{@{}c@{}}$0.93 $ \\ $\pm 0.14$\end{tabular}	 & 	\begin{tabular}[c]{@{}c@{}}$\mathbf{0.99} $ \\ $\pm 0.01$\end{tabular} 	 \\
Expert 	 & 	 \begin{tabular}[c]{@{}c@{}}$\mathbf{0.90} $ \\ $\pm 0.19$\end{tabular}	 & 	\begin{tabular}[c]{@{}c@{}}$0.14 $ \\ $\pm 0.14$\end{tabular}	 & 	\begin{tabular}[c]{@{}c@{}}$0.10 $ \\ $\pm 0.07$\end{tabular}	 & 	\begin{tabular}[c]{@{}c@{}}$0.12 $ \\ $\pm 0.09$\end{tabular}	 & 	\begin{tabular}[c]{@{}c@{}}$0.11 $ \\ $\pm 0.10$\end{tabular}	 & 	\begin{tabular}[c]{@{}c@{}}$0.14 $ \\ $\pm 0.19$\end{tabular}	 & 	\begin{tabular}[c]{@{}c@{}}$\mathbf{0.90} $ \\ $\pm 0.19$\end{tabular}	 & 	\begin{tabular}[c]{@{}c@{}}$\mathbf{0.90} $ \\ $\pm 0.19$\end{tabular}	 & 	\begin{tabular}[c]{@{}c@{}}$\mathbf{0.82} $ \\ $\pm 0.16$\end{tabular} 	 \\
\hline
\end{tabular}
\caption{Performance of algorithms averaged over dataset creation seeds and offline runs, where $\pm$ captures the standard deviation. Results are for MiniGrid environments on all nine algorithms.}
\label{tab:res:performance_2}
\end{table}

\begin{table}[]
\centering
\begin{tabular}{lcllcllccl}
\hline
Dataset & BC                                                                    & BVE                                                                  & MCE                                                                  & DQN                                                                  & QRDQN                                                               & REM                                                                  & BCQ                                                                  & CQL                                                                   & CRR                                                                  \\ \hline
\multicolumn{10}{c}{Breakout-MinAtar-v0}                                                                                                                                                                                                                                                                                                                                                                                                                                                                                                                                                                                                                                 \\
Random  & \begin{tabular}[c]{@{}c@{}}$0.02$ \\ $\pm 0.00$\end{tabular}          & \begin{tabular}[c]{@{}l@{}}$\mathbf{0.33}$\\ $\pm 0.02$\end{tabular} & \begin{tabular}[c]{@{}l@{}}$\mathbf{0.33}$\\ $\pm 0.02$\end{tabular} & \begin{tabular}[c]{@{}c@{}}$0.17$\\ $\pm 0.02$\end{tabular}          & \begin{tabular}[c]{@{}l@{}}$0.19$\\ $\pm 0.02$\end{tabular}          & \begin{tabular}[c]{@{}l@{}}$0.21$\\ $\pm 0.02$\end{tabular}          & \begin{tabular}[c]{@{}c@{}}$0.16$ \\ $\pm 0.01$\end{tabular}         & \begin{tabular}[c]{@{}c@{}}$\mathbf{0.33}$\\ $\pm 0.02$\end{tabular}  & \begin{tabular}[c]{@{}l@{}}$0.28$\\ $\pm 0.02$\end{tabular}          \\
Mixed   & \begin{tabular}[c]{@{}c@{}}$0.11$\\ $\pm 0.01$\end{tabular}           & \begin{tabular}[c]{@{}l@{}}$\mathbf{0.73}$\\ $\pm 0.05$\end{tabular} & \begin{tabular}[c]{@{}l@{}}$\mathbf{0.66}$\\ $\pm 0.09$\end{tabular} & \begin{tabular}[c]{@{}c@{}}$0.13$\\ $\pm 0.03$\end{tabular}          & \begin{tabular}[c]{@{}l@{}}$0.14$\\ $\pm 0.04$\end{tabular}          & \begin{tabular}[c]{@{}l@{}}$0.14$\\ $\pm 0.02$\end{tabular}          & \begin{tabular}[c]{@{}c@{}}$0.57$\\ $\pm 0.23$\end{tabular}          & \begin{tabular}[c]{@{}c@{}}$0.30$\\ $\pm 0.03$\end{tabular}           & \begin{tabular}[c]{@{}l@{}}$0.47$\\ $\pm 0.25$\end{tabular}          \\
Replay  & \begin{tabular}[c]{@{}c@{}}$0.67$\\ $\pm 0.05$\end{tabular}           & \begin{tabular}[c]{@{}l@{}}$0.85$\\ $\pm 0.04$\end{tabular}          & \begin{tabular}[c]{@{}l@{}}$0.71$\\ $\pm 0.10$\end{tabular}          & \begin{tabular}[c]{@{}c@{}}$0.39$\\ $\pm 0.11$\end{tabular}          & \begin{tabular}[c]{@{}l@{}}$0.47$\\ $\pm 0.25$\end{tabular}          & \begin{tabular}[c]{@{}l@{}}$0.30$\\ $\pm 0.15$\end{tabular}          & \begin{tabular}[c]{@{}c@{}}$\mathbf{0.98}$\\ $\pm 0.03$\end{tabular} & \begin{tabular}[c]{@{}c@{}}$\mathbf{0.88}$\\ $\pm 0.15$\end{tabular}  & \begin{tabular}[c]{@{}l@{}}$\mathbf{0.98}$\\ $\pm 0.03$\end{tabular} \\
Noisy   & \begin{tabular}[c]{@{}c@{}}$0.43$\\ $\pm 0.04$\end{tabular}           & \begin{tabular}[c]{@{}l@{}}$\mathbf{1.04}$\\ $\pm 0.09$\end{tabular} & \begin{tabular}[c]{@{}l@{}}$0.52$\\ $\pm 0.09$\end{tabular}          & \begin{tabular}[c]{@{}c@{}}$0.22$\\ $\pm 0.04$\end{tabular}          & \begin{tabular}[c]{@{}l@{}}$0.20$\\ $\pm 0.03$\end{tabular}          & \begin{tabular}[c]{@{}l@{}}$0.26$\\ $\pm 0.06$\end{tabular}          & \begin{tabular}[c]{@{}c@{}}$0.83$\\ $\pm 0.09$\end{tabular}          & \begin{tabular}[c]{@{}c@{}}$0.82$\\ $\pm 0.16$\end{tabular}           & \begin{tabular}[c]{@{}l@{}}$\mathbf{0.95}$\\ $\pm 0.04$\end{tabular} \\
Expert  & \begin{tabular}[c]{@{}c@{}}$\mathbf{0.82}$\\ $\pm 0.09$\end{tabular}  & \begin{tabular}[c]{@{}l@{}}$0.11$\\ $\pm 0.07$\end{tabular}          & \begin{tabular}[c]{@{}l@{}}$0.06$\\ $\pm 0.04$\end{tabular}          & \begin{tabular}[c]{@{}c@{}}$0.01$\\ $\pm 0.01$\end{tabular}          & \begin{tabular}[c]{@{}l@{}}$0.01$\\ $\pm 0.01$\end{tabular}          & \begin{tabular}[c]{@{}l@{}}$0.02$\\ $\pm 0.01$\end{tabular}          & \begin{tabular}[c]{@{}c@{}}$\mathbf{0.81}$\\ $\pm 0.09$\end{tabular} & \begin{tabular}[c]{@{}c@{}}$\mathbf{0.77}$\\ $\pm 0.15$\end{tabular}  & \begin{tabular}[c]{@{}l@{}}$\mathbf{0.71}$\\ $\pm 0.18$\end{tabular} \\ \hline
\multicolumn{10}{c}{SpaceInvaders-MinAtar-v0}                                                                                                                                                                                                                                                                                                                                                                                                                                                                                                                                                                                                                            \\
Random  & \begin{tabular}[c]{@{}c@{}}$0.05$\\ $\pm 0.01$\end{tabular}           & \begin{tabular}[c]{@{}l@{}}$0.91$\\ $\pm 0.06$\end{tabular}          & \begin{tabular}[c]{@{}l@{}}$\mathbf{1.13}$\\ $\pm 0.05$\end{tabular} & \begin{tabular}[c]{@{}c@{}}$\mathbf{1.18}$\\ $\pm 0.10$\end{tabular} & \begin{tabular}[c]{@{}l@{}}$\mathbf{1.11}$\\ $\pm 0.12$\end{tabular} & \begin{tabular}[c]{@{}l@{}}$\mathbf{1.15}$\\ $\pm 0.13$\end{tabular} & \begin{tabular}[c]{@{}c@{}}$\mathbf{1.12}$\\ $\pm 0.09$\end{tabular} & \begin{tabular}[c]{@{}c@{}}$\mathbf{1.24}$\\ $\pm 0.11$\end{tabular}  & \begin{tabular}[c]{@{}l@{}}$0.68$\\ $\pm 0.04$\end{tabular}          \\
Mixed   & \begin{tabular}[c]{@{}c@{}}$0.08$\\ $\pm 0.03$\end{tabular}           & \begin{tabular}[c]{@{}l@{}}$\mathbf{0.82}$\\ $\pm 0.08$\end{tabular} & \begin{tabular}[c]{@{}l@{}}$0.57$\\ $\pm 0.23$\end{tabular}          & \begin{tabular}[c]{@{}c@{}}$\mathbf{0.82}$\\ $\pm 0.35$\end{tabular} & \begin{tabular}[c]{@{}l@{}}$\mathbf{0.76}$\\ $\pm 0.29$\end{tabular} & \begin{tabular}[c]{@{}l@{}}$\mathbf{1.02}$\\ $\pm 0.36$\end{tabular} & \begin{tabular}[c]{@{}c@{}}$\mathbf{0.69}$\\ $\pm 0.17$\end{tabular} & \begin{tabular}[c]{@{}c@{}}$\mathbf{0.87}$\\ $\pm 0.20$\end{tabular}  & \begin{tabular}[c]{@{}l@{}}$\mathbf{0.76}$\\ $\pm 0.15$\end{tabular} \\
Replay  & \begin{tabular}[c]{@{}c@{}}$0.62$\\ $\pm 0.02$\end{tabular}           & \begin{tabular}[c]{@{}l@{}}$\mathbf{1.19}$\\ $\pm 0.20$\end{tabular} & \begin{tabular}[c]{@{}l@{}}$0.21$\\ $\pm 0.04$\end{tabular}          & \begin{tabular}[c]{@{}c@{}}$\mathbf{1.29}$\\ $\pm 0.07$\end{tabular} & \begin{tabular}[c]{@{}l@{}}$\mathbf{1.34}$\\ $\pm 0.13$\end{tabular} & \begin{tabular}[c]{@{}l@{}}$\mathbf{1.29}$\\ $\pm 0.07$\end{tabular} & \begin{tabular}[c]{@{}c@{}}$\mathbf{1.25}$\\ $\pm 0.07$\end{tabular} & \begin{tabular}[c]{@{}c@{}}$\mathbf{1.26}$\\ $\pm 0.07$\end{tabular}  & \begin{tabular}[c]{@{}l@{}}$1.12$\\ $\pm 0.06$\end{tabular}          \\
Noisy   & \begin{tabular}[c]{@{}c@{}}$0.37$\\ $\pm 0.16$\end{tabular}           & \begin{tabular}[c]{@{}l@{}}$0.82$\\ $\pm 0.18$\end{tabular}          & \begin{tabular}[c]{@{}l@{}}$0.30$\\ $\pm 0.06$\end{tabular}          & \begin{tabular}[c]{@{}c@{}}$\mathbf{1.16}$\\ $\pm 0.19$\end{tabular} & \begin{tabular}[c]{@{}l@{}}$\mathbf{1.20}$\\ $\pm 0.21$\end{tabular} & \begin{tabular}[c]{@{}l@{}}$\mathbf{1.17}$\\ $\pm 0.16$\end{tabular} & \begin{tabular}[c]{@{}c@{}}$0.62$\\ $\pm 0.33$\end{tabular}          & \begin{tabular}[c]{@{}c@{}}$\mathbf{0.96}$ \\ $\pm 0.17$\end{tabular} & \begin{tabular}[c]{@{}l@{}}$\mathbf{0.97}$\\ $\pm 0.18$\end{tabular} \\
Expert  & \begin{tabular}[c]{@{}c@{}}$\mathbf{0.48}$ \\ $\pm 0.26$\end{tabular} & \begin{tabular}[c]{@{}l@{}}$0.08$\\ $\pm 0.05$\end{tabular}          & \begin{tabular}[c]{@{}l@{}}$0.09$\\ $\pm 0.04$\end{tabular}          & \begin{tabular}[c]{@{}c@{}}$0.02$\\ $\pm 0.03$\end{tabular}          & \begin{tabular}[c]{@{}l@{}}$0.02$\\ $\pm 0.02$\end{tabular}          & \begin{tabular}[c]{@{}l@{}}$0.04$\\ $\pm 0.03$\end{tabular}          & \begin{tabular}[c]{@{}c@{}}$\mathbf{0.48}$\\ $\pm 0.27$\end{tabular} & \begin{tabular}[c]{@{}c@{}}$\mathbf{0.49}$ \\ $\pm 0.27$\end{tabular} & \begin{tabular}[c]{@{}l@{}}$\mathbf{0.54}$\\ $\pm 0.29$\end{tabular} \\ \hline
\end{tabular}
\caption{Performance of algorithms averaged over dataset creation seeds and offline runs, where $\pm$ captures the standard deviation. Results are for MinAtar environments on all nine algorithms.}
\label{tab:res:performance_3}
\end{table}

\clearpage
\subsection{Illustration of sampled State-Action Space on MountainCar}

In Fig.~\ref{fig:res:projections} we illustrate the effect of the behavioral policy on the distribution of the sampled dataset on the example of the \texttt{MountainCar} environment. 
This environment was chosen as the state space is two-dimensional and thus provides axes with physical meaning.

In this example, the dataset obtained through a random policy has only limited coverage of the whole state-action space. 
This is, because the random policy is not able to transition far from the starting position due to restrictive environment dynamics, necessitating long sequences with identical actions (swinging left and right) to deviate farther from the starting state distribution.

Furthermore, the expert policies obtained in each independent run differ from one another in how they steer the agent towards the goal, for instance, refraining to use the action \textit{"Don't accelerate"} in the first run.

\begin{figure}[h]
    \centering
    \includegraphics[width=1\textwidth]{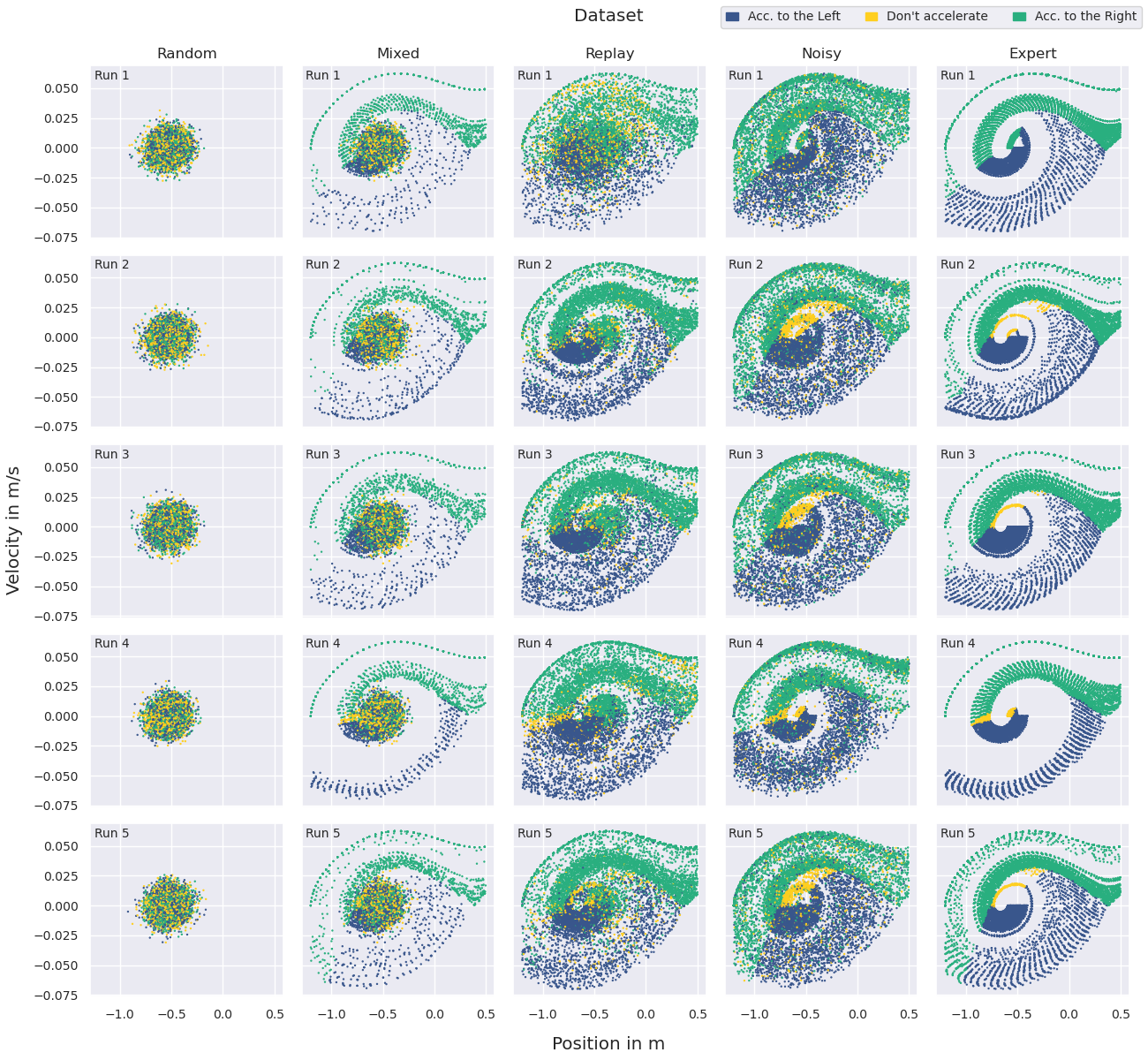}
    \caption{State-action space for different datasets created from the environment \texttt{MountainCar} under different dataset schemes for five independent runs. 10\% of the datasets were sub-sampled for plotting.}
    \label{fig:res:projections}
\end{figure}

\clearpage
\subsection{Performance per Dataset Generation Scheme}

To obtain results per dataset generation scheme, the results for the five dataset creation runs per scheme are averaged. Therefore, the TQ and SACo are averaged as well as the performance for the respective algorithm on each dataset. Results are depicted in Fig.~\ref{fig:algos_all_avg}.\\

\begin{figure}[h]
    \centering
    \includegraphics[width=\textwidth]{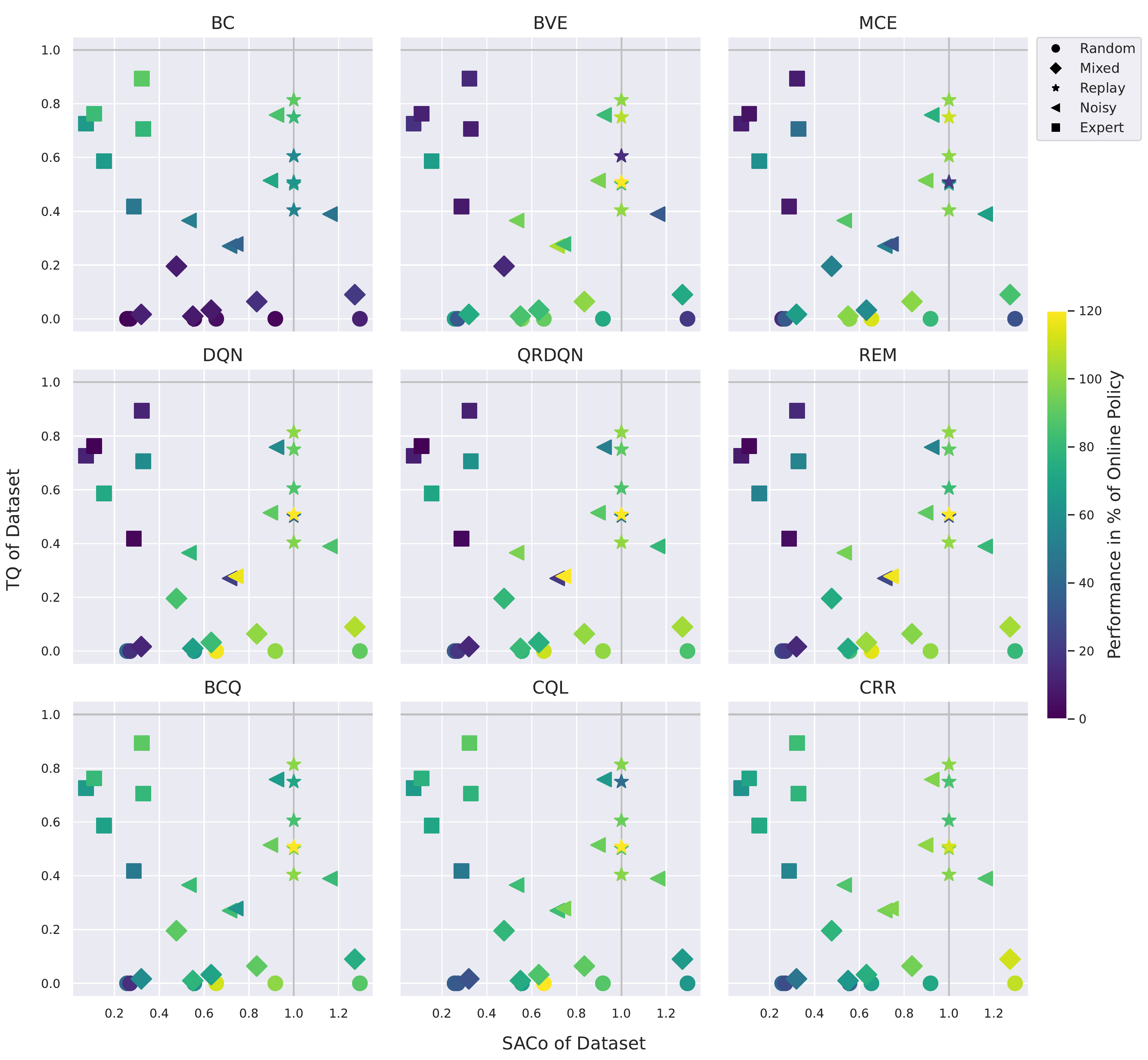}
    \caption{Performance of algorithms compared to the online policy used to create the datasets, with respect to the TQ and SACo of the dataset. Points denote the different datasets. TQ, SACo and performance are averaged over results for each of the five dataset creation seeds.}
    \label{fig:algos_all_avg}
\end{figure}

\clearpage
\subsection{Results with $\mathit{lSACo}$}
\label{app:lsaco}

We define lSACo analogous to SACo as defined in Sec.~\ref{derived-measures} as

\begin{align}
    \mathit{lSACo}(\cD) := \frac{\log(u_{s,a}(\cD))}{\log(u_{s,a}(\cD_{\text{ref}}))}.
    \label{eq:lsaco}
\end{align}

In Fig.~\ref{fig:algos_all_log} we visualize the same results as presented in Fig.~\ref{fig:algos_all}.
We conclude that our findings elaborated in Sec.\ref{sec:experiments} remain the same, but are harder to interpret from an empirical point of view due to the logarithmic scaling of lSACo.

\begin{figure}[h]
    \centering
    \includegraphics[width=\textwidth]{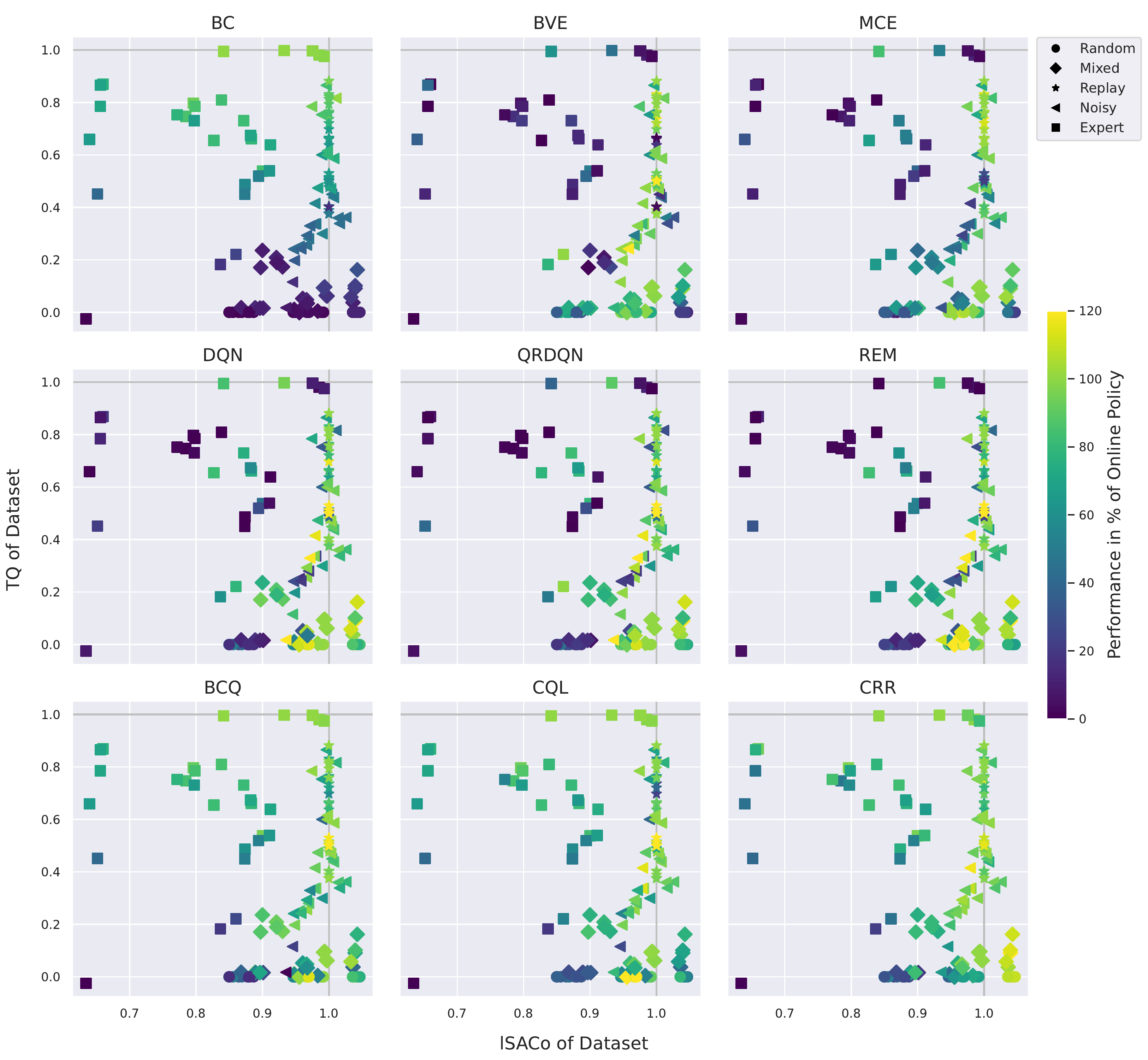}
    \caption{Performance of algorithms compared to the online policy used to create the datasets, with respect to the TQ and lSACo of the dataset. Points denote the different datasets, color denotes performance of the respective algorithm.}
    \label{fig:algos_all_log}
\end{figure}

\clearpage
\subsection{Results with Na\"{\i}ve Entropy Estimator}\label{sec:naive_emp}

As discussed in Sec.~\ref{derived-measures}, we also evaluated our main experiments with an empirical measures based on the na\"{\i}ve entropy estimator $\hat H(\cD) = - \sum_i^K \hat p_i \log(\hat p_i)$ with $\hat p_i = n_{s,a} / N$, where $n_{s,a}$ is the count of a specific visitable state-action pair $(s,a)$ of the visitable state-action pairs $K$ under the policy, in the dataset of size $N$.
The entropy estimate of a given dataset is also normalized with a reference dataset $\cD_{\text{ref}}$, where we use the replay dataset throughout all experiments.
The empirical exploration measure of a dataset is thus $\hat H(\cD) / \hat H(\cD_{\text{ref}})$.

We find qualitatively similar results to our main experiments using SACo.
A major difference is the strong concentration around one, thus the entropy estimate of different datasets are close together in value.
Furthermore, random datasets and mixed datasets (consisting to 80\% of the random dataset) exhibit much higher scores than for SACo.
This is expected, because even if a smaller number of state-action pairs is visited, chances are high the reachable states are visited relatively even under a random policy which results in a high entropy estimate.

\begin{figure}[h]
    \centering
    \includegraphics[width=\textwidth]{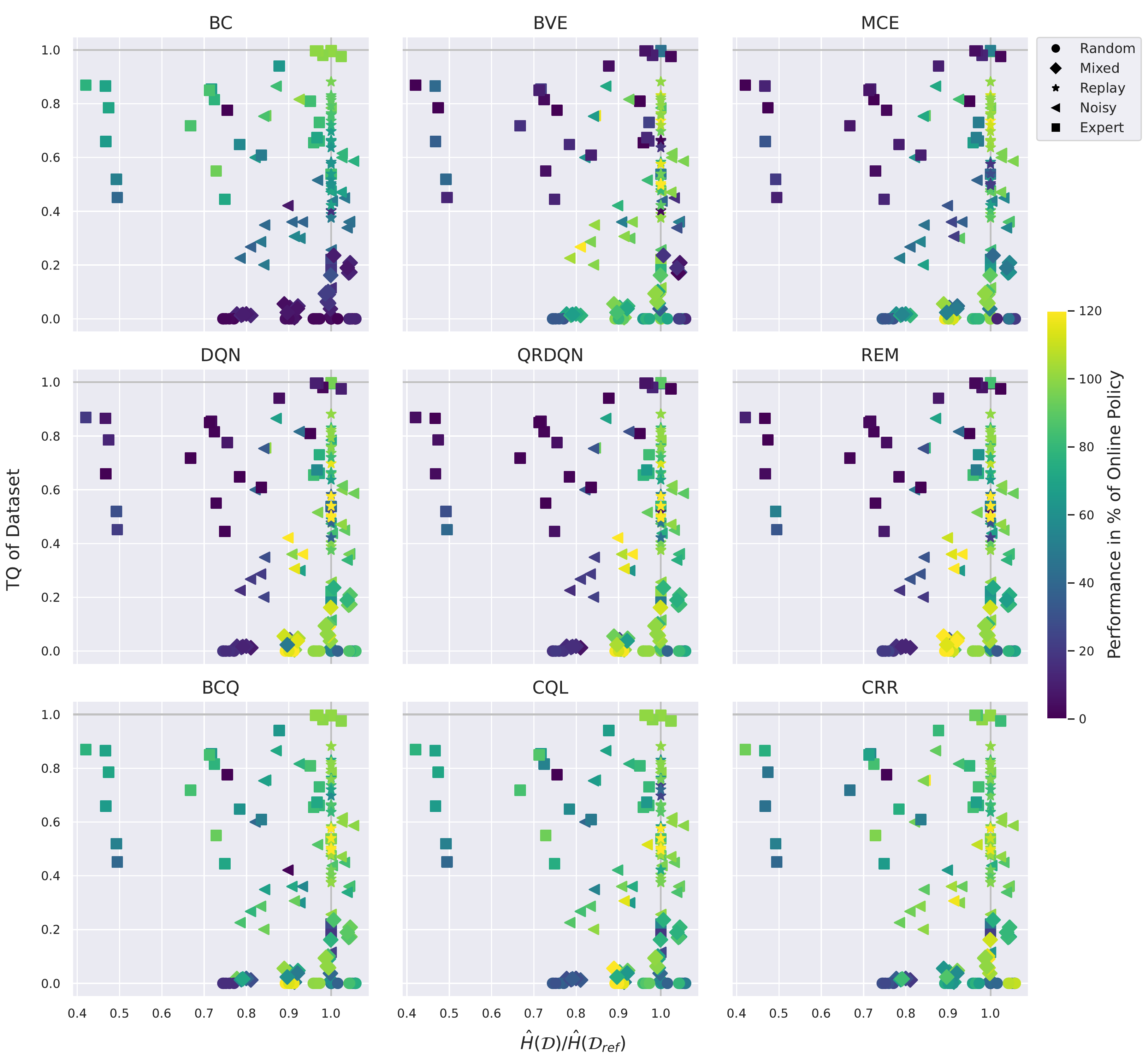}
    \caption{Performance of algorithms compared to the online policy used to create the datasets, with respect to the TQ and $\hat H(\cD) / \hat H(\cD_{\text{ref}})$ of the dataset. Points denote different datasets, color denotes performance of the respective algorithm.}
    \label{fig:pp_algos_all_log}
\end{figure}

\clearpage
\subsection{Additional Results on D4RL Gym-MuJoCo Datasets}\label{mujoco}

Many recent algorithmic advances in Offline RL for continuous action spaces \citep{Wu:19, Kumar:19, Kumar:20, Kostrikov:21, Fujimoto:21} report results on D4RL datasets \citep{Fu:21}, most prominently on the Gym-MuJoCo datasets.
Here we present some additional results on these datasets, based on the results reported by \cite{Kumar:20}, to investigate the applicability of our proposed measures to characterize datasets sampled from continuous state-action spaces.

Calculating SACo in continuous state-action spaces requires discretization.
We discretized for each dimension for the state and action space individually, using the same ranges and bin sizes for each datasets of the same environment.
In line with our experiments depicted in Fig.~\ref{fig:algos_all} in the main paper, we use the replay dataset as reference dataset.
The reference scores for TQ are the return of the best online policy and the return of the random policy, which are generally used to normalize the performance when reporting results on these datasets \citep{Fu:21}.
Performances of all algorithms are taken from \cite{Kumar:20}.
Results are depicted in Fig.~\ref{fig:algos_d4rl}.

\begin{figure}[h]
    \centering
    \includegraphics[width=\textwidth]{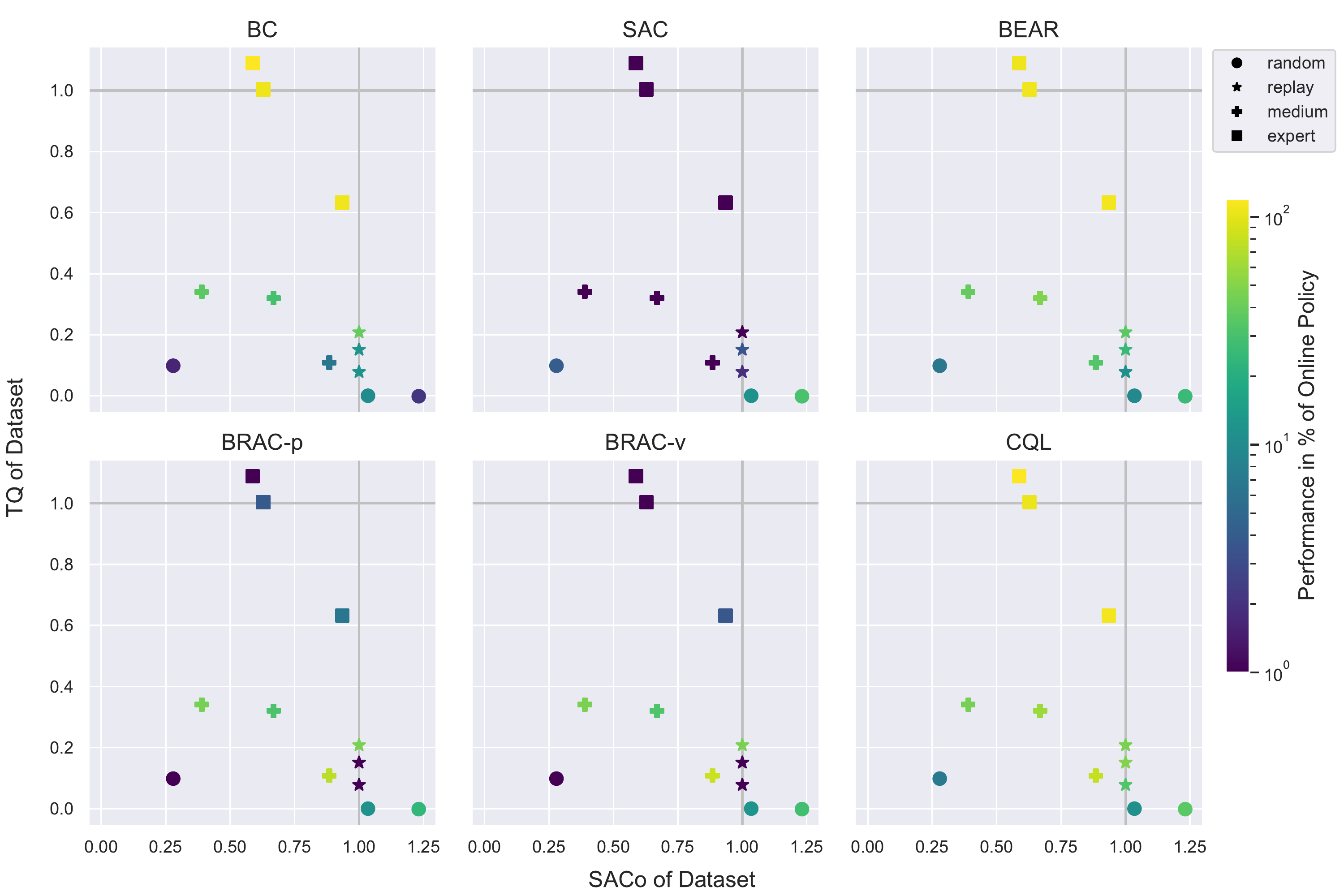}
    \caption{Results of a variety of algorithms applicable to continuous action-spaces on D4RL Gym-MuJoCo \texttt{halfcheetah}, \texttt{walker2d} and \texttt{hopper} environments and feature the random, replay, mixed and expert datasets. Datasets are \texttt{-v0} variants as published by \cite{Fu:21}. Performance of algorithms are taken from \cite{Kumar:20}, who evaluates on the same datasets. The color coding of the performance is scaled logarithmic to highlight performance differences on the lower performance end.}
    \label{fig:algos_d4rl}
\end{figure}

In line with the results on discrete action environments, we find that BC has a strong correlation with TQ.
The algorithm used for online learning, SAC \citep{Haarnoja:18}, only finds policies that are better than random for datasets with high SACo.
Furthermore, we find that BEAR \cite{Kumar:19} and CQL \citep{Kumar:20} attain good policies (relative to BC and SAC) for both datasets with high TQ and SACo.
The results for BRAC-p and BRAC-v \cite{Wu:19} are worse than for BEAR and CQL overall, especially for the expert datasets with high TQ, where they fail to meet the performance of BC.

Overall, other than exhibited in our main experiments, algorithms only attain performance close to the online agent if trained on large portions of expert data.
We hypothesize, this might be due to the nature of the robotic control environments, where steering the agent to any point in the state space is in general very complex and useful behavior only operates in a small subspace of the full state-action space.
Future work would have to investigate the effects of dynamics and dimensionality on continuous state-action space problems, as all the environments considered have similarly high dimensional state-action spaces and similar dynamics (especially \texttt{walker2d} and \texttt{hopper}).

\end{document}